\documentclass[twoside]{article}

\usepackage[accepted]{aistats2026}
\usepackage{natbib}
\usepackage{hyperref}
\usepackage{url}
\usepackage{physics}
\usepackage{cleveref}
\usepackage{booktabs}
\usepackage{adjustbox}    
\usepackage{amssymb}
\usepackage{xcolor}
\usepackage{subcaption}
\usepackage{float}
\usepackage{bbding} 
\usepackage{makecell}

\usepackage{xcolor}
\definecolor{codegreen}{rgb}{0,0.6,0}
\definecolor{codegray}{rgb}{0.5,0.5,0.5}
\definecolor{codepurple}{rgb}{0.58,0,0.82}
\definecolor{navyblue}{rgb}{0, 0, 0.52}
\definecolor{rubyred}{rgb}{0.6,0.07,0.12}
\definecolor{backcolour}{rgb}{0.95,0.95,0.92}

\usepackage{hyperref}
\hypersetup{
    colorlinks=true,
    linkcolor=rubyred,      
    citecolor=navyblue,       
    filecolor=codegray,        
    urlcolor=codepurple,       
    pdfborder={0 0 0},         
}

\newcommand{\transparentcheck}{%
  {\color{black!50}\checkmark}%
}

\usepackage[utf8]{inputenc}
\usepackage{amsmath, amssymb, amsthm}
\usepackage{bm}

\newtheorem{theorem}{Theorem}
\newtheorem{lemma}[theorem]{Lemma}
\newtheorem{proposition}[theorem]{Proposition}
\newtheorem{corollary}[theorem]{Corollary}
\newtheorem{assumption}{Assumption}

\runningauthor{Gopakumar et al.}
\begin{document}

%

%

\twocolumn[ 

\aistatstitle{Learning Physical Operators using Neural Operators}

\aistatsauthor{ 
  Vignesh Gopakumar$^{1,2}$, 
  Ander Gray$^{2,3}$, 
  Daniel Giles$^{1}$, 
  Lorenzo Zanisi$^{2}$,}
\aistatsauthor{
  Matt J. Kusner$^{4,5}$, 
  Timo Betcke$^{1}$, 
  Stanislas Pamela$^{2}$, 
  Marc Peter Deisenroth$^{1}$ 
}

\vspace{2ex}
\aistatsaddress{
  $^{1}$UCL Centre for Artificial Intelligence, 
  $^{2}$UK Atomic Energy Authority, \\
  $^{3}$LIX, CNRS, École Polytechnique, $^{4}$Polytechnique Montr\'{e}al, $^{5}$Mila - Quebec AI Institute
}]

\begin{abstract}
Neural operators have emerged as promising surrogate models for solving partial differential equations (PDEs), but struggle to generalise beyond training distributions and are often constrained to a fixed temporal discretisation. This work introduces a physics-informed training framework that addresses these limitations by decomposing PDEs using operator splitting methods, training separate neural operators to learn individual non-linear physical operators while approximating linear operators with fixed finite-difference convolutions. This modular mixture-of-experts architecture enables generalisation to novel physical regimes by explicitly encoding the underlying operator structure. We formulate the modelling task as a neural ordinary differential equation (ODE) where these learned operators constitute the right-hand side, enabling continuous-in-time predictions through standard ODE solvers and implicitly enforcing PDE constraints. Demonstrated on incompressible and compressible Navier--Stokes equations, our approach achieves better convergence and superior performance when generalising to unseen physics. The method remains parameter-efficient, enabling temporal extrapolation beyond training horizons, and provides interpretable components whose behaviour can be verified against known physics. 
\end{abstract}

\section{Introduction}
\label{sec:introduction}
Partial differential equations (PDEs) serve as the mathematical foundation for modelling complex physical systems across diverse scientific and engineering applications, from fluid dynamics \citep{OpenFOAM} and heat transfer \citep{giudicelli2024moose} to nuclear fusion \citep{Hoelzl2021jorek} and climate modelling \citep{cesm2}. Traditional numerical methods, such as finite element and spectral approaches \citep{reddy2006introduction, fornberg1996practical}, while mathematically rigorous, present substantial computational challenges requiring supercomputing resources, extensive solution times and a higher carbon footprint \citep{Keyes_Sim_Challenges,carbonfootprint_CFD}. These computational limitations restrict large-scale deployment for iterative design and optimisation \citep{simintelligence}. 

Neural Operators (NO) have emerged as promising surrogate models, offering potential computational cost reductions of several orders of magnitude while maintaining considerable accuracy \citep{Li2021fourier, pfaff2021learning, Lu2021deeponet}. However, most NOs are trained from simulation data in a supervised manner, lacking an explicit definition of the underlying PDE structure, which limits their ability to generalise outside training distributions. Their autoregressive structure further constrains temporal flexibility and extrapolation capabilities \citep{lee2023autoregressiverenaissanceneuralpde}.
\begin{figure*}
    \centering
    \includegraphics[width=0.9\linewidth]{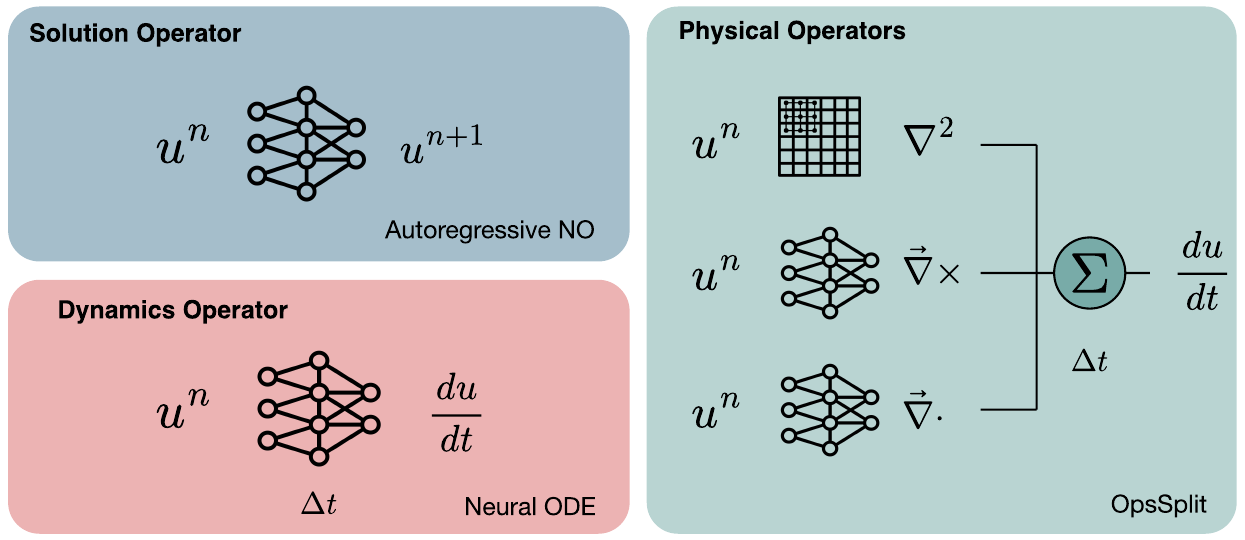}
    \caption{Comparison of operator learning approaches for PDE solving. The figure contrasts three methods: (top left) the traditional autoregressive approach where a single neural operator learns the solution operator mapping directly from state $u^n$ to $u^{n+1}$; (bottom left) the neural ODE approach where a neural operator learns the dynamics operator $\dv{u}{t}$ integrated with an ODE solver; and (right) the proposed OpsSplit method where individual neural operators learn specific physical operators $(\nabla \cross, \nabla \cdot)$ and are combined via operator splitting to compute $\dv{u}{t}$ which is integrated using an ODE solver. The OpsSplit approach decomposes the PDE into its constituent physical operators, with linear operators approximated by convolutions and non-linear operators learned by neural networks, enabling physics-informed and modular PDE solving.}
    \label{fig:framework}
\end{figure*}
This work presents a novel physics-informed training framework that addresses these limitations by decomposing PDEs using operator splitting methods \citep{blanes2024splittingmethodsdifferentialequations}. We adopt a Mixture of Experts (MoE) approach \citep{jacobs1991adaptive} in which individual neural operators learn specific physical operators, as shown in \cref{fig:framework}. The method reformulates the modelling task as a neural ordinary differential equation (ODE) with the right-hand side characterising spatial physical phenomena through a combination of neural operators and linear operators. Following operator splitting principles, our approach models non-linear operators using neural networks while approximating linear operators with fixed finite-difference convolutions. By learning physical operators instead of solution operators, we enforce PDE constraints during prediction and enable the construction of interpretable models whose behaviour can be numerically verified \footnote{\url{https://github.com/gitvicky/NOs_for_POs}}. Being modular, it also allows us to add and remove operators as the physics of the system change.  Our methodology provides three key advantages:
\begin{enumerate}
    \item \textbf{Generalisable Neural Operators:} The Mixture-of-Experts structure enables neural operators to explicitly learn physical operators, creating efficient models capable of generalising beyond training regimes. The modular structure allows for the insertion and removal of operators to adapt to changing physics. By learning physical operators, we obtain better convergence for both pre-training and fine-tuning across PDEs. 
    \item \textbf{Continuous in Time:} A neural ODE formulation provides temporal flexibility, with the right-hand side given by combined neural and linear operators integrated using standard ODE solvers.
    \item \textbf{Physics-Informed Machine Learning:} Physics consistency is inherently encoded through built-in enforcement of PDE constraints via the operator splitting structure of the prediction regime.
\end{enumerate}


\section{Related Work}

Physics-informed machine learning has created a new paradigm in computational physics by merging data with numerical models \citep{PIML,simintelligence}, enabling predictive forward modelling \citep{lippe2023pderefiner}, inverse modelling \citep{Jagtap_2022,Chen2021}, and scientific discovery \citep{Poels_2025,cranmer2023interpretablemachinelearningscience}. Neural PDE solvers have found interdisciplinary applications as surrogate models for efficient spatio-temporal PDE approximation. Physics-Informed Neural Networks (PINNs) optimise neural networks by minimising PDE residuals as loss functions \citep{Raissi2019PINNs}, while Neural Operators (NOs) learn PDE mappings by extracting dominant modes from simulation data for chosen basis functions \citep{Lu2021deeponet}. Physics-informed Neural Operators (PINOs) combine both approaches, training NO architectures in a supervised manner before fine-tuning with PDE residual minimisation \citep{LiPino2024, Rosofsky_2023}. Most neural operators learn solution operators through autoregressive construction, maintaining discrete temporal evolution \citep{Kovachki_NeuralOperators_2023}. Recent works formulate continuous-time approaches using neural ODEs to learn dynamics operators \citep{ChenNeuralODE2018,serrano2023operator,Zhou_2025_change_pde}, with extensions to differential algebraic equations where algebraic components are approximated within neural ODE frameworks \citep{koch2025learningneuraldifferentialalgebraic}. Heavily parameterised foundation models employing transformer-based architectures have enabled learning across multiple physics domains \citep{alkin2024upt, mccabe2023multiple, rahman2024pretraining}. AI/ML frameworks have also led to the development of differentiable physics models, where surrogate models exist within differentiable simulation frameworks, allowing for more hybrid modelling tools \citep{Holl2020Learning,um2021solverinthelooplearningdifferentiablephysics,citrin2024toraxfastdifferentiabletokamak,bhatia2025prdpprogressivelyrefineddifferentiable}.

\section{Background}
\subsection{Partial Differential Equations (PDEs)}
Consider a generic formulation of a PDE modelling the spatio-temporal evolution of $n$ field variables $u\in\mathbb{R}^n$:
\begin{align}
    \pdv{u}{t} = \lambda D_X(u) &, \quad X\in\Omega,\; t\in[0,T], \label{eq:pde} \\
    u(X,t) &= g, \quad X \in \partial \Omega, \label{eq:bc}\\
    u(X,0) &= a, \quad a \in \mathcal{A} \label{eq:ic}.
\end{align}
Here, $X$ defines the spatial domain $\Omega$, $[0,T]$ the temporal domain, and $\pdv{u}{t}$ the temporal gradient. $D_X$ represents the composite spatial derivative operator up to the PDE order. The physics coefficients are expressed via $\lambda$. The PDE has boundary condition $g$ and initial condition $a$ from the function space $\mathcal{A}$. 

\textbf{Our objective} is the forward problem: solving the PDE as an initial value problem under varying coefficients, mathematically expressed as
\begin{align}
\mathcal{G} : \mathcal{P} \rightarrow \mathcal{U}, \label{eq:objective}
\end{align}
where $\mathcal{P}$ represents all characterisations of initial conditions $a$ and PDE parameters $\lambda$, and $\mathcal{U}$ is the space of all PDE solutions over the domain.

\subsection{Autoregressive Neural Operators}
\label{sec: autoregressive_NO}
Neural operators (NOs) learn operator mappings across function spaces, enabling learning in infinite dimensions \citep{Kovachki_NeuralOperators_2023, bartolucci2023representation}. Being discretisation-agnostic, they have found significant applications in mapping PDE initial conditions $a \in \mathcal{A}$ to solutions $u \in \mathcal{U}$ \citep{Li2021fourier}. A neural operator parameterised by $\theta$ learns the solution operator
\begin{align}
    \mathcal{U} = \mathcal{NO_\theta}(\mathcal{A}), \;
    u^t = \mathcal{NO_\theta}(u^{t-dt}, u^{t-2dt}, \ldots, a), \label{eq:NO_ivp}
\end{align}
where $a=u^0$ is the initial condition. NOs couple point-wise estimations $W$ with kernel integration $\kappa$
\begin{equation}
\label{eq:kernel_integration}
    u^{n+1} = \sigma\bigg(Wu^n(x) + \int \kappa_\phi(x, y) \;  u^n(y) dy\bigg),
\end{equation}
combining local linear operator $W$ and non-local integral kernel operator $K$. Discretisation-agnosticism arises from kernel integration operating in continuous function spaces. Parameterising the integral operator via basis functions (Fourier modes, wavelets, Laplace eigenfunctions) yields variants like Fourier NO, Laplace NO, or wavelet NO \citep{Li2021fourier, cao2023lnolaplaceneuraloperator, WaveletNeuralOperatorTripura2023}, each providing complete bases for continuous domains.

Trained autoregressively, NOs rollout field evolution with fixed temporal discretisation $(dt)$, learning the mapping in \cref{eq:NO_ivp} as a discrete-time Markov process \citep{kallenberg1997foundations}. Long-term PDE evolution remains challenging, with significant divergence from ground truth due to cumulative rollout error \citep{Gopakumar_2024, mccabe2023towards, koehler2024apebench, carey2025neuraloperatorsurrogatemodels}.

\subsection{Neural ODEs}
\label{sec: neural_ODE}
Neural ordinary differential equations (Neural ODEs) are neural networks that learn the dynamic evolution of a system's state \citep{ChenNeuralODE2018}. Rather than learning the system state directly, a neural ODE approximates temporal dynamics, the rate of change. Expressed as a neural network $f$ parameterised by $\theta$, it learns the dynamics operator and evaluates the system state by integrating the initial value problem
\begin{align}
    \dv{u}{t} = f_\theta(u,t), \qquad u(T) = u(0) +  \int_{0}^{T} f_\theta(u(t), t)dt. \label{eq:node}
\end{align}

Neural ODEs enable continuous-time dynamical modelling \citep{kidger2022neuraldifferentialequations} and serve as universal differential equation solvers for interdependent ODEs \citep{rackauckas2021universaldifferentialequationsscientific}. Recently, neural operators framed as neural ODEs have achieved improved accuracy and stability for PDEs compared to autoregressive approaches \citep{Zhou_2025_change_pde}. Any differentiable architecture satisfying the universal approximation theorem can be deployed as a neural ODE. Extensions include augmented neural ODEs for better expressivity \citep{dupont2019augmentedneuralodes}, neural SDEs \citep{kidger2021neuralsdesinfinitedimensionalgans}, and continuous graph representations \citep{poli2021graphneuralordinarydifferential}. While enabling continuous function modelling, neural ODEs incur additional computational costs from ODE solving.

\subsection{Operator Splitting}
\label{sec: operator-splitting}
Operator splitting methods decompose complex PDEs into simpler sub-problems solvable sequentially or in parallel \citep{holden2010splitting}. The composite operator $D_X$ in \cref{eq:pde} is split into multiple operators representing different physical processes (e.g., convection, diffusion) that dominate field evolution within a bounded domain \citep{hormander1983analysis}. The decomposition pairs terms with specialised numerical techniques—advection operators with characteristic methods, diffusion operators with implicit schemes—while maintaining computational efficiency and solution accuracy. This requires domain knowledge of the physics and numerical method coupling. There may exist several ways of performing operator-splitting for a PDE requiring extensive experimentation. 

Consider decomposing the spatial operator into $j$ linear and $k$ non-linear components
\begin{align}
    D_X(u) &= D_{l_1}(u) + D_{l_2}(u) + ... + D_{l_j}(u) \nonumber \\
            & +D_{nl_1}(u) + D_{nl_2}(u) + ... + D_{nl_k}(u) , \label{eq:operator_split}
\end{align}
where each $D_i$ represents a distinct component. This rewrites the PDE from \cref{eq:pde} as
\begin{align}
    \pdv{u}{t}  = \left(\sum_{i=1}^{j} \lambda_{l_i} D_{l_i}(u^n) + \sum_{i=1}^{k} \lambda_{nl_i} D_{nl_i}(u^n)\right),
\label{eq:split_pde}
\end{align}
where $\lambda_{l_i}$ and $\lambda_{nl_i}$ are coefficients for linear and non-linear components, respectively. Each sub-problem is solved independently over small time steps $\Delta t$ using different numerical methods optimised for each operator's characteristics. Methods range from first-order (Godunov) to higher-order (Strang) splitting \citep{MacNamara2016}. Decomposing linear and non-linear components enables asynchronous resource allocation, dedicating more resources to complex non-linear operators.

\section{Method: OpsSplit}
\label{sec: method}

Numerical PDE solvers typically follow a discretise-then-optimise approach \citep{ghobadi2009discretize}, approximating spatial operators in discretised domains and solving the resulting ODEs via time integration. Neural network-based solvers like PINNs follow the inverse structure: optimise-then-discretise. Neural Operators arguably follow this paradigm \citep{furuya2024can}, with parameters optimised across fixed or varying discretisations during training \citep{george2024incremental}, enabling discretisation-agnostic predictions. We propose a hybrid approach merging both paradigms.

The generic PDE in \cref{eq:pde} can be decomposed into linear and non-linear physical operators as in \cref{eq:operator_split}, yielding \cref{eq:split_pde}. Rather than traditional finite difference, integral transform, or polynomial approximations, we use neural operators $\mathbb{NO}$ for non-linear operators and convolutions with finite difference stencils $\mathbb{FD}$ for linear operators \citep{Actor2020-ng,CHEN2024116974,chen2024usingailibrariesincompressible,gopakumar2025calibrated}, converting \cref{eq:split_pde} to:
\begin{align}
    \dv{u}{t} &=\lambda_{l_1} \mathbb{FD}_1(u) + \lambda_{l_2} \mathbb{FD}_2(u) + ... + \lambda_{l_j} \mathbb{FD}_j(u) \nonumber\\
    &+ \lambda_{nl_1} \mathbb{NO}_1(u) + \lambda_{nl_2} \mathbb{NO}_2(u) + ... + \lambda_{nl_k} \mathbb{NO}_k(u). \label{eq:no_split_pde}
\end{align}

Unlike autoregressive neural operators (\cref{sec: autoregressive_NO}) that map function spaces without explicit PDE structure, or neural ODEs (\cref{sec: neural_ODE}) that remain spatio-temporally continuous but approximate the entire $D_X$ using a single operator \citep{Zhou_2025_change_pde}, our formulation (\cref{eq:no_split_pde}, \cref{fig:framework}) expresses a neural ODE as a linear combination of neural operators, each learning specific physical operators. This Mixture of Experts (MoE) approach \citep{eigen2014learningfactoredrepresentationsdeep, shazeer2017} dedicates expert neural operators to individual physical phenomena, enabling better efficiency and performance while facilitating disciplined scaling. Traditional neural operators suffer from spectral bias when a single parameterisation must learn dominant frequency modes across competing physical phenomena \citep{spectral_bias_NNs-rahaman19a, qin2024betterunderstandingfourierneural, NO_SpectralBIasKhodakarami2026}. Dedicating separate operators to each phenomenon provides better expressibility and flexibility—operators can be added or removed with changing physics. 

By approximating non-linear spatial operators with neural operators, our method parallels spectral PDE methods \citep{polyanin1998handbook}, producing an ODE in time once spatial gradients are approximated. The linear combination of neural operators approximates the PDE's RHS, enabling time integration via ODE solvers such as explicit Euler or Runge-Kutta methods \citep{Ascher1997ImplicitexplicitRM}. For Euler integration with $j$ linear and $k$ non-linear operators:
\begin{equation}
    u^{n+1} = \left(\sum_{i=1}^j \lambda_i \mathbb{FD}_i(u^n) + \sum_{i=1}^k \lambda_i \mathbb{NO}_i(u^n)\right) \Delta t + u^n, 
    \label{eq:abstract_nox_pde_euler}
\end{equation}
where the prediction formulation implicitly enforces PDE constraints. This constitutes a novel physics-informed approach to operator learning, structuring the differential PDE form into training and inference rather than into loss functions or architectures. As demonstrated in \cref{sec: experiments}, this framework generalises to unseen physics and novel PDE coefficient regimes. Operators are trained via supervised learning using relative LP-Loss \citep{kossaifi2024neural}:
\begin{equation}
    \mathcal{L}_{\text{rel}}(\hat{u}^{n+1}, u^{n+1}) = \frac{\|\hat{u}^{n+1}- u^{n+1}\|_p}{\|u^{n+1}\|_p + \epsilon},
\end{equation}
where $\hat{u}^{n+1}, u^{n+1}$ denote prediction and ground truth at the $(n+1)$-th time instance.

\subsection{Considerations for neural operator splitting}
In neural settings, operator splitting requires balancing physical fidelity with computational cost. Assigning distinct Neural Operators (NOs) to each physical process substantially increases parameterisation and memory overhead. To minimise redundancy, the strategy should target vector operations invariant across the PDE family. We explore these trade-offs further with ablation studies across OpsSplit methods in \cref{appendix:operator_splitting}. Furthermore, leveraging the fine-tuning capabilities demonstrated in \cref{appendix:convergence}, the design should prioritise modular, reusable operators that transfer across related physical systems.

\section{Experiments}
\label{sec: experiments}

\begin{table*}[ht]
\centering
\resizebox{\textwidth}{!}{
\begin{tabular}{cccccccccc}
\toprule
& & & & & \textbf{Train Time} & \multicolumn{4}{c}{\textbf{NRMSE}} \\
\cmidrule(lr){7-10}
\textbf{Operator Learning} & \textbf{Model} & \textbf{Time Stepping} & \textbf{OpsSplit} & \textbf{Parameters} & \textbf{(hrs:mins)} & \textbf{Test} & \textbf{t-extrapolate} & \textbf{OOD} & \textbf{OOD+t-extrapolate} \\
\midrule
Solution & FNO & Autoregressive & False & 13437314 & 1:03 & 0.2770 $\pm$ 0.0491 &  0.6860 ± 0.26476 & 0.4821 ± 0.0639  &   0.8986 ± 0.2727   \\
Dynamics & FNO & Euler & False & 13437314 & 1:04 & 0.0325 $\pm$ 0.0165  & \textbf{0.0536 ± 0.0246} & 0.1276 ± 0.0081   &   0.3038 ± 0.0119  \\
Physical & FNO & Euler & True & 13437314 & 1:06 & \textbf{0.0297 $\pm$ 0.0017}  &  0.0630 ± 0.0111  & \textbf{0.1160 ± 0.0162 } &  \textbf{0.2185 ± 0.0846}  \\    
\midrule 
Solution & U-Net & Autoregressive & False & 31384322 & 1:22 & 0.9508 ± 0.1369  & 1.3183 ± 0.1318 &  0.8192 ± 0.2007  & 1.1872 ± 0.2534    \\
Dynamics & U-Net  & Euler & False & 31384322 & 1:22 & 0.1006 ± 0.0133  & 0.2212 ± 0.0315  & 0.2823 ± 0.0154  &  0.5135 ± 0.0331\\
Physical & U-Net & Euler & True & 31384322 & 1:26 &  \textbf{0.0887 ± 0.0116}   & \textbf{0.1996 ± 0.0190} & \textbf{ 0.2705 ± 0.0060 } & \textbf{ 0.4550 ± 0.0181 } \\
\midrule
Solution & ViT & Autoregressive & False & 11405510 & 3:42 & 0.2335 ± 0.0468 & 0.4234 ± 0.0468 & 0.7505 ± 0.0713 &  0.8996 ± 0.0647 \\
Dynamics & ViT & Euler & False & 11405510 & 3:43 &  0.1034 ± 0.0081 & 0.2179 ± 0.0191 & 0.3747 ± 0.0199  &  0.5984 ± 0.0699 \\
Physical & ViT & Euler & True & 11405510 & 3:45 & \textbf{0.0157 ± 0.0008} &  \textbf{0.0822 ± 0.0016} & \textbf{0.3065 ± 0.0125} &  \textbf{0.4742 ± 0.0195} \\
\midrule
Solution & CNO & Autoregressive & False & 1420458 & 9:43 & 0.5094 $\pm$ 0.0611 &  0.9002 $\pm$ 0.1620 & 0.5831 $\pm$ 0.0875 &  3.1951 $\pm$ 0.3515 \\
Dynamics & CNO & Euler & False & 1420458 & 9:45 &  0.2985 $\pm$ 0.0478 & 0.7122 $\pm$ 0.0997 & 0.4387 $\pm$ 0.0834 & 0.9055 $\pm$ 0.1177 \\
Physical & CNO & Euler & True & 1420458 & 9:52 & \textbf{0.2797 $\pm$ 0.0475} &  \textbf{0.5200 $\pm$ 0.0624} & \textbf{0.3634 $\pm$ 0.0545} & \textbf{0.6917 $\pm$ 0.1245} \\
\midrule
Solution & UNO & Autoregressive & False & 1397442 & 2:23 &  0.0912 $\pm$ 0.0100 &  0.2856 $\pm$ 0.0543 & 0.6190 $\pm$ 0.0867 &  0.8215 $\pm$ 0.1314 \\
Dynamics & UNO & Euler & False & 1397442 & 2:23 &  0.0141 $\pm$ 0.0018 &  \textbf{0.1747 $\pm$ 0.0297} &  \textbf{0.2986 $\pm$ 0.0358} &  \textbf{0.6221 $\pm$ 0.1120} \\
Physical & UNO & Euler & True & 1397442 & 2:24 & \textbf{0.0130 $\pm$ 0.0020} & 0.1857 $\pm$ 0.0204 & 0.3231 $\pm$ 0.0614 &  0.7076 $\pm$ 0.0991 \\
\bottomrule
\end{tabular}
}
\caption{Incompressible Navier--Stokes: Performance comparison across methods and model architectures. The models and methods are set up to have comparable parameter sizes while exploring a certain architecture. Best performance within each test setting is given in bold. Our method of using neural operators to learn physical operators offers the best performance across most architectures. OOD refers to a different parameterisation of the initial condition and viscosity than that used for training (\cref{appendix:incomp_ns}).}
\label{tab:incomp_ns}
\end{table*}

We evaluate OpsSplit (learning physical operators) against Autoregressive (solution operator) and Neural-ODE (dynamics operator) approaches on incompressible and compressible Navier--Stokes equations. Performance is assessed via Normalised Root Mean Square Error (NRMSE) across four scenarios: (a) test data, (b) temporal extrapolation under identical physics, (c) out-of-distribution (OOD) with novel initial conditions and PDE coefficients, and (d) combined temporal extrapolation with OOD. NRMSE is:
\begin{align}
   \text{NRMSE} =  \sqrt{\frac{\frac{1}{n}\sum_{i=1}^{n}(u_i - \hat{u}_i)^2}{\frac{1}{n}\sum_{i=1}^{n} u_i^2 + \epsilon}},
   \label{eq:nrmse}
\end{align}
where $n$ is the number of data points, $\epsilon = 1e-6$ prevents division by zero, and $\hat{u}_i, u_i$ denote predictions and targets.



We test multiple architectures—Fourier Neural Operator (FNO) \citep{Li2021fourier}, U-Net \citep{ronneberger2015unet, gupta2023towards}, Vision Transformer (ViT) \citep{dosovitskiy2021imageworth16x16words, herde2024poseidon}, U-shaped Neural Operator (UNO) \citep{rahman2023unoushapedneuraloperators}, and Convolutional Neural Operator (CNO) \citep{bartolucci2023representation}—within each deployment method. While architecture choice impacts performance, our architecture-agnostic study focuses on deployment strategy comparisons. Performance should be compared across methods (autoregressive, neural-ODE, OpsSplit) rather than architectures. Given cost-accuracy tradeoffs \citep{dehoop2022costaccuracytradeoffoperatorlearning}, ablation studies in the appendices use FNOs exclusively.

\paragraph{Training:} All models train for 250 epochs using Adam optimiser \citep{adam} with initial learning rate 0.001, decaying by half every 50 steps. Models use LP-Loss \citep{Gopakumar_2024} on single Nvidia H100 GPUs with matched parameter counts (within each method) for fair comparison. Each experiment uses 100 PDE simulations with varied initial conditions (data efficiency ablations are demonstrated in \cref{appendix:data_efficiency}). Models predict $u^{t+1}$ from $u^t$ but train by sequentially predicting 5 steps ($u^{t+5}$) before backpropagation. Balancing temporal learning capability against memory/compute costs, we fix the rollout length at 5 \citep{koehler2024apebench}, enabling consistent comparison (see \cref{appendix:rollout_length} for ablations). Linear range normalisation maps field values to [-1, 1]; PDE coefficients in OpsSplit use matching normalisation for physical coherence.

\subsection{Incompressible Navier--Stokes}
\label{sec:exp_incomp}
\begin{figure*}[!ht]
    \centering
    \begin{subfigure}{0.24\textwidth}
        \centering
            \includegraphics[width=\textwidth]{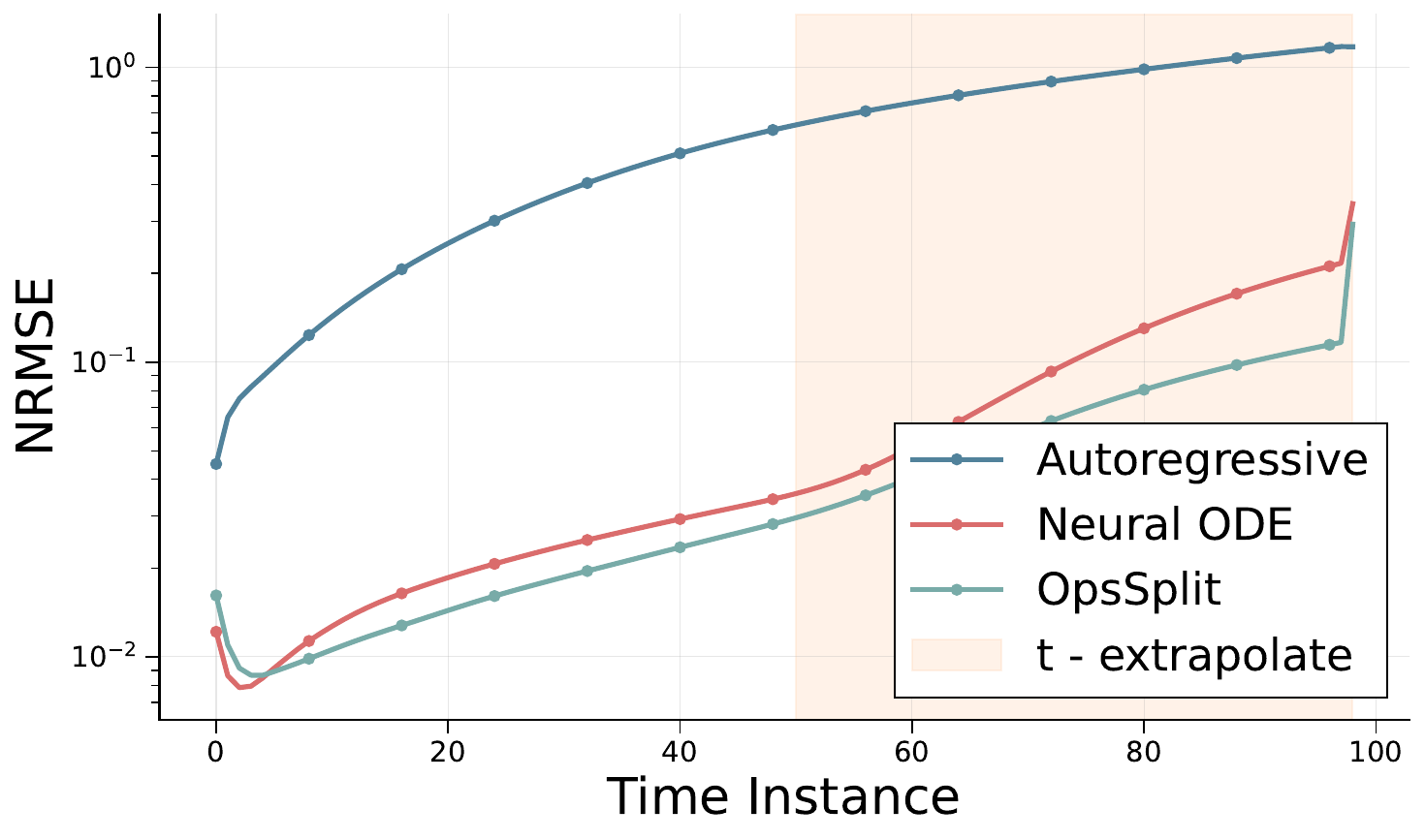}
        \caption{FNO: in-distribution}
        \label{fig:nrmse_incomp_100_iid_fno}
    \end{subfigure}
    \hfill
    \begin{subfigure}{0.24\textwidth}
        \centering
        \includegraphics[width=\textwidth]{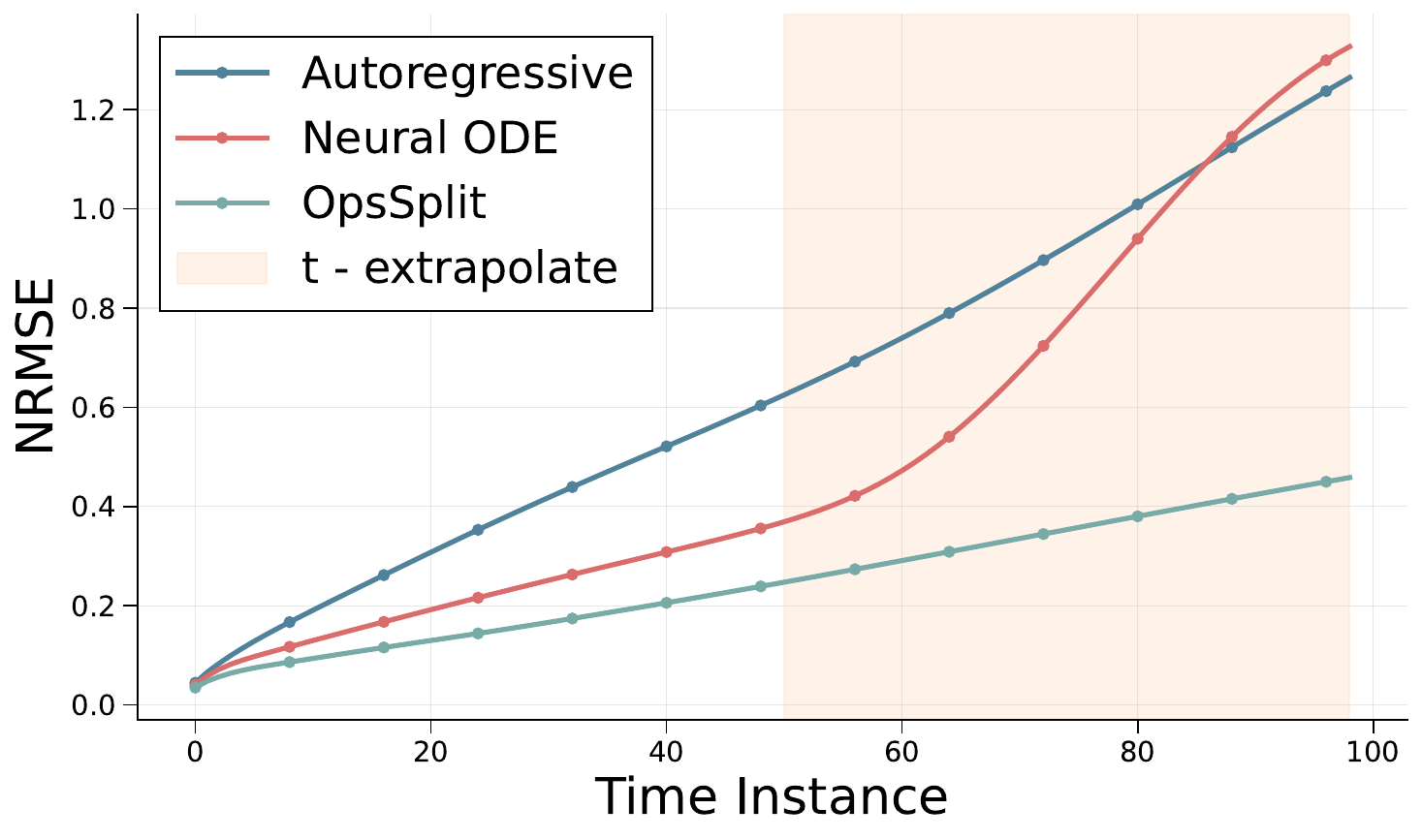}
        \caption{FNO: out-of-distribution}
        \label{fig:nrmse_incomp_100_ood_fno}
    \end{subfigure}
    \hfill    
    \begin{subfigure}{0.24\textwidth}
        \centering
        \includegraphics[width=\textwidth]{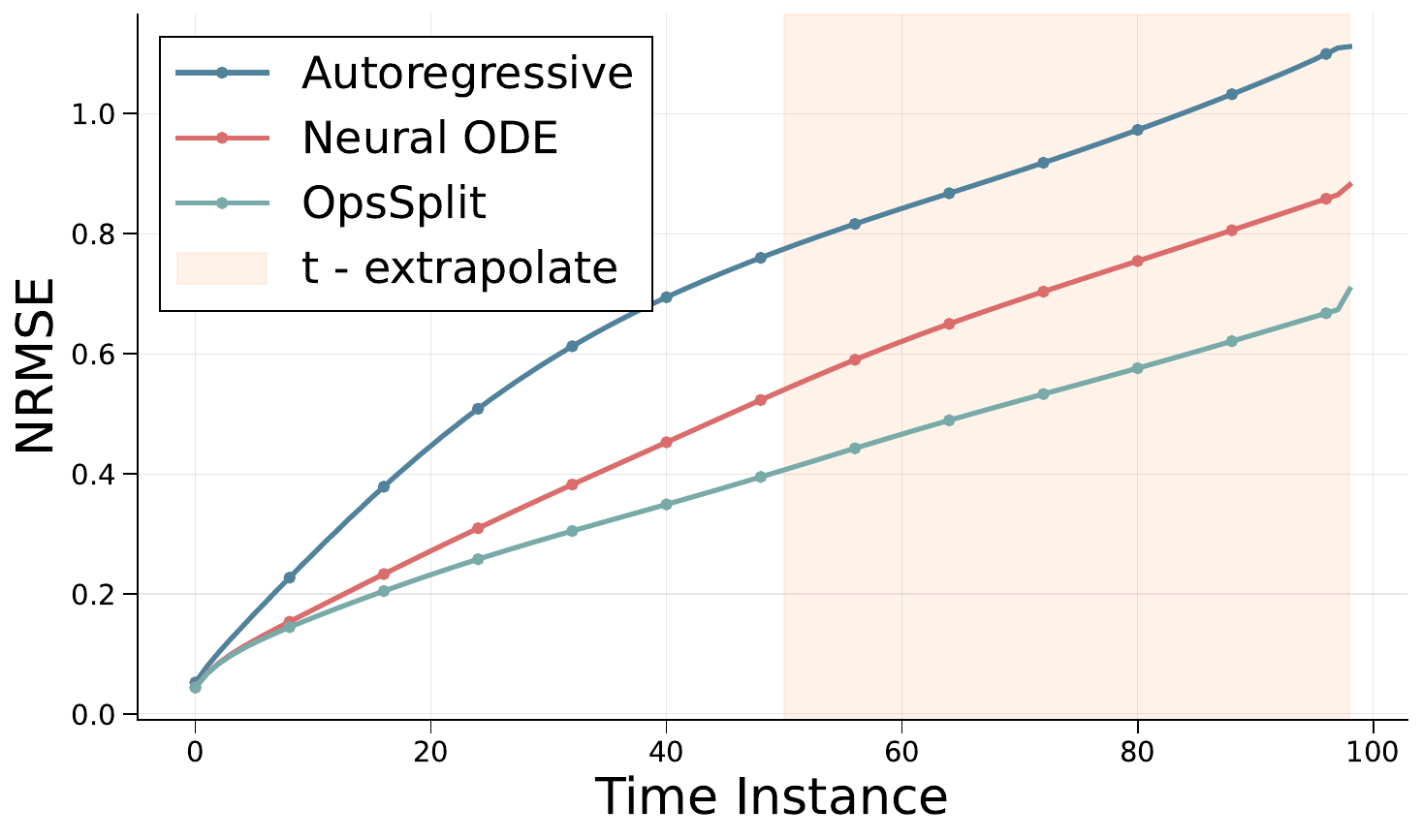}
            \caption{CNO: in-distribution}
        \label{fig:nrmse_incomp_100_iid_cno}
    \end{subfigure}
    \hfill
    \begin{subfigure}{0.24\textwidth}
        \centering
        \includegraphics[width=\textwidth]{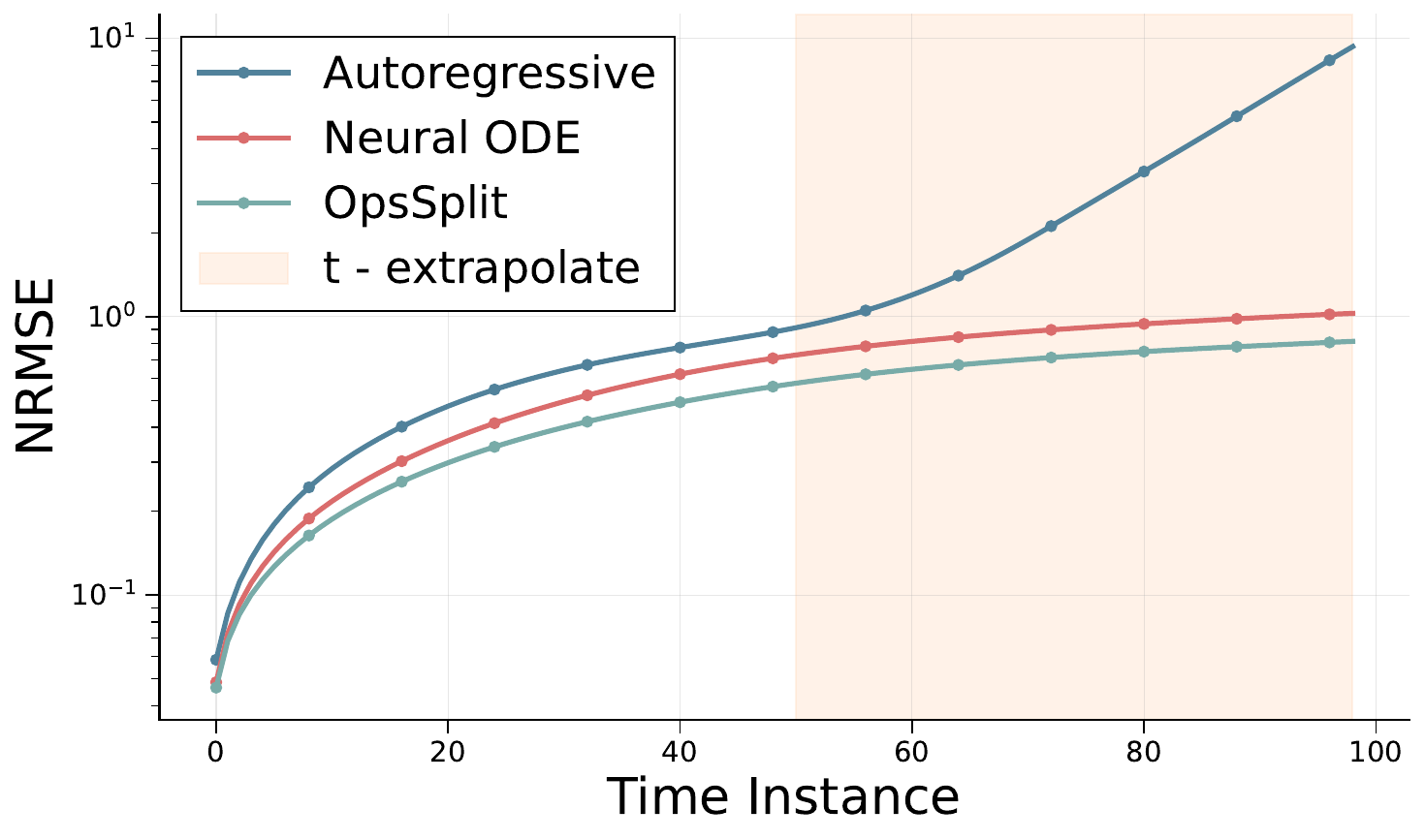}
        \caption{CNO: out-of-distribution}
        \label{fig:nrmse_incomp_100_ood_cno}
    \end{subfigure}
    \caption{Rollout error for incompressible Navier--Stokes equations. FNO (\cref{fig:nrmse_incomp_100_iid_fno,fig:nrmse_incomp_100_ood_fno}) and CNO (\cref{fig:nrmse_incomp_100_iid_cno,fig:nrmse_incomp_100_ood_cno}) predictive error growth for in-distribution and out-of-distribution cases. The temporal extrapolation region is shaded orange. OpsSplit accumulates fewer errors and provides more stable temporal rollout than autoregressive and neural ODE methods across architectures and scenarios.} 
    \label{fig:incomp_rollout_error}
\end{figure*}

The incompressible Navier--Stokes equations are fundamental to fluid dynamics applications from aerodynamics and weather prediction to blood flow and ocean currents \citep{Kwak_CFD_NASA_2009,BauerNWP2015,Zhang_blood_flow_incompNS2025}. These equations pose significant computational challenges due to non-linear convection operators and coupled pressure-velocity relationships through continuity constraints, making them ideal for evaluating neural operators' multi-scale physics capture and long-term stability. Consider the two-dimensional incompressible Navier--Stokes equations:
\begin{align}
    \mathbf{\nabla} \cdot \mathbf{v} &= 0,  \label{eq:ns_cont} \\
    & \hfill \text{(Continuity equation)} \nonumber \\[1ex]
    \pdv{\mathbf{v}}{t} + (\mathbf{v} \cdot \nabla) \mathbf{v}  &= \nu \nabla^2 \mathbf{v} - \nabla P, \label{eq:ns_mom} \\
    & \hfill \text{(Momentum equation)} \nonumber
\end{align}
where $\mathbf{v}=[u,v]$ is the velocity vector of an incompressible fluid with kinematic viscosity $\nu$ under periodic boundary conditions. The system is dominated by the non-linear convection operator $(\mathbf{v} \cdot \mathbf{\nabla})$ accounting for momentum transport, and the linear diffusion operator $\nabla^2$ (Laplacian) characterising viscous momentum diffusion. Pressure derives from velocity via the pressure Poisson formulation (\cref{appendix:incomp_ns}).

OpsSplit decomposes \cref{eq:ns_mom} into neural and linear components: $\mathbb{NO}_{conv}$ characterises combined convection and pressure Poisson effects, while $\mathbb{FD}_{\nabla^2}$ approximates diffusion via linear convolution with finite difference stencil kernels \citep{gopakumar2025calibrated}. Linear operators leverage well-established finite difference approximations as fixed convolutional kernels, exploiting natural correspondences between discrete differential stencils and spatial convolutions \citep{Actor2020-ng,CHEN2024116974,chen2024usingailibrariesincompressible}, reserving neural capacity for non-linear physics while hard-coding known structures parameter-efficiently. An ablation study across various approximation methods for the linear operators are given in \cref{appendix:operator_splitting_guidelines}. 

\begin{align}
    \dv{\mathbf{v}}{t} &= - \mathbb{NO}_{conv} (\mathbf{v})+ \nu \mathbb{FD}_{\nabla^2}(\mathbf{v}) \label{eq:ns_mom_NO} 
\end{align}

\Cref{eq:ns_mom_NO} reformulates momentum as a neural ODE explicitly separating non-linear (convection via $\mathbb{NO}_{conv}$, including pressure Poisson effects from \cref{eq:pressure_poisson}) and linear physics (diffusion via fixed stencil $\mathbb{FD}_{\nabla^2}$). This converts the PDE to ODEs integrable via standard solvers. During training, velocity fields pass through both operators; their weighted combination (scaled by $\nu$) provides $\dv{\mathbf{v}}{t}$, integrated forward for predictions. Training data generated by solving \cref{eq:ns_cont,eq:ns_mom} on $x \in [0,1], y \in [0,1], t \in [0, 0.5]$ via spectral solver \citep{canuto2007spectral}. Temporal extrapolation tests use data evolved to $t=1.0$; OOD tests use different initial condition parameterisations and viscosity conditions. See \cref{appendix:incomp_ns} for physics, solver, parameterisation, and model details.

\begin{figure}[!ht]
    \centering
    \begin{subfigure}{\linewidth}
        \centering
            \includegraphics[width=0.87\textwidth]{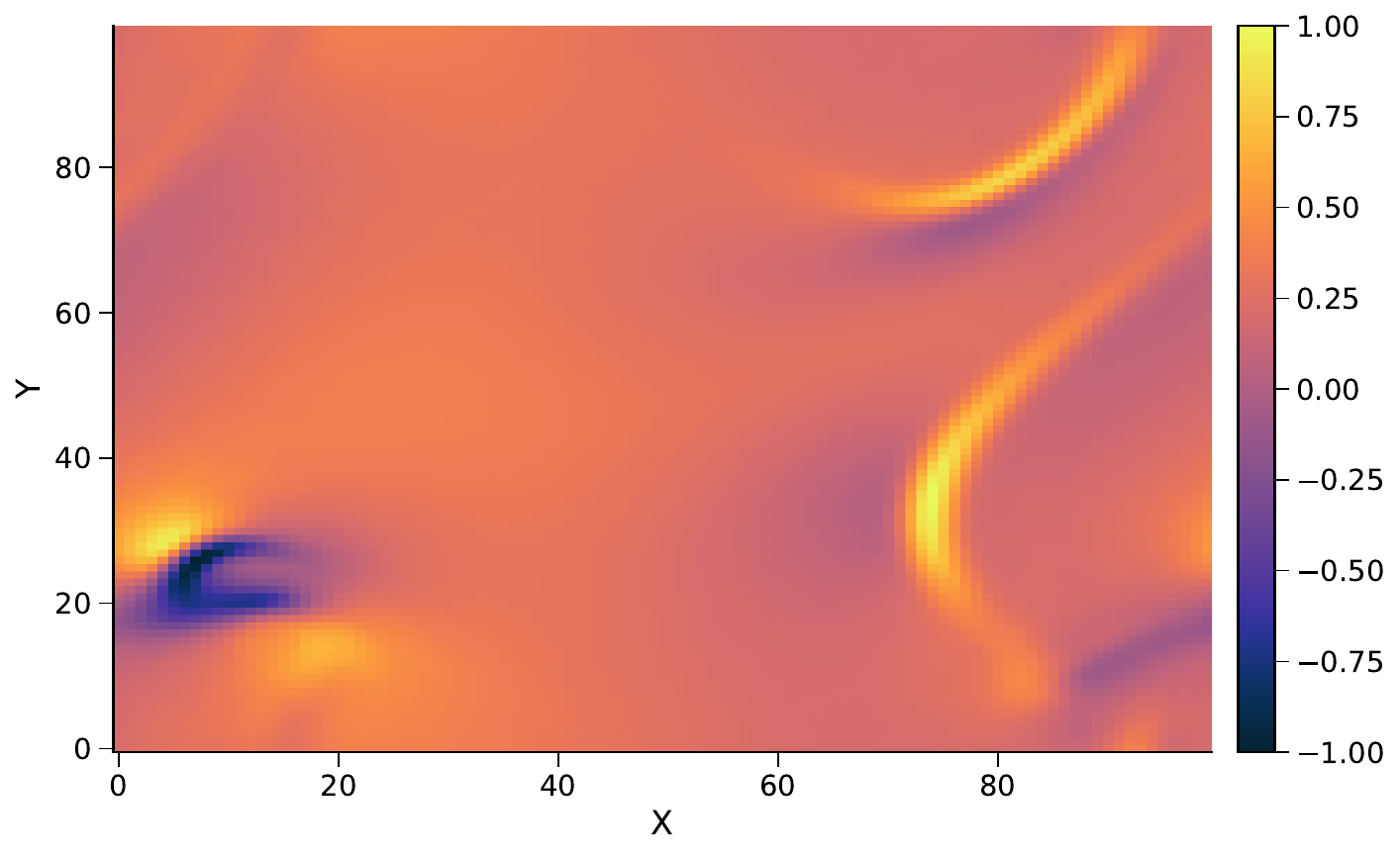}
        \caption{Convection Operator: Neural}
        \label{fig:convops_neural}
    \end{subfigure}
    \hfill
    \begin{subfigure}{\linewidth}
        \centering
        \includegraphics[width=0.87\textwidth]{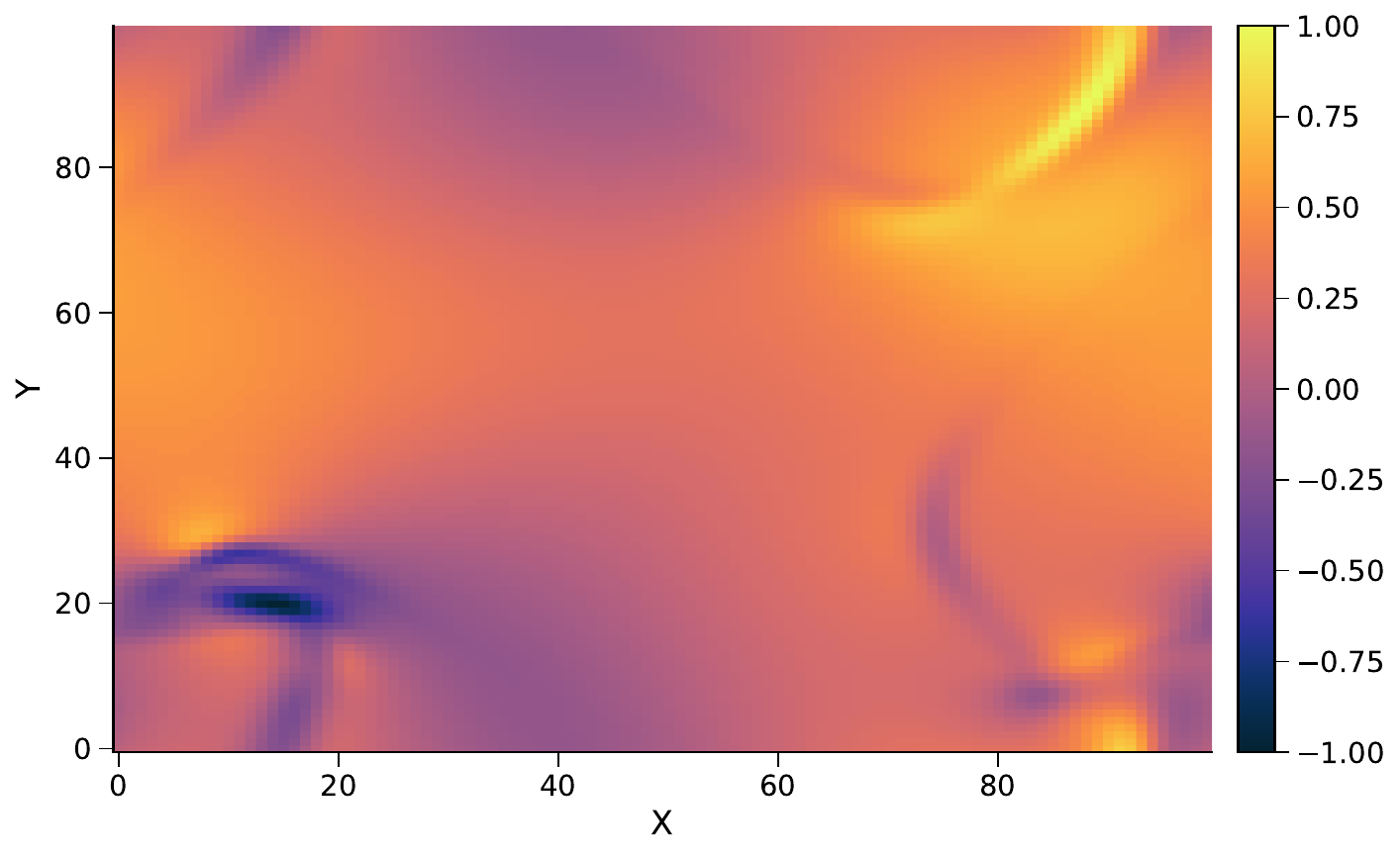}
        \caption{Convection Operator: Numerical}
        \label{fig:convops_numerical}
    \end{subfigure}
    
    \caption{Neural operator learned convection (\cref{fig:convops_neural}) versus numerical method. NO captures convection nuances and advection-driven flow trends. Qualitative interpretability study: learned convection exists in latent space; numerical convection shown in physical space, normalised to [-1, 1].} 
    \label{fig:convops_interpretability}
    \vspace{-7.5em}
\end{figure}

\Cref{tab:incomp_ns} quantitatively compares OpsSplit (physical) against autoregressive (solution) and neural ODE (dynamics) methods. OpsSplit demonstrates superior performance across architectures, particularly in OOD scenarios with novel physics and initial conditions. \Cref{fig:incomp_rollout_error} shows FNO and CNO rollout errors across all deployment methods—OpsSplit exhibits the lowest error growth in both in-distribution and OOD scenarios. \Cref{fig:convops_interpretability} compares the learned neural convection operator (\cref{fig:convops_neural}) with that utilised in the numerical solver (\cref{fig:convops_numerical}). OpsSplit's NO explicitly learns general behaviour of the physical operator, enabling physics-informed, interpretable models (see \cref{appendix:interpretability} for further studies). \Cref{fig:PRE_cont_FNO_OOD} evaluates physical consistency of each method via physics residual error of the continuity equation (\cref{eq:ns_cont}). Surprisingly, neural ODE struggles with the physics as we extrapolate temporally, while autoregressive and OpsSplit methods strive to keep the inconsistencies to a minimum.

\begin{figure}
    \centering
    \includegraphics[width=0.9\linewidth]{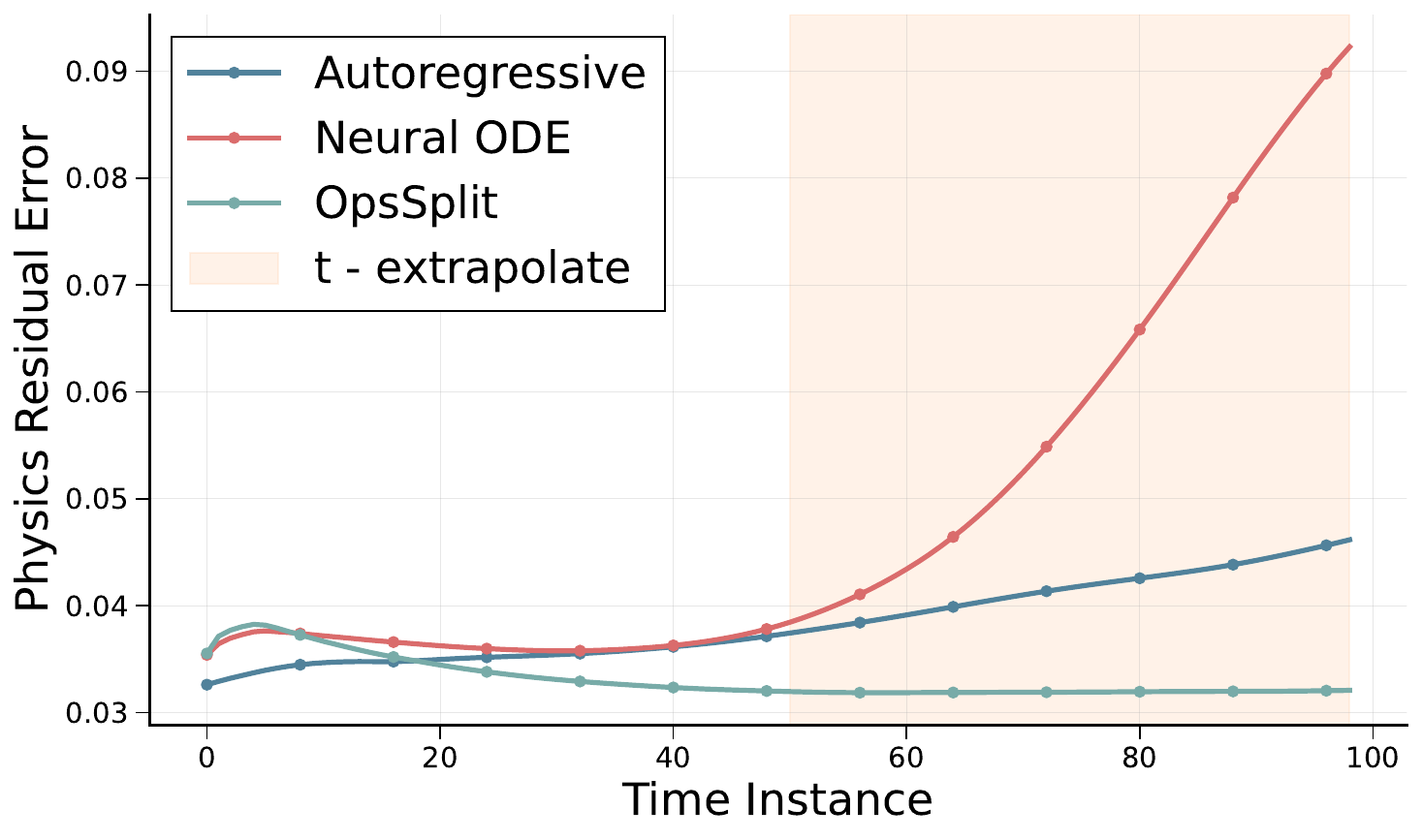}
    \caption{Continuity equation (\cref{eq:ns_cont}) violation across FNO predictions for OOD modeling.}
    \label{fig:PRE_cont_FNO_OOD}
\end{figure}

\subsection{Compressible Navier--Stokes}
\label{sec:exp_comp}

\begin{table*}[ht]
\centering
\resizebox{\textwidth}{!}{
\begin{tabular}{cccccccccc}
\toprule
& & & & & \textbf{Train Time} & \multicolumn{4}{c}{\textbf{NRMSE}} \\
\cmidrule(lr){7-10}
\textbf{Operator Learning} & \textbf{Model} & \textbf{Time Stepping} & \textbf{OpsSplit} & \textbf{Parameters} & \textbf{(hrs:mins)} & \textbf{Test} & \textbf{t-extrapolate} & \textbf{OOD} & \textbf{OOD+t-extrapolate} \\
\midrule

Solution & FNO & Autoregressive & False & 13437828 & 1:53 & 0.0938 $\pm$ 0.0131 & 0.2462 $\pm$ 0.0443 & 0.1209 $\pm$ 0.0133 & 0.2792 $\pm$ 0.0530 \\
Dynamics & FNO & Euler & False & 13437828 & 1:56 & \textbf{0.0245 $\pm$ 0.0039} & \textbf{0.0854 $\pm$ 0.0102} & 0.0763 $\pm$ 0.0130 & 0.1185 $\pm$ 0.0154 \\
Physical & FNO & Euler & True & 13454787 & 3:08 & 0.0602 $\pm$ 0.0090 & 0.0959 $\pm$ 0.0182 & \textbf{0.0573 $\pm$ 0.0063} & \textbf{0.0751 $\pm$ 0.0105} \\    
\midrule 
Solution & U-Net & Autoregressive & False & 31385604 & 1:53 & 0.5963 $\pm$ 0.0716 & 1.1914 $\pm$ 0.2025 & 0.6069 $\pm$ 0.1153 & 1.2384 $\pm$ 0.1610 \\
Dynamics & U-Net & Euler & False & 31385604 & 1:54 & 0.1222 $\pm$ 0.0196 & 0.2912 $\pm$ 0.0320 & 0.1477 $\pm$ 0.0222 & 0.2989 $\pm$ 0.0538 \\
Physical & U-Net & Euler & True & 31384322 & 2:47 & \textbf{0.0781 $\pm$ 0.0109} & \textbf{0.1081 $\pm$ 0.0184} & \textbf{0.0799 $\pm$ 0.0096} & \textbf{0.1149 $\pm$ 0.0218} \\ 
\midrule 
Solution & ViT & Autoregressive & False & 11642836 & 7:54 & 0.2337 $\pm$ 0.0351 & 0.4323 $\pm$ 0.0476 & 0.2852 $\pm$ 0.0513 & 0.5109 $\pm$ 0.0664 \\
Dynamics & ViT & Euler & False & 11642836 & 7:55 & \textbf{0.0557 $\pm$ 0.0089} & 0.2761 $\pm$ 0.0387 & 0.0963 $\pm$ 0.0183 & 0.3044 $\pm$ 0.0365 \\
Physical & ViT & Euler & True & 11609798 & 10:13 & 0.0796 $\pm$ 0.0135 & \textbf{0.1296 $\pm$ 0.0143} & \textbf{0.0818 $\pm$ 0.0123} & \textbf{0.1396 $\pm$ 0.0251} \\ 
\midrule 
Solution & UNO & Autoregressive & False & 1398468 & 3:51 & 0.2689 $\pm$ 0.0323 & 0.7733 $\pm$ 0.1315 & 0.3028 $\pm$ 0.0454 & 0.8366 $\pm$ 0.1590 \\
Dynamics & UNO & Euler & False & 1398468 & 3:51 & \textbf{0.0142 $\pm$ 0.0020} & \textbf{0.0660 $\pm$ 0.0073} & 0.0706 $\pm$ 0.0127 & 0.0831 $\pm$ 0.0108 \\
Physical & UNO & Euler & True & 2400563 & 11:09 & 0.0563 $\pm$ 0.0090 & 0.0861 $\pm$ 0.0164 & \textbf{0.0574 $\pm$ 0.0069} & \textbf{0.0752 $\pm$ 0.0113} \\

\bottomrule
\end{tabular}
}

\caption{Compressible Navier--Stokes: Performance comparison across methods and model architectures. The models and methods are set up to have comparable parameter sizes while exploring a certain architecture. Best performance within each test setting is given in bold. Our method of using neural operators to learn physical operators offers the best performance across most architectures. OOD refers to a different parameterisation of the initial condition and adiabatic index than that used for training (\cref{appendix:comp_ns}). The higher computational times for OpsSplit is observed due to the several forward passes required to estimate the physical operators.}
\label{tab:comp_ns}
\end{table*}

\begin{figure*}[!ht]
    \centering
    \begin{subfigure}{0.24\textwidth}
        \centering
            \includegraphics[width=\textwidth]{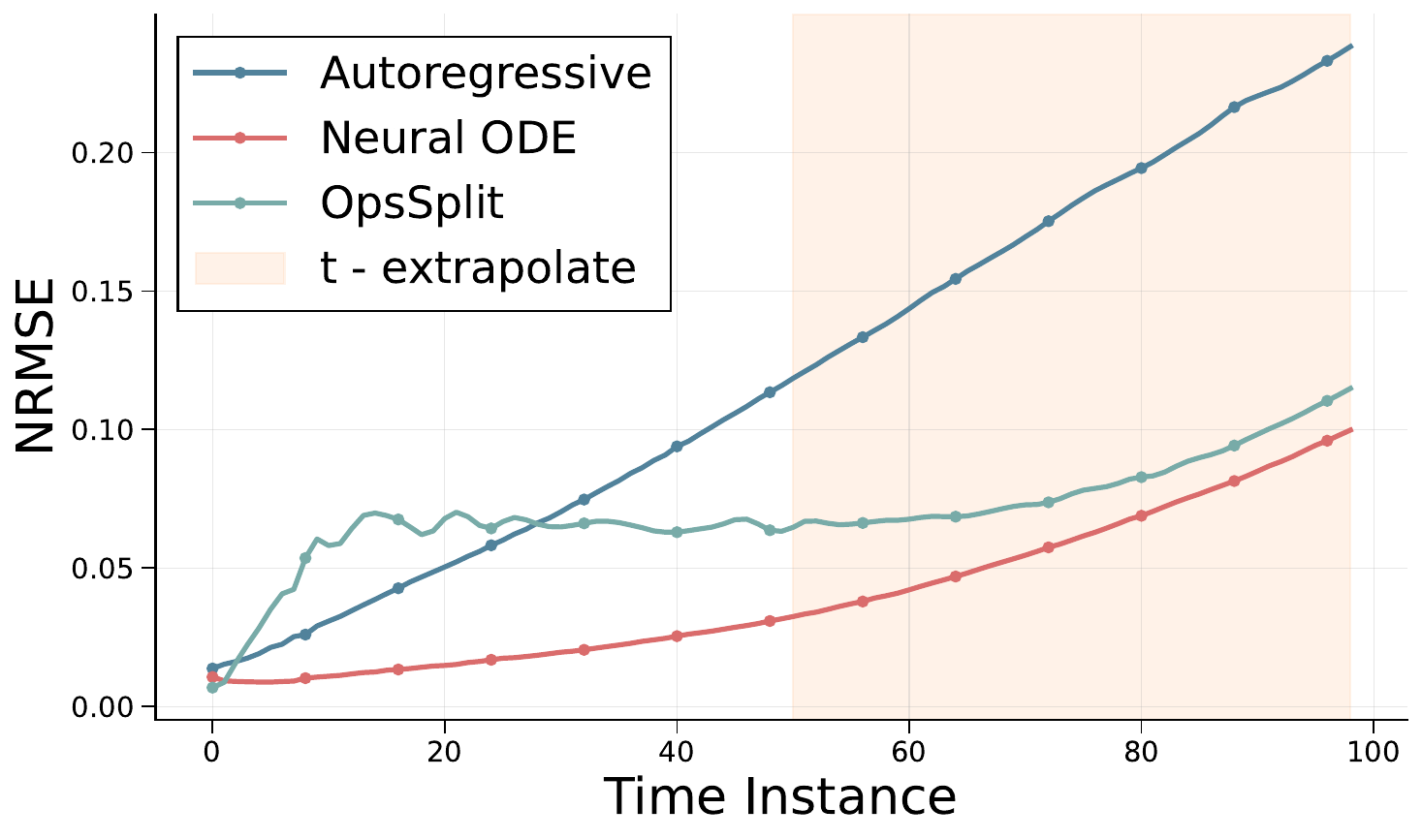}
        \caption{FNO: in-distribution}
        \label{fig:nrmse_comp_100_id_fno}
    \end{subfigure}
    \begin{subfigure}{0.24\textwidth}
        \centering
            \includegraphics[width=\textwidth]{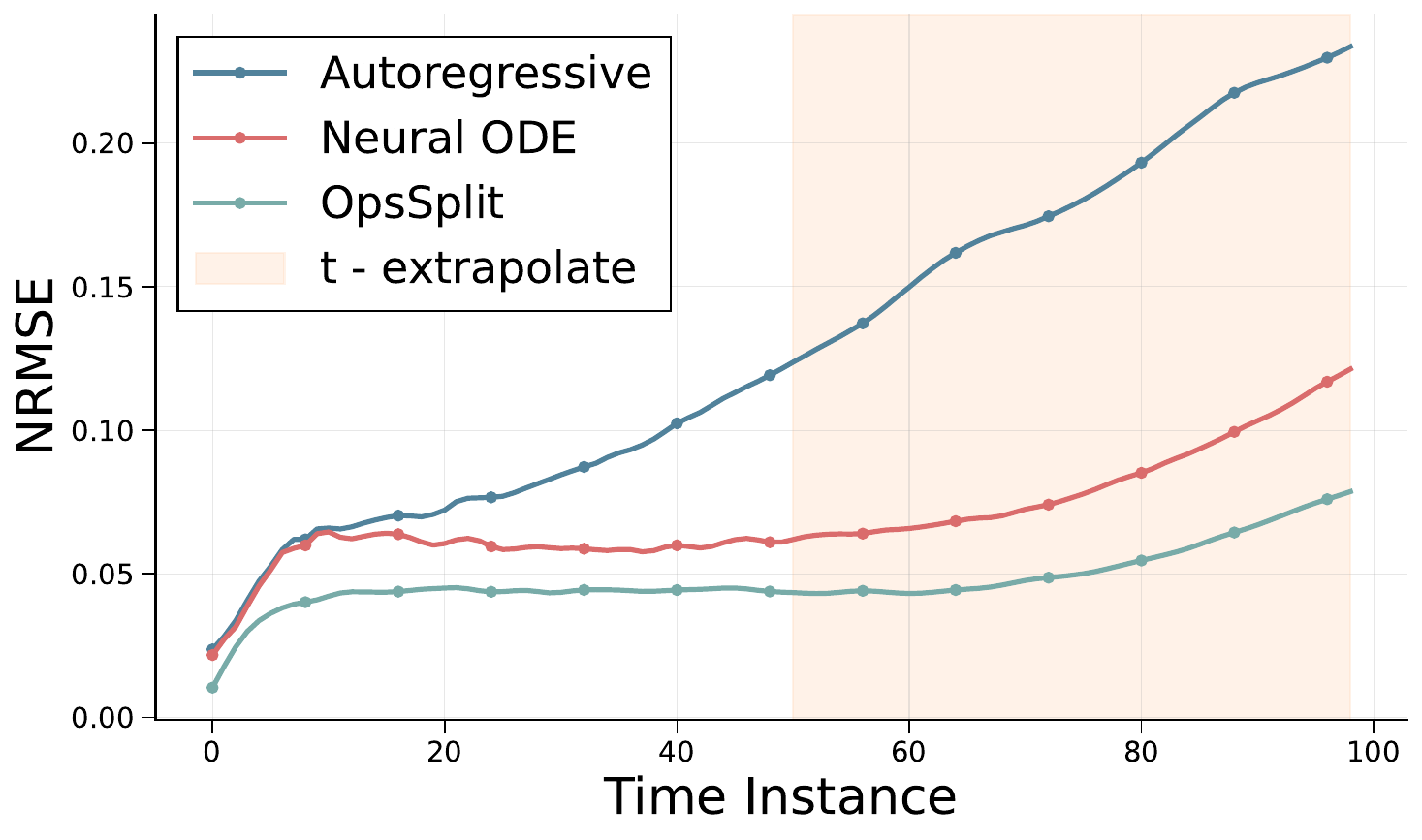}
        \caption{FNO: out-of-distribution}
        \label{fig:nrmse_comp_100_ood_fno}
    \end{subfigure}
    \hfill
    \begin{subfigure}{0.24\textwidth}
        \centering
            \includegraphics[width=\textwidth]{Images/temporal_error_Compressible_Navier-Stokes_fno_MSE_100_OOD.pdf}
        \caption{UNet: in-distribution}
        \label{fig:nrmse_comp_100_id_unet}
    \end{subfigure}
    \hfill
    \begin{subfigure}{0.24\textwidth}
        \centering
        \includegraphics[width=\textwidth]{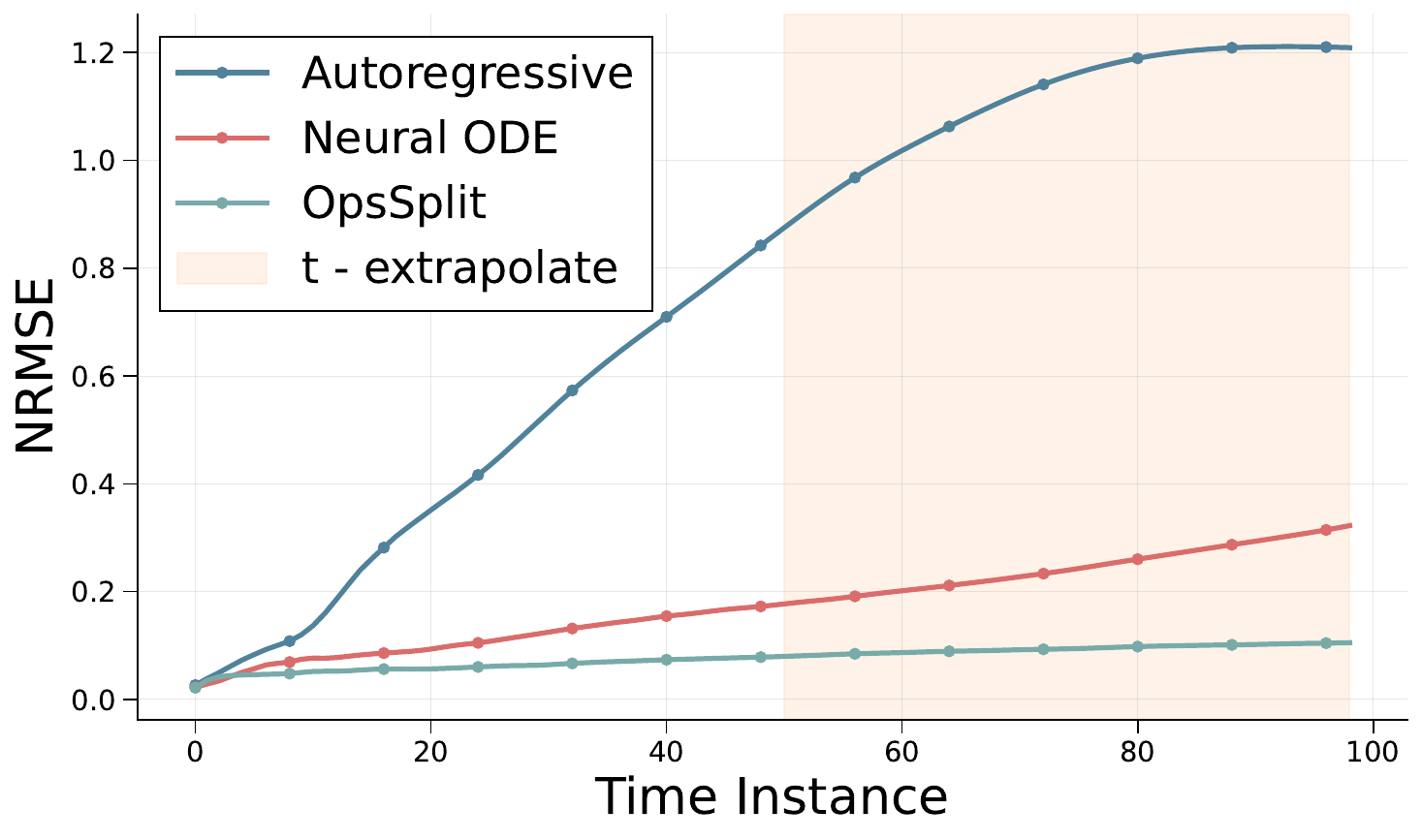}
        \caption{UNet: out-of-distribution}
        \label{fig:nrmse_comp_100_ood_unet}
    \end{subfigure}
    
    \caption{Rollout error for the compressible Navier--Stokes equations. Similar to \cref{fig:incomp_rollout_error}, show the rollout error of various methods for an FNO and U-Net for both in and out-of-distribution scenarios. In most cases, across neural operator architectures, our method of deploying operator splitting to learn the physical operator accumulates less error and provides a more stable temporal rollout than autoregressive and neural ODE-based methods.} 
    \label{fig:comp_rollout_error}
\end{figure*}


The compressible Navier--Stokes equations are essential for modelling high-speed flows in aerospace, shock waves, and astrophysical phenomena where density variations are significant \citep{shockwaves_CNS1986,astroCNS1998}. This system challenges neural PDE solvers through tight density-velocity-pressure coupling. They explore discontinuities and shock formation while preserving conservation laws across multiple interacting physical processes. Consider the Euler formulation of the compressible Navier--Stokes equations, which govern adiabatic, inviscid fluid flow:
\begin{align}
    \frac{\partial \rho}{\partial t} &= -\nabla \cdot (\rho\mathbf{v}), \label{eqn:comp_mass} \\
    & \hfill \text{(Mass conservation / Continuity)} \nonumber \\[1ex]
    \frac{\partial \mathbf{v}}{\partial t} &= -(\mathbf{v} \cdot \nabla)\mathbf{v} - \frac{1}{\rho}\nabla P, \label{eqn:comp_momentum} \\
    & \hfill \text{(Momentum conservation)} \nonumber \\[1ex]
    \frac{\partial P}{\partial t} &= -\mathbf{v} \cdot \nabla P - \gamma P(\nabla \cdot \mathbf{v}). \label{eqn:comp_pressure} \\
    & \hfill \text{(Pressure evolution)} \nonumber
\end{align}
These equations model density $\rho$, velocity $\mathbf{v}=[u,v]$, and pressure $P$ evolution for a gas with adiabatic index (specific heat ratio) $\gamma$ under periodic boundary conditions. Density evolution applies divergence to momentum; convection and pressure gradient terms control momentum. Density-weighted pressure terms ensure greater acceleration in low-density regions. Energy (pressure) depends on pressure advection $(\mathbf{v} \cdot \nabla P)$ and compression/expansion $(\gamma P \nabla \cdot \mathbf{v})$, modulating pressure via velocity divergence. 

OpsSplit uses two neural operators to learn vector divergence and convection operations explicitly. The vector divergence operator $\mathbb{NO}_{\nabla \cdot}$ (\cref{eqn:comp_mass_no,eqn:comp_pressure_no}) learns spatial momentum and energy evolution relative to density and pressure, with adiabatic index influence explicitly applied over the divergence of pressure-velocity as shown in\cref{eqn:comp_pressure_no}. The convection operator (\cref{eqn:comp_momentum_no}) learns fluid advection similarly to \cref{eq:ns_mom_NO}. Pressure gradient terms use linear convolution with higher-order finite difference stencils as a fixed kernel $\mathbb{FD}_{\nabla}$. For temporal integration stability, density weighting in \cref{eqn:comp_momentum} uses logarithms to avoid division by zero \citep{Shebalin1993logcompressiblelogpressure, Erlebacher1990}. See \cref{appendix:comp_ns} for physics, solver, parameterisation, and model details.

\begin{align}
    \dv{\rho}{t} &= -\mathbb{NO}_{\nabla \cdot}(\rho, \mathbf{v}) , \label{eqn:comp_mass_no} \\
    & \hfill  \nonumber \\[1ex]
    \dv{\mathbf{v}}{t} &= -\mathbb{NO}_{conv}(\mathbf{v}) - \ln{(\rho)} \, \mathbb{FD}_\nabla (P), \label{eqn:comp_momentum_no} \\
    & \hfill \nonumber \\[1ex]
    \dv{P}{t} &= -\gamma \, \mathbb{NO}_{\nabla \cdot}(P, \mathbf{v}). \label{eqn:comp_pressure_no} \\
    & \hfill  \nonumber
\end{align}

\Cref{tab:comp_ns} demonstrates OpsSplit's superior performance across most architectures. While all methods achieve comparable in-distribution accuracy, OpsSplit outperforms autoregressive and neural ODE approaches when generalising to unseen parameter regimes and extrapolating in time. \Cref{fig:comp_rollout_error} shows OpsSplit maintains substantially lower and more stable error growth throughout temporal evolution in OOD scenarios, stemming from explicit PDE constraint enforcement via operator splitting. Note that despite comparable parameters, OpsSplit's dual NOs increase training time versus autoregressive and neural ODE methods: $1.5\times$ (FNO) to $3\times$ (UNO).

\subsection{Unstructured Neural Operators}
We extend our OpsSplit framework to irregular geometries by replacing the nonlinear operators with graph-based neural operators. These neural operators rely on message passing across an unstructured mesh to learn the evolution of the spatial dynamics within the PDE \citep{li2020neuraloperatorgraphkernel,li2023geometryinformed, brandstetter2022message}. We follow the \textsc{CylinderFlow} experiment from \citet{pfaff2021learning}, where the authors develop a graph neural network to learn the spatio-temporal evolution of vortex shedding as fluid passes over a cylinder. The experiment is governed by the incompressible Navier--Stokes equations given in \cref{eq:ns_cont,eq:ns_mom}. In implementing OpsSplit, we adopt the same splitting as in \cref{eq:ns_mom_NO}, but employ a neural operator for both convection and diffusion (\cref{eq:ns_mom_GNO}). By modelling the linear operators with a learnable graph-based operator, we bypass the non-trivial challenge of implementing finite differences on unstructured meshes.
\begin{align}
    \dv{\mathbf{v}}{t} &= - \mathbb{NO}_{\text{conv}} (\mathbf{v})+ \nu \mathbb{NO}_{\text{diff}} (\mathbf{v}) \label{eq:ns_mom_GNO} 
\end{align}
The \textsc{CylinderFlow} dataset comprises 1000 simulations of fluid flow, all with identical viscosity coefficients but featuring distinct geometries and mesh structures for each simulation. Using AR, NODE, and OpsSplit, we model the spatio-temporal evolution of velocities, training up to the halfway point of each simulation and then extrapolating further during evaluation. 
\begin{table*}[ht]
\centering
\resizebox{\textwidth}{!}{
\begin{tabular}{cccccccc}
\toprule
& & & & & \textbf{Train Time} & \multicolumn{2}{c}{\textbf{NRMSE}} \\
\cmidrule(lr){6-6} \cmidrule(lr){7-8}
\textbf{Operator Learning} & \textbf{Model} & \textbf{Time Stepping} & \textbf{OpsSplit} & \textbf{Parameters} & \textbf{(hrs:mins)} & \textbf{In-Distribution} & \textbf{Out-of-Distribution}\\
\midrule
Solution & GNO & Autoregressive & False & 330,370 & 6:25 & $0.4131 \pm 0.0475$ & $0.5697 \pm 0.0530$ \\
Dynamics & GNO & Euler & False & 330,370 & 6:26 & $0.3276 \pm 0.0398$ & $0.4510 \pm 0.0571$ \\
Physical & GNO & Euler & True & 660,740 & 9:31 & $\mathbf{0.2963\pm 0.0043}$ & $\mathbf{0.3621\pm0.0510}$ \\
\bottomrule
\end{tabular}
}

\caption{Incompressible Navier--Stokes on irregular geometries: Performance comparison across methods and model architectures. Best performance within each test setting is shown in bold. Our method of using neural operators to learn physical operators achieves the best performance. In-distribution refers to modelling within the training regime, whereas out-of-distribution refers to extrapolation further in time on a different validation dataset. The higher computational time for OpsSplit is due to the multiple models required for each physical operator.}
\label{tab:incomp_ns_irreg}
\end{table*}

The results in \cref{tab:incomp_ns_irreg} demonstrate that OpsSplit achieves superior performance on both in-distribution and out-of-distribution test cases compared to both the neural ODE and the autoregressive baseline. This enhanced performance comes at a moderate computational cost, with training time somewhat higher due to the multiple neural operators required. These results confirm that incorporating physical structure through operator splitting translates effectively to unstructured geometries, offering substantial gains in predictive accuracy for complex fluid flow scenarios.

\section{Discussion}

\begin{figure}[H]
\centering
\includegraphics[width=0.90\linewidth]{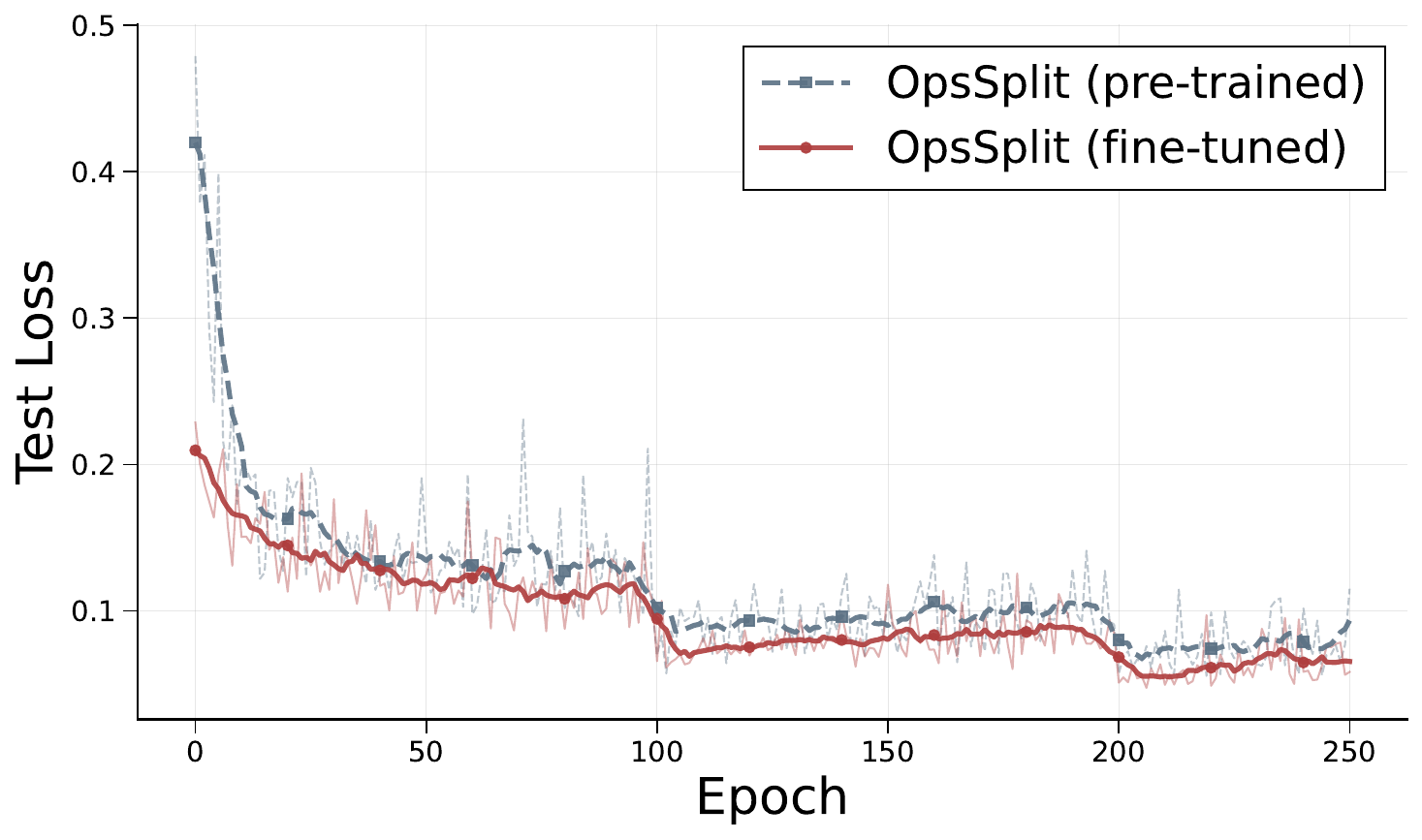}
\caption[Test loss convergence]{Loss convergence for OpsSplit models on incompressible Navier--Stokes. The \textbf{pre-trained} model learns all operators from scratch. The \textbf{fine-tuned} model initialises the convection operator with trained weights from the compressible case, showing quicker convergence. For more details, refer \cref{appendix:convergence}.}
\label{fig:NS_incomp_ft_test_main}
\end{figure}

\paragraph{Strengths} By explicitly learning physical operators, OpsSplit invokes a modular architecture where a specialised neural network approximates each operator. This enables generalisable neural operators reusable across multiple physical scenarios while facilitating efficient transfer learning as shown in \cref{fig:NS_incomp_ft_test_main}. The modular structure supports dynamic insertion and removal of operators as physics changes, and enhances model interpretability by isolating failure regions to specific physical operators. Like neural ODEs, our approach maintains explicit awareness of temporal dynamics, enabling stable long-term rollouts and continuous-in-time predictions that support both interpolation and extrapolation. Our physics-informed decomposition provides a novel constraint enforcement mechanism: since the prediction regime in \cref{eq:abstract_nox_pde_euler} aligns with the PDE definition, it naturally implements soft physical constraints. OpsSplit achieves parameter and computational efficiency by approximating only non-linear operators with neural networks while using linear methods for linear operators, reducing modelling complexity. The decomposition also enables direct parallelisation and efficient training procedures.

\paragraph{Weaknesses} Physics-informed machine learning requires a complete understanding of the underlying equations and domain expertise for operator splitting. Implementation demands explicit model configuration for non-linear operators and convolution setup with finite difference stencils for linear operators, increasing development effort and introducing approximation errors. Critically, no general splitting rule exists, and different splitting of the same problem yield vastly different performances \citep{McLachlanSplitting2002} (refer \cref{appendix:operator_splitting}). When systems possess multiple geometric properties preserved during evolution, different splittings preserve different properties, making it difficult to find one that preserves most. For complex multi-physics settings, structuring individual neural operators for each PDE operator may lead to heavily parameterised models with elevated computational and memory overheads. The ODE formulation incurs higher computational costs than autoregressive methods due to time integration. While we currently explore periodic boundary conditions, the method can be adapted to other boundary conditions by augmented padding schemes \citep{alguacil2021effectsboundaryconditionsfully,mccabe2025walruscrossdomainfoundationmodel}.







\section{Conclusion}
\label{sec:conclusion}
We present a novel method for solving PDEs using neural operators that enforce physical constraints while maintaining continuity in space and time. By learning non-linear physical operators rather than solution operators (autoregressive) or dynamics operators (neural ODE), our approach achieves physics-informed, parameter-efficient modelling with robust out-of-distribution generalisation. The physics-informed, modular design also opens avenues for inverse modelling, system identification, and accelerating the transition from simulation to experimental data, a new paradigm in bridging the \textit{sim2real} gap. 


\clearpage
\bibliography{references}

\clearpage
\section*{Checklist}

\begin{enumerate}

  \item For all models and algorithms presented, check if you include:
  \begin{enumerate}
    \item A clear description of the mathematical setting, assumptions, algorithm, and/or model. [Yes]
    \item An analysis of the properties and complexity (time, space, sample size) of any algorithm. [Yes]
    \item (Optional) Anonymized source code, with specification of all dependencies, including external libraries. [Yes]
  \end{enumerate}

  \item For any theoretical claim, check if you include:
  \begin{enumerate}
    \item Statements of the full set of assumptions of all theoretical results. [Yes]
    \item Complete proofs of all theoretical results. [Yes]
    \item Clear explanations of any assumptions. [Yes]     
  \end{enumerate}

  \item For all figures and tables that present empirical results, check if you include:
  \begin{enumerate}
    \item The code, data, and instructions needed to reproduce the main experimental results (either in the supplemental material or as a URL). [Yes]
    \item All the training details (e.g., data splits, hyperparameters, how they were chosen). [Yes]
    \item A clear definition of the specific measure or statistics and error bars (e.g., with respect to the random seed after running experiments multiple times). [Yes]
    \item A description of the computing infrastructure used. (e.g., type of GPUs, internal cluster, or cloud provider). [Yes]
  \end{enumerate}

  \item If you are using existing assets (e.g., code, data, models) or curating/releasing new assets, check if you include:
  \begin{enumerate}
    \item Citations of the creator If your work uses existing assets. [Yes]
    \item The license information of the assets, if applicable. [Yes]
    \item New assets either in the supplemental material or as a URL, if applicable. [Yes]
    \item Information about consent from data providers/curators.  [Not Applicable]
    \item Discussion of sensible content if applicable, e.g., personally identifiable information or offensive content. [Not Applicable]
  \end{enumerate}

  \item If you used crowdsourcing or conducted research with human subjects, check if you include:
  \begin{enumerate}
    \item The full text of instructions given to participants and screenshots. [Not Applicable]
    \item Descriptions of potential participant risks, with links to Institutional Review Board (IRB) approvals if applicable. [Not Applicable]
    \item The estimated hourly wage paid to participants and the total amount spent on participant compensation. [Not Applicable]
  \end{enumerate}

\end{enumerate}

\clearpage
\appendix
\thispagestyle{empty}

\onecolumn
\aistatstitle{Supplementary Materials}

\section{Qualitative Comparison}
\label{appendix:comparison}
\begin{table}[htbp]
\centering
\label{tab:comparison}
\resizebox{\textwidth}{!}{
\begin{tabular}{@{}p{3.5cm}ccccccc@{}}
\toprule
\textbf{Work (citation)} & \textbf{Physics-informed} &  \textbf{Continuous-in-time } &   \makecell{\textbf{Modular/} \\ \textbf{Mixture of Experts}} &  \textbf{Interpretable} & \makecell{\textbf{Extrapolation /} \\ \textbf{Generalisation}} \\
\midrule
\citep{Li2021fourier} (FNO) & \XSolidBrush & \XSolidBrush & \XSolidBrush & \XSolidBrush  & \XSolidBrush \\
\citep{Lu2021deeponet} (DeepONet) & \XSolidBrush & \Checkmark & \transparentcheck & \XSolidBrush & \XSolidBrush \\
\citep{Raissi2019PINNs} (PINNs) & \Checkmark & \Checkmark & \XSolidBrush & \Checkmark  & \transparentcheck\\ 
\citep{rackauckas2021universaldifferentialequationsscientific} (UDE) & \Checkmark & \Checkmark & \transparentcheck & \XSolidBrush & \XSolidBrush \\ 
\citep{LiPino2024} (PINO) &  \Checkmark & \XSolidBrush & \XSolidBrush & \XSolidBrush & \transparentcheck \\
\citep{koch2025learningneuraldifferentialalgebraic} & \Checkmark & \Checkmark & \Checkmark & \XSolidBrush & \transparentcheck \\
\citep{Zhou_2025_change_pde} & \XSolidBrush & \Checkmark & \XSolidBrush & \XSolidBrush &\XSolidBrush \\
\citep{donati2025kernelbasedapproachphysicsinformednonlinear} &  \Checkmark & \XSolidBrush & \XSolidBrush & \transparentcheck & \XSolidBrush \\
\citep{Zhang_symreg_pde} &  \XSolidBrush & \Checkmark & \XSolidBrush & \Checkmark & \Checkmark \\
\midrule
\textbf{OpsSplit (Ours)} &  \Checkmark & \Checkmark & \Checkmark & \transparentcheck  &\Checkmark \\
\midrule
\bottomrule
\end{tabular}
}
\vspace{0.3cm}
\raggedright
\caption{Comparison of different approaches for Neural PDE solvers. \Checkmark : Yes, \XSolidBrush $\:$: No, \transparentcheck : Partial. }
\label{tab:qualtiative_comparison}
\end{table}

Comparing across standard neural PDE solvers:  Table \ref{tab:qualtiative_comparison} benchmarks the proposed OpsSplit framework against leading neural PDE solver methods, including FNO, DeepONet, and PINNs. We evaluate these methods across five critical dimensions: physical consistency, temporal continuity, architectural modularity, interpretability, and generalisation capability. As the comparison illustrates, existing approaches typically specialise in specific attributes at the expense of others; for instance, while PINNs provide physics-informed continuity, they lack the modular flexibility of Mixture of Experts (MoE) systems. Similarly, standard operator learners like FNO are efficient but struggle with interpretability and temporal flexibility. OpsSplit distinguishes itself by unifying these properties, offering a uniquely holistic framework that is simultaneously physics-informed, modular, and capable of robust extrapolation.

\newpage 
\section{Theoretical Analysis}
\label{appendix:theory}

Consider a PDE governing field variables $\mathbf{u}\in H^s(\Omega;\mathbb{R}^n)$:
\begin{equation}
\frac{\partial \mathbf{u}}{\partial t} = \mathcal{D}_X(\mathbf{u})
= \sum_{i=1}^{j}\lambda_i^{l}\,D_i^{l}(\mathbf{u})
+ \sum_{i=1}^{k}\lambda_i^{nl}\,D_i^{nl}(\mathbf{u}), 
\qquad t\in[0,T].
\end{equation}

We approximate $\mathcal{D}_X$ using the OpsSplit formulation:
\begin{equation}
\widehat{\mathcal{D}}_X(\mathbf{u})
= \sum_{i=1}^{j}\lambda_i^{l}\,\mathrm{FD}_i(\mathbf{u})
+ \sum_{i=1}^{k}\lambda_i^{nl}\,\mathrm{NO}_{\theta_i}(\mathbf{u}),
\end{equation}
where $\mathrm{FD}_i$ are finite-difference operators and $\mathrm{NO}_{\theta_i}$ are neural operators.

We measure errors in $H^{s'}(\Omega)$.

\vspace{1em}
\hrule
\vspace{1em}

\subsection*{1. Assumptions}

\begin{assumption}[Regularity and Lipschitz Continuity]
Each $D_i^{nl}: H^s \to H^{s'}$ is Lipschitz continuous on a compact set $K \subset H^s$, i.e.
\[
\|D_i^{nl}(\mathbf{u}) - D_i^{nl}(\mathbf{v})\|_{s'}
\leq L_i \|\mathbf{u}-\mathbf{v}\|_s.
\]
\end{assumption}

\begin{assumption}[Approximation Classes]
Let $\mathcal{N}(P)$ denote the class of neural operators with at most $P$ parameters. Define the optimal approximation error:
\[
\epsilon_i^*(P)
:= \inf_{\theta:\,\mathrm{NO}_\theta \in \mathcal{N}(P)}
\sup_{\mathbf{u}\in K}
\|D_i^{nl}(\mathbf{u}) - \mathrm{NO}_\theta(\mathbf{u})\|_{s'}.
\]
\end{assumption}

\begin{assumption}[Additive Model Class]
Define the monolithic class:
\[
\mathcal{N}_{\mathrm{mono}}(P)
:= \left\{ \mathrm{NO}_\theta \in \mathcal{N}(P) \right\},
\]
and the split class:
\[
\mathcal{N}_{\mathrm{split}}(P)
:= \left\{
\sum_{i=1}^k \lambda_i^{nl} \mathrm{NO}_{\theta_i} :
\sum_{i=1}^k P_i \le P,\;
\mathrm{NO}_{\theta_i}\in\mathcal{N}(P_i)
\right\}.
\]
\end{assumption}

\begin{assumption}[Finite Difference Consistency]
Each $\mathrm{FD}_i$ satisfies
\[
\|D_i^{l}(\mathbf{u}) - \mathrm{FD}_i(\mathbf{u})\|_{s'}
\leq C_i h^{p_i} \|\mathbf{u}\|_{s+p_i}.
\]
\end{assumption}

\vspace{1em}
\hrule
\vspace{1em}

\subsection*{2. Capacity Allocation Lower Bound}

\begin{lemma}[Capacity Allocation Lower Bound]
\label{lemma:allocation}
Let $\mathcal{D}^{nl} = \sum_{i=1}^k \lambda_i^{nl} D_i^{nl}$. Then:
\begin{equation}
\inf_{\mathrm{NO}_\theta \in \mathcal{N}(P)}
\sup_{\mathbf{u}\in K}
\|\mathcal{D}^{nl}(\mathbf{u}) - \mathrm{NO}_\theta(\mathbf{u})\|_{s'}
\;\geq\;
\inf_{\{P_i\}:\sum P_i \le P}
\sum_{i=1}^k |\lambda_i^{nl}|\,\epsilon_i^*(P_i).
\end{equation}
\end{lemma}

\begin{proof}
Let $\mathrm{NO}_\theta \in \mathcal{N}(P)$ be arbitrary. For each $i$, define the residual operator:
\[
R_i(\mathbf{u})
:= D_i^{nl}(\mathbf{u}) - \widetilde{D}_i(\mathbf{u}),
\]
where $\widetilde{D}_i$ denotes the implicit contribution of $\mathrm{NO}_\theta$ toward approximating $D_i^{nl}$.

Since $\mathrm{NO}_\theta$ has finite capacity $P$, its approximation of the sum must implicitly allocate representational capacity across the components. This induces an effective decomposition into sub-approximations with budgets $\{P_i\}$ satisfying $\sum P_i \le P$.

By definition of $\epsilon_i^*$, each component satisfies:
\[
\sup_{\mathbf{u}\in K}
\|D_i^{nl}(\mathbf{u}) - \widetilde{D}_i(\mathbf{u})\|_{s'}
\ge \epsilon_i^*(P_i).
\]

Applying the triangle inequality,
\[
\|\mathcal{D}^{nl} - \mathrm{NO}_\theta\|_{s'}
\ge \sum_{i=1}^k |\lambda_i^{nl}|
\|D_i^{nl} - \widetilde{D}_i\|_{s'}.
\]

Taking the supremum over $\mathbf{u}$ and infimum over $\theta$ yields the result.
\end{proof}

\vspace{1em}
\hrule
\vspace{1em}

\subsection*{3. Approximation Error Bounds}

\begin{theorem}[OpsSplit Error Bound]
\label{thm:opsplit}
Under the stated assumptions,
\begin{equation}
\epsilon_{\mathrm{OpS}}(P)
\leq
\delta_{\mathrm{lin}}
+ \inf_{\{P_i\}:\sum P_i \le P}
\sum_{i=1}^k |\lambda_i^{nl}|\,\epsilon_i^*(P_i),
\end{equation}
where
\[
\delta_{\mathrm{lin}}
:= \sum_{i=1}^{j}|\lambda_i^l|\,C_i h^{p_i}\|\mathbf{u}\|_{s+p_i}.
\]
\end{theorem}

\begin{proof}
By the triangle inequality,
\[
\|\mathcal{D}_X - \widehat{\mathcal{D}}_X\|_{s'}
\le
\sum_{i=1}^{j}|\lambda_i^l|\|D_i^l-\mathrm{FD}_i\|_{s'}
+ \sum_{i=1}^{k}|\lambda_i^{nl}|\|D_i^{nl}-\mathrm{NO}_{\theta_i}\|_{s'}.
\]

The first term is bounded by $\delta_{\mathrm{lin}}$ via Assumption 4.

Optimizing over $\{\theta_i\}$ subject to $\sum P_i \le P$ yields the second term.
\end{proof}

\begin{theorem}[Comparison with Monolithic Models]
\label{thm:comparison}
Let $\epsilon_{\mathrm{mono}}(P)$ denote the optimal error over $\mathcal{N}_{\mathrm{mono}}(P)$. Then:
\begin{equation}
\epsilon_{\mathrm{mono}}(P)
\;\geq\;
\inf_{\{P_i\}:\sum P_i \le P}
\sum_{i=1}^k |\lambda_i^{nl}|\,\epsilon_i^*(P_i).
\end{equation}
\end{theorem}

\begin{proof}
This follows directly from Lemma~\ref{lemma:allocation}.
\end{proof}

\vspace{1em}
\hrule
\vspace{1em}

\subsection*{4. Out-of-Distribution Generalization}

\begin{theorem}[Coefficient Generalization]
\label{thm:ood}
Let $\boldsymbol{\lambda}, \boldsymbol{\lambda}' \in \mathbb{R}^k$. Then the OpsSplit error satisfies:
\begin{equation}
\|\mathcal{D}_X^{\boldsymbol{\lambda}'}(\mathbf{u})
- \widehat{\mathcal{D}}_X^{\boldsymbol{\lambda}'}(\mathbf{u})\|_{s'}
\leq
\delta_{\mathrm{lin}}(\boldsymbol{\lambda}')
+ \sum_{i=1}^{k}|(\lambda_i^{nl})'|\,\delta_i,
\end{equation}
where $\delta_i$ is the training error of $\mathrm{NO}_{\theta_i}$.
\end{theorem}

\begin{proof}
Directly,
\[
\mathcal{D}_X^{\boldsymbol{\lambda}'} - \widehat{\mathcal{D}}_X^{\boldsymbol{\lambda}'}
=
\sum_{i=1}^j (\lambda_i^l)' (D_i^l - \mathrm{FD}_i)
+
\sum_{i=1}^k (\lambda_i^{nl})' (D_i^{nl} - \mathrm{NO}_{\theta_i}).
\]

Apply the triangle inequality and definitions of $\delta_{\mathrm{lin}}$ and $\delta_i$.
\end{proof}

\vspace{1em}
\hrule
\vspace{1em}

\subsection*{5. Temporal Stability}

\begin{corollary}[Euler Stability]
Let $L$ be the Lipschitz constant of $\mathcal{D}_X$. Then under explicit Euler:
\begin{equation}
\|\mathbf{u}^N - \hat{\mathbf{u}}^N\|_s
\le
\frac{e^{L N \Delta t}-1}{L}
\left(
\delta_{\mathrm{lin}} + \sum_{i=1}^k |\lambda_i^{nl}|\delta_i
+ \mathcal{O}(\Delta t)
\right).
\end{equation}
\end{corollary}

\begin{proof}
Standard discrete Grönwall argument applied to the error recursion.
\end{proof}

\vspace{1em}
\hrule
\vspace{1em}

\subsection*{6. Transferability}

\begin{proposition}[Operator Transfer]
Let systems $A$ and $B$ differ only in operators $D_i^{nl,A}$ and $D_i^{nl,B}$. Then:
\begin{equation}
\delta_i^{(B)}
\le
\delta_i^{(A)}
+
\sup_{\mathbf{u}\in K_B}
\|D_i^{nl,B}(\mathbf{u}) - D_i^{nl,A}(\mathbf{u})\|_{s'}.
\end{equation}
\end{proposition}

\begin{proof}
Immediate from the triangle inequality.
\end{proof}

\newpage
\section{Interpretability}
\label{appendix:interpretability}
As discussed in \cref{sec:exp_incomp}, constructing neural operators to explicitly learn physical operators within the OpsSplit framework enables soft enforcement of PDE constraints during prediction. This approach yields models that demonstrate improved fitting and enhanced capability for temporal extrapolation and generalisation to unseen physical conditions. Moreover, the learned physical operators provide inherent interpretability, as their behaviour can be directly verified against numerical estimations of the corresponding operators. This interpretability facilitates identification of failure modes within the prediction mechanism and enables more parameter-efficient, physics-aware model debugging strategies.

\subsection{Incompressible Navier-Stokes}
\Cref{fig:NS_incomp_ConvOps} visualises the convection operator employed in the numerical solver (\cref{fig:incomp_convops_num}) alongside those learned by various neural operator architectures: FNO (\cref{fig:incomp_convops_fno}), CNO (\cref{fig:incomp_convops_cno}), and UNO (\cref{fig:incomp_convops_uno}). Each neural operator successfully captures the nuances and general behaviour of the convection operator, approximating the underlying gradients to varying degrees of fidelity, as evidenced by the variations across architectures. These visualisations constitute a qualitative interpretability study wherein the learned convection operators, evaluated in latent space, are compared against their numerical counterparts computed in physical space. All operators are normalised to the range $[-1, 1]$ to facilitate visual comparison. The observation that distinct operator architectures learn similar attributes of the convection operator warrants further investigation in future work.

\begin{figure}[!htb]
    \centering
    \begin{subfigure}{0.48\textwidth}
        \centering
        \includegraphics[width=\linewidth]{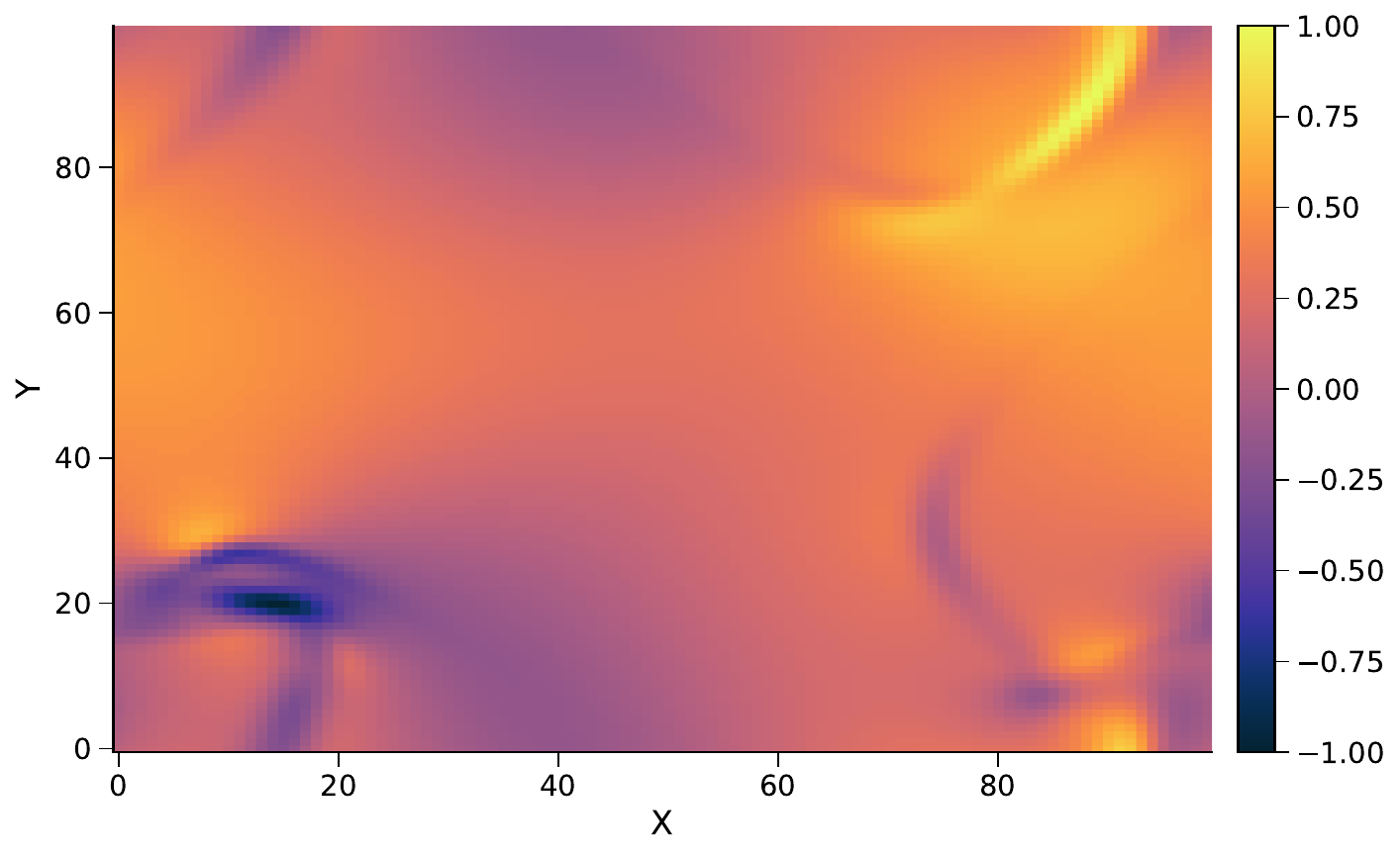}
        \caption{Convection Operator: Numerical}
        \label{fig:incomp_convops_num}
    \end{subfigure}
    \begin{subfigure}{0.48\textwidth}
        \centering
        \includegraphics[width=\linewidth]{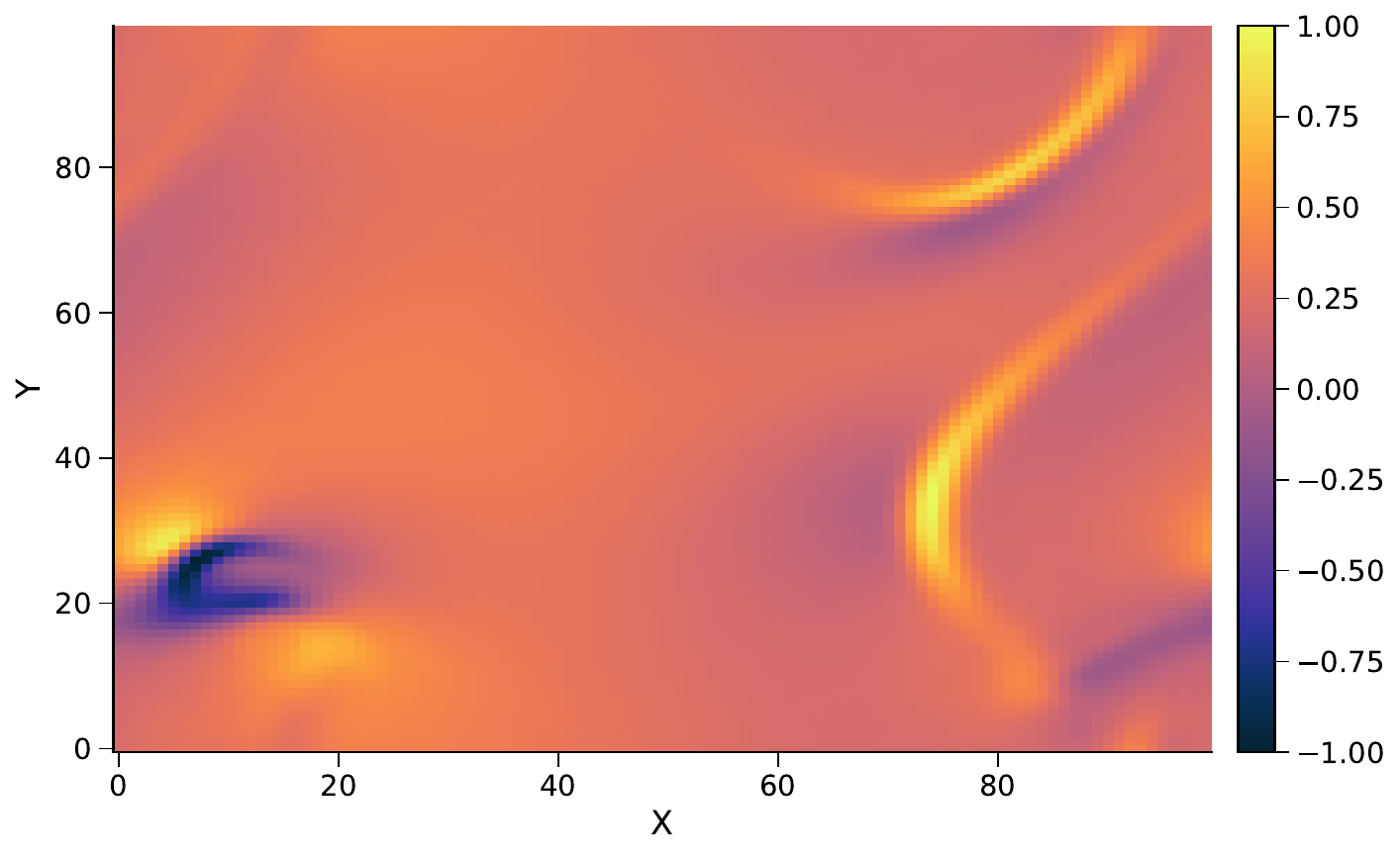}
        \caption{Convection Operator: FNO}
        \label{fig:incomp_convops_fno}
    \end{subfigure}
        \begin{subfigure}{0.48\textwidth}
        \centering
        \includegraphics[width=\linewidth]{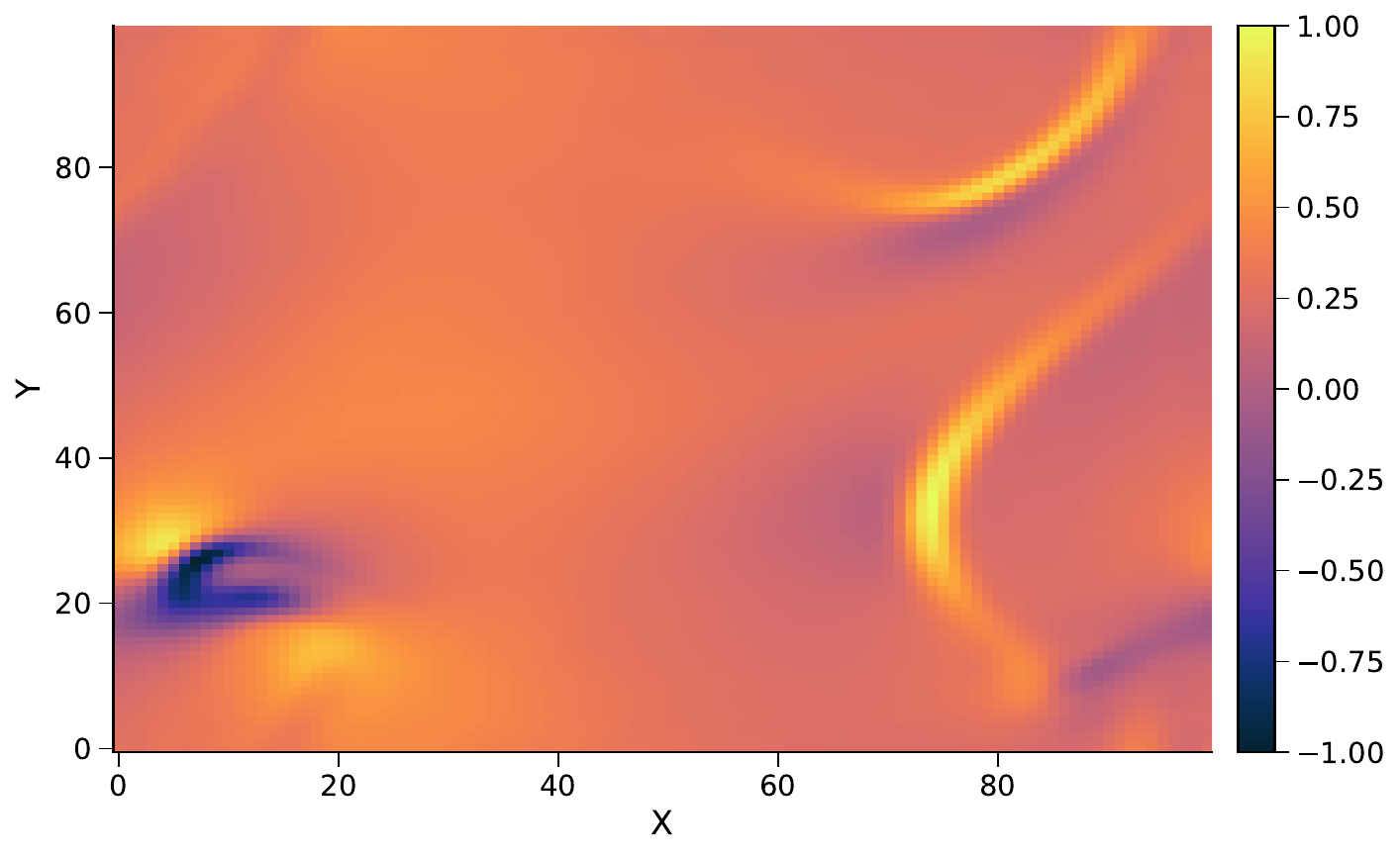}
        \caption{Convection Operator: CNO}
        \label{fig:incomp_convops_cno}
    \end{subfigure}
    \begin{subfigure}{0.48\textwidth}
        \centering
        \includegraphics[width=\linewidth]{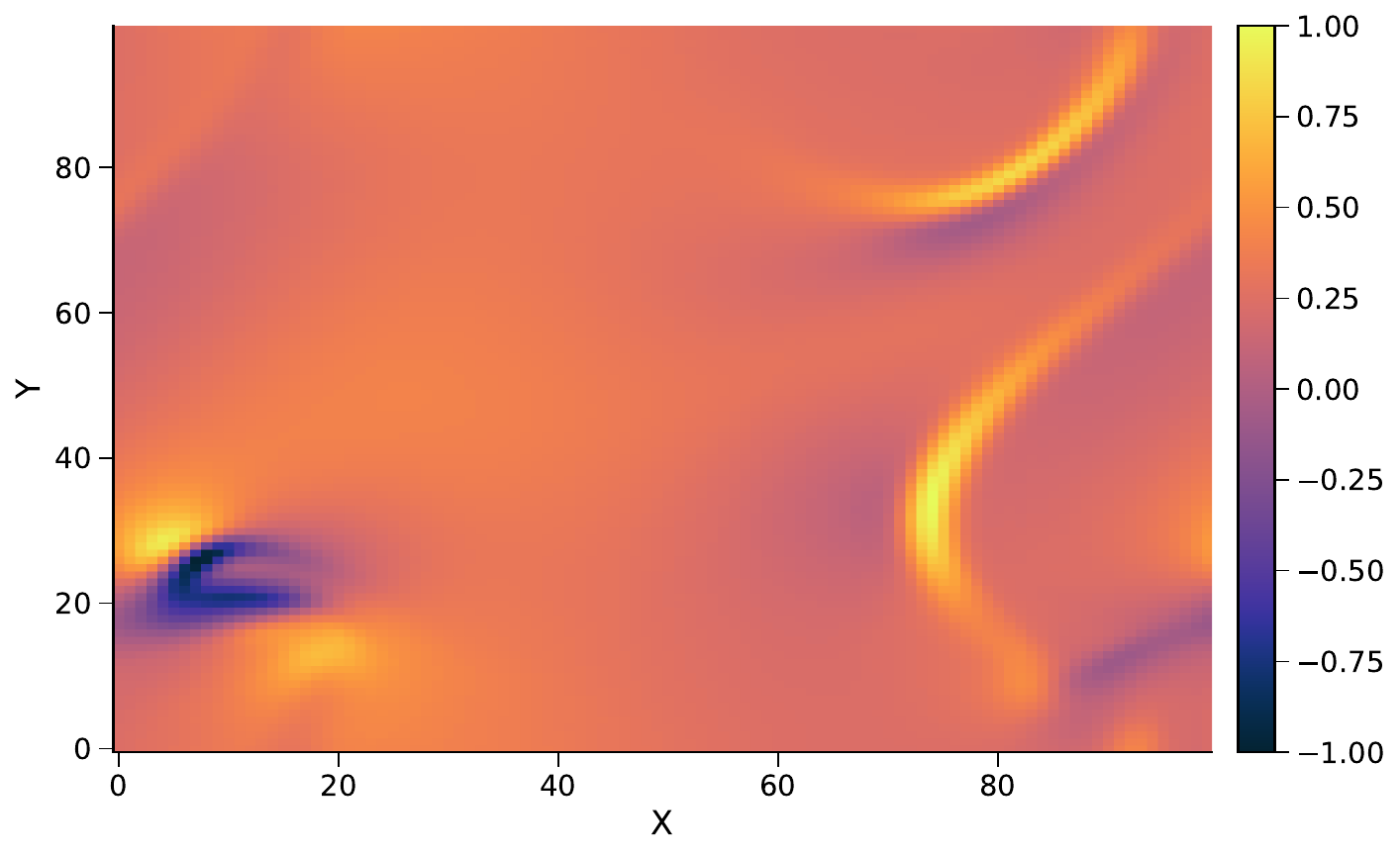}
        \caption{Convection Operator: UNO}
        \label{fig:incomp_convops_uno}
    \end{subfigure}
    \caption{Incompressible Navier--Stokes: comparing across convection operator obtained numerically against those learnt using various neural operators.}
    \label{fig:NS_incomp_ConvOps}
\end{figure}

\subsection{Compressible Navier-Stokes}
\Cref{fig:NS_comp_interp} visualises both the convection and vector divergence operators computed via finite differences (\cref{fig:comp_convops_num,fig:comp_divops_num_rho,fig:comp_divops_num_p}) alongside their learned counterparts from an FNO architecture (\cref{fig:comp_convops_fno,fig:comp_divops_fno_rho,fig:comp_divops_fno_p}, respectively). For the convection operator, the FNO successfully captures the general fluid motion across the domain, though it fails to reproduce certain fine-scale features present in the physical operator. As formulated in \cref{eqn:comp_mass_no,eqn:comp_pressure_no}, the vector divergence operator characterises the advection of scalar fields, and a single neural operator learns its effects on both pressure and density. \Cref{fig:NS_comp_interp} demonstrates that while the FNO captures dominant features, it introduces spurious characteristics into the learned operation. This suggests the need for further investigation into alternative operator splitting strategies. As with the incompressible case, these visualisations provide qualitative interpretability analysis, with learned operators evaluated in latent space and numerical operators displayed in physical space, all normalised to $[-1, 1]$ for effective visual comparison.

\begin{figure}[H]
    \centering
    \begin{subfigure}{0.48\textwidth}
        \centering
        \includegraphics[width=\linewidth]{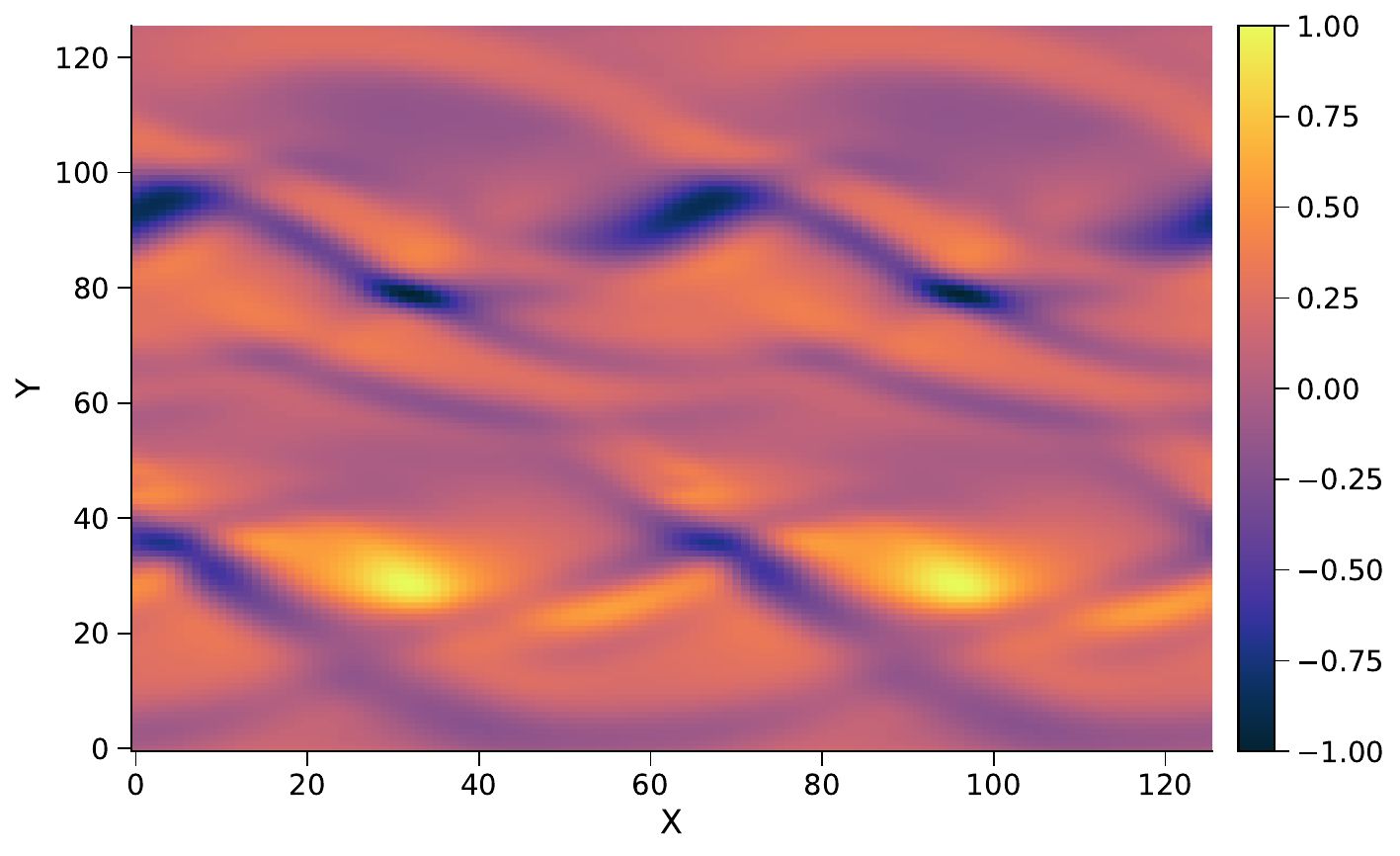}
        \caption{Convection Operator: Numerical}
        \label{fig:comp_convops_num}
    \end{subfigure}
    \begin{subfigure}{0.48\textwidth}
        \centering
        \includegraphics[width=\linewidth]{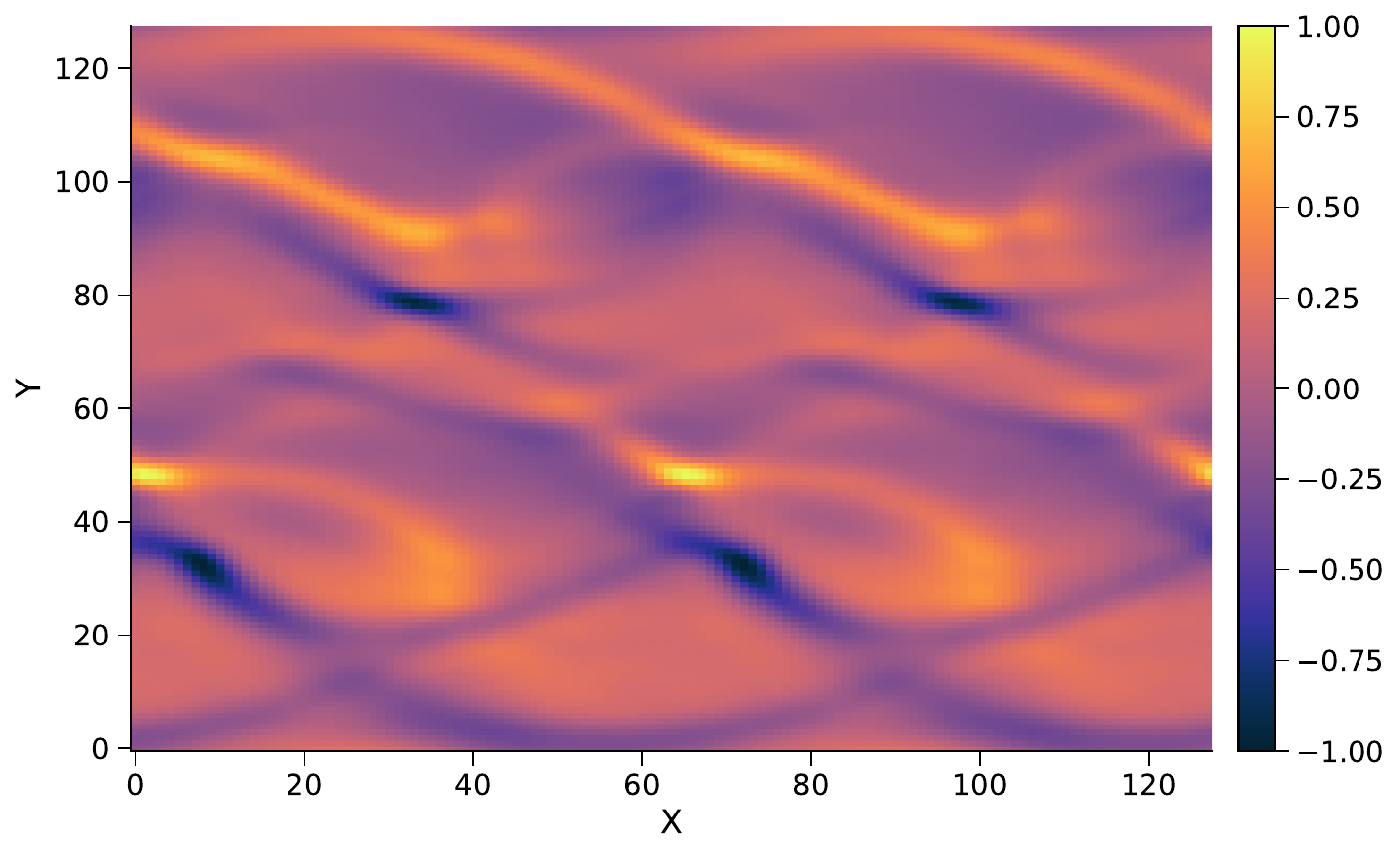}
        \caption{Convection Operator: FNO}
        \label{fig:comp_convops_fno}
    \end{subfigure}
    \begin{subfigure}{0.48\textwidth}
        \centering
        \includegraphics[width=\linewidth]{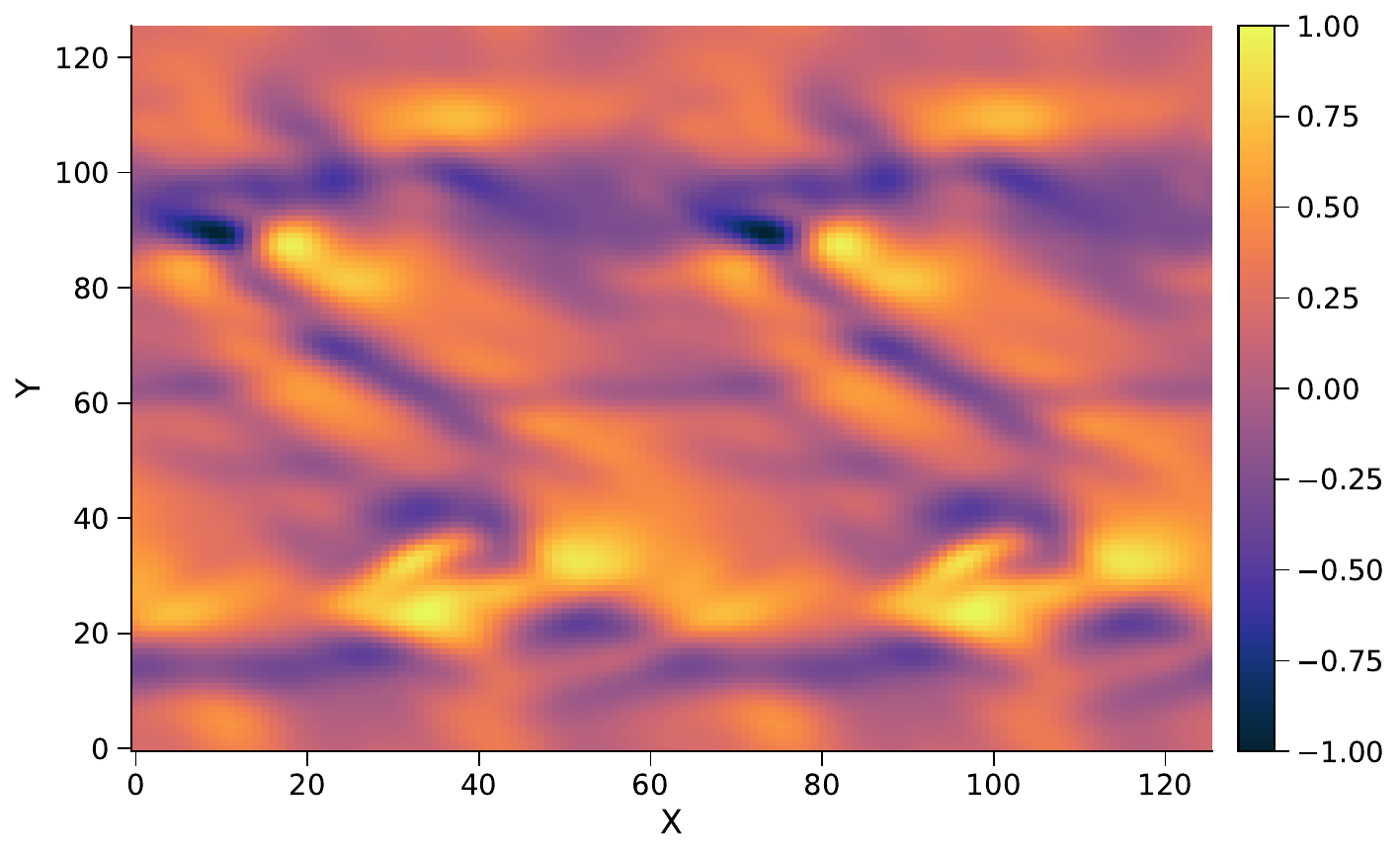}
        \caption{Vector Divergence Operator $(\rho)$: Numerical}
        \label{fig:comp_divops_num_rho}
    \end{subfigure}
    \begin{subfigure}{0.48\textwidth}
        \centering
        \includegraphics[width=\linewidth]{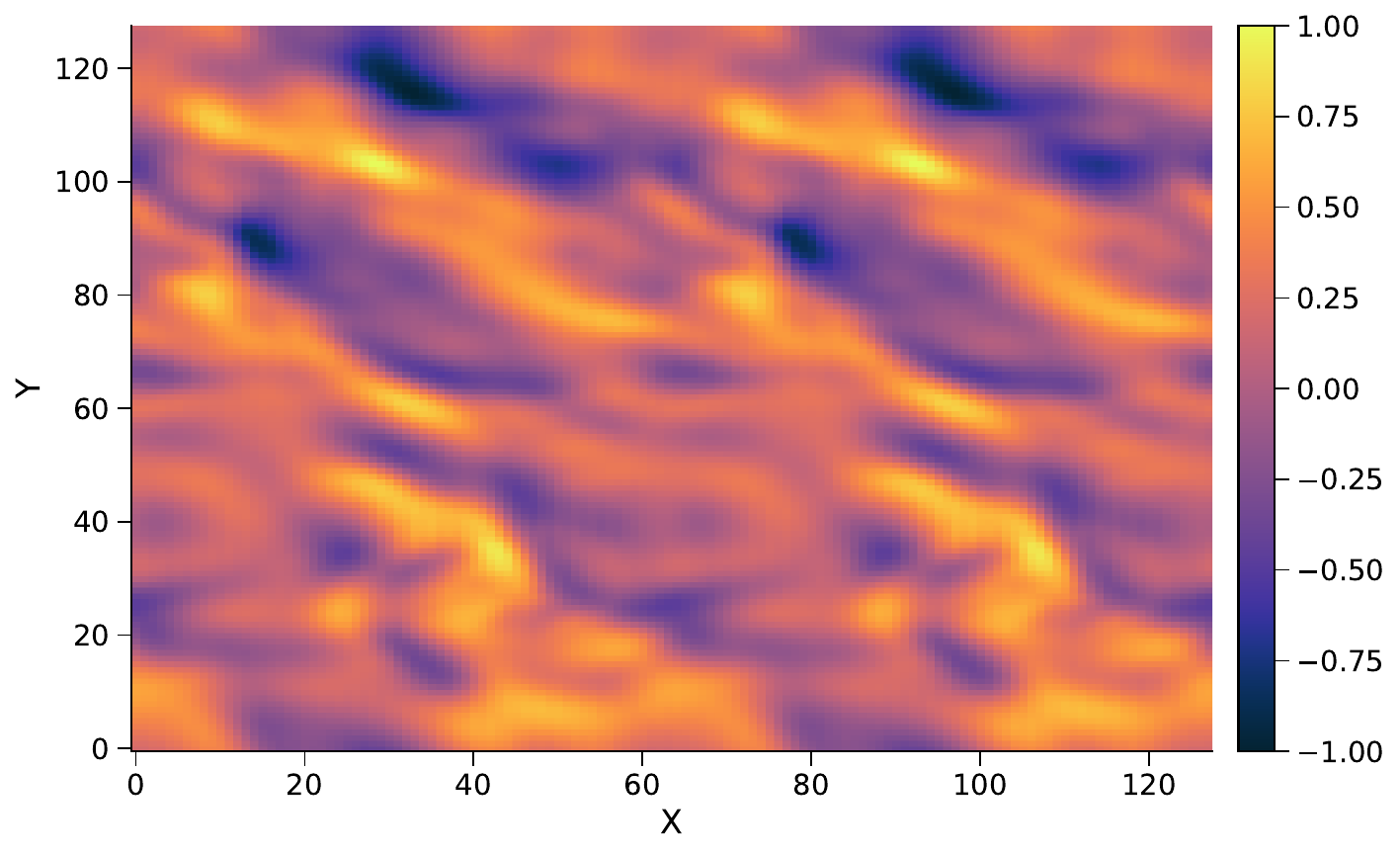}
        \caption{Vector Divergence Operator $(P)$:FNO}
        \label{fig:comp_divops_fno_rho}
    \end{subfigure}
    \begin{subfigure}{0.48\textwidth}
        \centering
        \includegraphics[width=\linewidth]{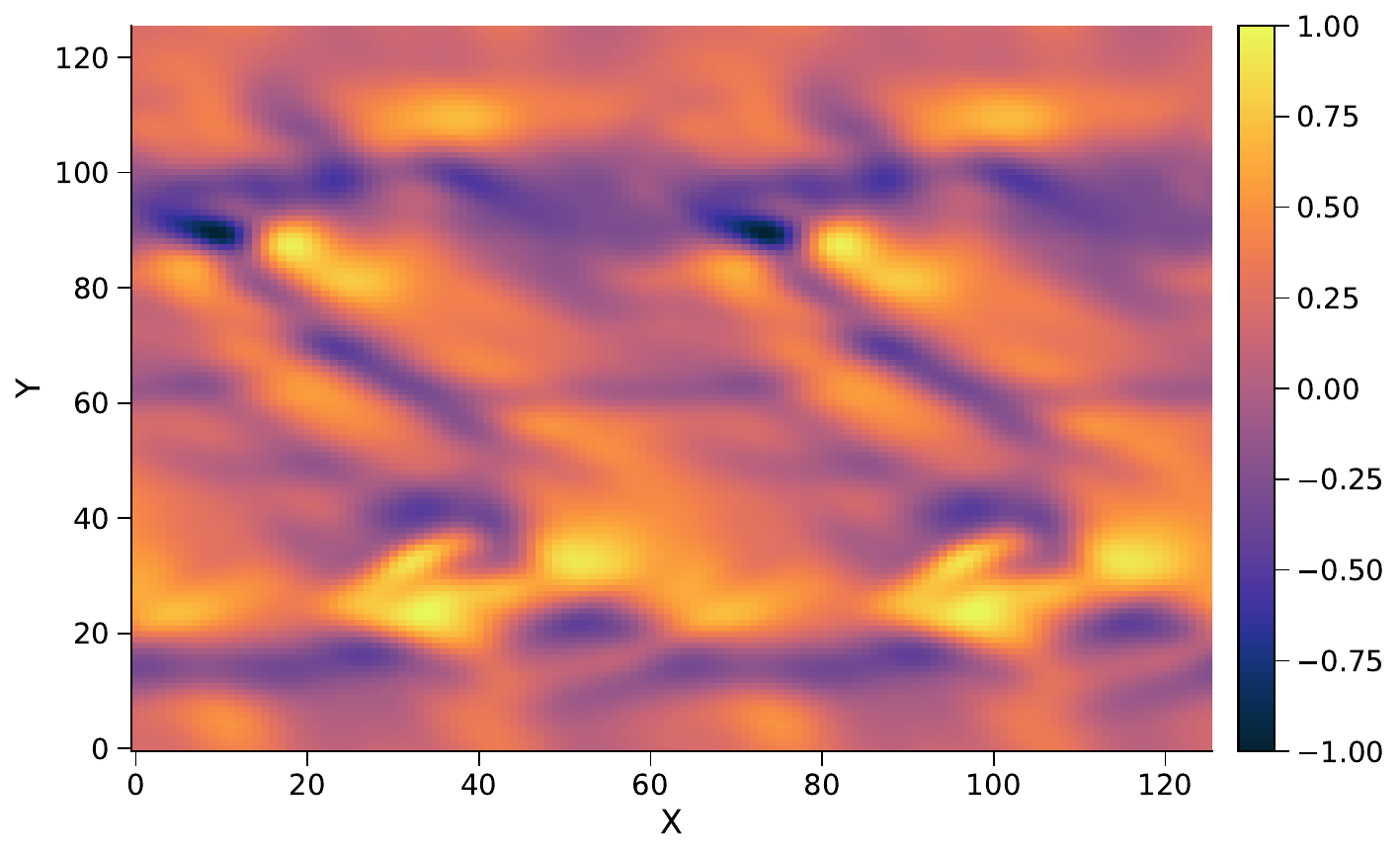}
        \caption{Vector Divergence Operator $(P)$: Numerical}
        \label{fig:comp_divops_num_p}
    \end{subfigure}
    \begin{subfigure}{0.48\textwidth}
        \centering
        \includegraphics[width=\linewidth]{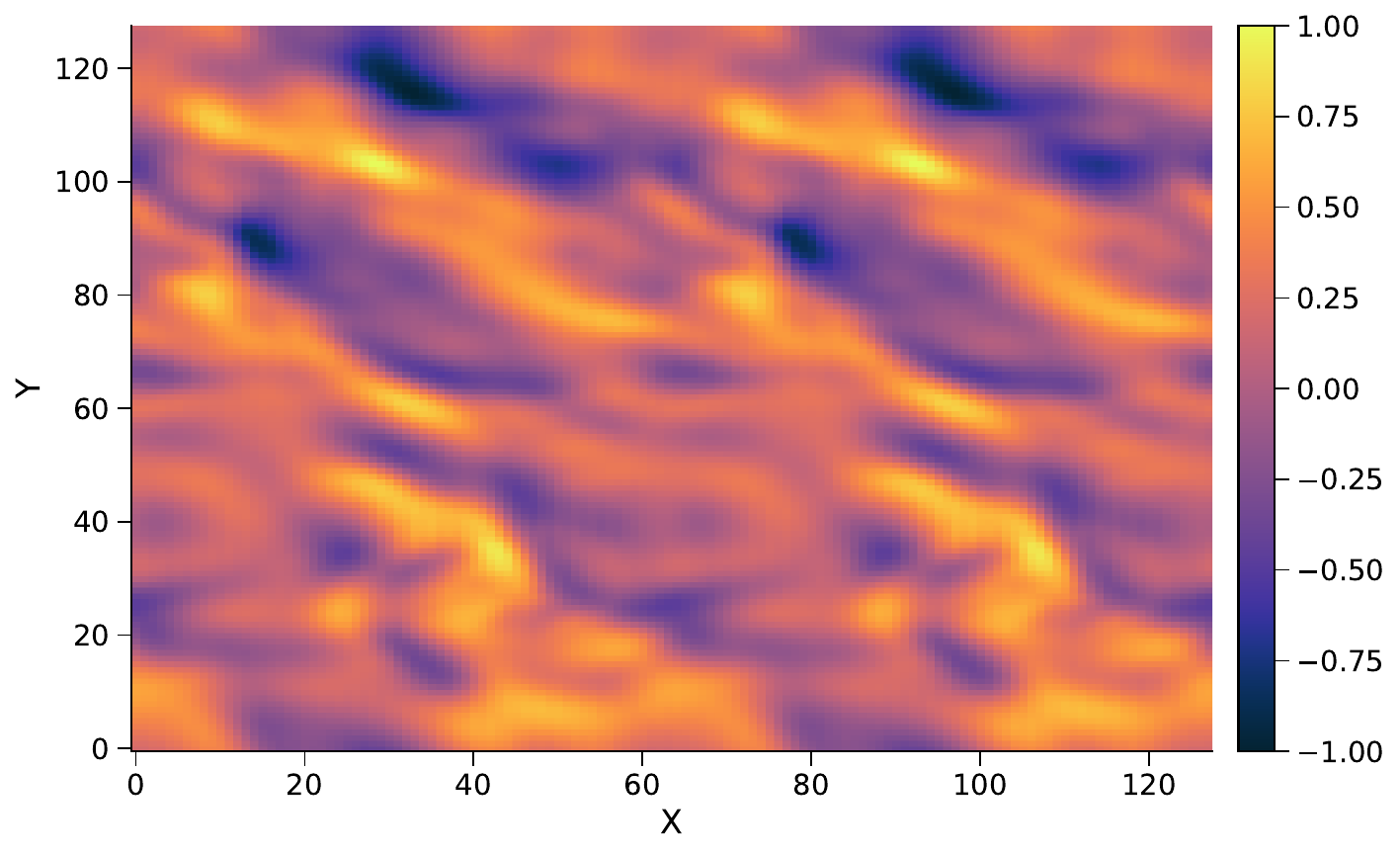}
        \caption{Vector Divergence Operator $(P)$: FNO}
        \label{fig:comp_divops_fno_p}
    \end{subfigure}
    \caption{Compressible Navier--Stokes: comparing across convection and vector divergence operator obtained numerically against those learnt using various neural operators.}
    \label{fig:NS_comp_interp}
\end{figure}

\newpage
\section{Operator Splitting: A Guideline}
\label{appendix:operator_splitting}
 
Operator splitting, in its classical numerical sense, refers to
\textit{temporal} (fractional-step) methods that advance each
sub-operator sequentially over separate sub-intervals—for example,
Lie--Trotter or Strang splitting—introducing splitting errors that
arise from the non-commutativity of the split operators
\citep{holden2010splitting, MacNamara2016}.
The OpsSplit framework described in \cref{sec: method} adopts a
conceptually related but structurally distinct strategy: an
\textit{additive spatial decomposition} following the Method of Lines
(MoL) paradigm \citep{schiesser1991numerical}.
Here the spatial operator $D_X(u)$ is decomposed into a sum of
physically motivated sub-operators, each evaluated \emph{simultaneously}
on the current state $u^n$, and their combined contribution constitutes
the right-hand side (RHS) of a time-continuous ODE:
\begin{equation}
    \dv{u}{t}
    = \sum_{i=1}^{j} \lambda_{l_i}\,\mathbb{FD}_i(u)
    + \sum_{i=1}^{k} \lambda_{nl_i}\,\mathbb{NO}_i(u),
    \label{eq:mol_decomp_app}
\end{equation}
which is then integrated forward using a standard ODE solver
(e.g., explicit Euler or Runge--Kutta).
Crucially, \emph{no sequential fractional time steps are executed}:
all sub-operators see the same state $u^n$ at each integration point,
and their outputs are summed before the time-update is applied.
Consequently, the temporal non-commutativity errors and intermediate
fractional boundary conditions that characterise classical Strang
splitting are \emph{not} present in this framework.
The approximation error instead manifests in the RHS itself—from
replacing continuous spatial operators with finite-difference stencils
($\mathbb{FD}$) and learned neural operators ($\mathbb{NO}$)—and is
fully decoupled from the temporal integration error, which is
controlled independently by the choice of ODE solver.
 
The practical motivation for the decomposition in
\cref{eq:mol_decomp_app} mirrors that of classical splitting:
isolating physical processes (e.g., advection from diffusion) allows
each to be approximated by the most suitable tool—fixed stencils for
well-understood linear physics and neural operators for complex
non-linear phenomena—while preserving a monolithic time integration
pipeline.
This section provides guidelines for designing effective additive
decompositions within this neural Method of Lines framework.

\subsection{Implementation Guidelines}
\label{appendix:operator_splitting_guidelines}
 
Since no universal decomposition rule exists \citep{McLachlanSplitting2002},
empirical validation remains essential.
For the additive MoL framework, a robust decomposition strategy
should adhere to the following principles:
 
\begin{enumerate}
    \item \textbf{Physics-driven process identification:}
    Identify dominant physical processes and separate terms by
    mathematical character—linearity versus nonlinearity, locality
    versus non-locality.
    Assign linear, analytically tractable operators (e.g., diffusion,
    linear advection) to fixed finite-difference stencils, and reserve
    neural operator capacity for operators whose mathematical form is
    complex or whose effective behaviour is difficult to encode
    analytically (e.g., non-linear convection, coupled
    pressure--velocity effects from the pressure Poisson equation).
 
    \item \textbf{Simultaneous evaluation consistency:}
    Because all sub-operators are evaluated on the same state $u^n$
    in the additive framework, each $\mathbb{NO}_i$ or
    $\mathbb{FD}_i$ must be designed to act on the \emph{full}
    current state, not an intermediate or partially-advanced state.
    This distinguishes OpsSplit from fractional-step methods where
    later sub-operators receive states that have already been advanced
    by earlier sub-operators.
 
    \item \textbf{Timescale and stiffness management:}
    Operators with disparate timescales (e.g., fast diffusion versus
    slow advection) may impose restrictive time-step requirements
    via the CFL condition when treated fully explicitly.
    Within the MoL framework, stiff linear operators can be
    incorporated via implicit-explicit (IMEX) integrators
    \citep{Ascher1997ImplicitexplicitRM}, treating the $\mathbb{FD}$
    terms implicitly and the $\mathbb{NO}$ terms explicitly, without
    altering the additive spatial decomposition.
 
    \item \textbf{Conservation and invariant preservation:}
    Where physical conservation laws (mass, momentum, energy) or
    geometric constraints (e.g., divergence-free velocity fields) are
    critical, the decomposition should be designed so that each
    neural operator can learn to enforce the relevant invariants in
    its output on the shared state $u^n$.
    Conservation can be monitored via the physics residual error
    during training and rollout (\cref{fig:PRE_cont_FNO_OOD}),
    providing an interpretable diagnostic tied directly to each
    operator's individual contribution to the RHS.
 
    \item \textbf{Modular reusability:}
    Because sub-operators in the additive formulation share no
    sequential coupling through intermediate states, they can be
    trained and validated independently before being composed into
    the full RHS sum.
    This modularity supports transfer across related PDE families:
    a neural operator trained to approximate the convection operator
    for one set of PDE parameters can be fine-tuned or directly
    re-used for a related system, as demonstrated in
    \cref{appendix:convergence}.
\end{enumerate}

\subsection{Additive Decomposition and Approximation Error}
\label{appendix:operator_splitting_accuracy}
 
In the additive RHS formulation of \cref{eq:mol_decomp_app}, the
forward Euler discrete update takes the form
\begin{equation}
    u^{n+1}
    = u^n
    + \Delta t\!\left(
        \sum_{i=1}^{j} \lambda_{l_i}\,\mathbb{FD}_i(u^n)
      + \sum_{i=1}^{k} \lambda_{nl_i}\,\mathbb{NO}_i(u^n)
      \right),
    \label{eq:mol_euler_app}
\end{equation}
which is identical to \cref{eq:abstract_nox_pde_euler} in the main
text.
Higher-order ODE integrators (e.g., explicit Runge--Kutta methods)
may be substituted without altering the additive spatial structure,
improving temporal accuracy independently of the decomposition choice.
 
The dominant error sources in this framework are therefore:
 
\begin{enumerate}
    \item \textbf{Spatial approximation error from $\mathbb{FD}$:}
    Each finite-difference stencil $\mathbb{FD}_i$ introduces a
    truncation error that depends on the stencil order.
    A second-order stencil yields a spatial error of
    $\mathcal{O}(h^2)$ in the grid spacing $h$; higher-order stencils
    reduce this but may degrade out-of-distribution generalisation,
    as evidenced in \cref{tab:incomp_ns_opsplit}.
 
    \item \textbf{Neural operator approximation error from
    $\mathbb{NO}$:}
    Each neural operator $\mathbb{NO}_i$ introduces a function
    approximation error bounded by its capacity and training
    distribution.
    Because each $\mathbb{NO}_i$ is trained to approximate a specific
    physical operator acting on individual snapshots—rather than on
    intermediate partially-advanced states—the error is localised to
    the RHS evaluation at each time step and does not accumulate
    across fractional sub-steps.
 
    \item \textbf{Temporal integration error:}
    For explicit Euler the global temporal error is
    $\mathcal{O}(\Delta t)$; substituting a fourth-order
    Runge--Kutta scheme reduces this to $\mathcal{O}(\Delta t^4)$.
    This error is orthogonal to the spatial decomposition and is
    controlled independently by the ODE solver.
\end{enumerate}
 
There is no splitting error analogous to Strang or Lie--Trotter
schemes because the sub-operators are \emph{summed}, not sequentially
composed.
The commutator $[\mathcal{A},\mathcal{B}]$ is therefore not a direct
source of error in this framework; non-commutativity of spatial
operators affects the RHS only to the extent that each
$\mathbb{FD}_i$ or $\mathbb{NO}_i$ must individually capture the
correct operator action on the current state $u^n$.
This property simplifies both the error analysis and the training
procedure: operators can be trained and validated independently
without reasoning about their sequential interaction.
 
\subsection{Boundary Conditions}
Boundary conditions (BCs) are defined for the complete operator
$D_X$ and must be respected by each sub-operator in the additive
sum.
In the MoL formulation of \cref{eq:mol_decomp_app}, BCs are
applied once to the full state $u^n$ before the RHS is evaluated;
there is no need to prescribe intermediate fractional boundary
states as required in sequential fractional-step methods.
Each $\mathbb{FD}_i$ and $\mathbb{NO}_i$ is therefore designed to
operate on states that already satisfy the global BCs.
In the experiments reported here, periodic BCs are exploited
through standard circular padding in both stencil convolutions
and neural operator architectures.
Extension to non-periodic BCs (e.g., Dirichlet or Neumann)
can be achieved by augmented padding schemes
\citep{alguacil2021effectsboundaryconditionsfully,
mccabe2025walruscrossdomainfoundationmodel}.
 
\subsection{Neural Extension}
The MoL perspective clarifies how classical splitting intuition
translates to the neural setting.
Practitioners should leverage domain expertise for
physics-driven decomposition, exploiting modular
Mixture of Experts (MoE) structures to experiment with different
operator groupings efficiently.
Individual neural operators should be validated in isolation
before being combined into the full RHS, taking advantage of
the decoupled error structure identified in
\cref{appendix:operator_splitting_accuracy}.
Transfer learning across related physical systems is further
facilitated by the additive independence of sub-operators
(see \cref{appendix:convergence}).
Conservation laws and PDE residuals should be monitored during
rollout to detect failure modes localised to specific
physical operators, supporting interpretable model diagnosis.
 
When devising decompositions for neural architectures, the key
trade-off remains between physical fidelity and computational
resource constraints.
Deploying distinct neural operators for individual physical
processes increases model parameterisation and memory overhead.
To mitigate this, the splitting strategy should prioritise
identifying vector operations that are invariant across the
target family of PDEs, thereby reducing redundant
parameterisation.
As demonstrated in \cref{appendix:convergence}, neural operators
exhibit strong generalisation capabilities across different
regimes; consequently, the design process should emphasise the
development of reusable, modular operator models that can be
transferred across related physical systems, maximising the
return on computational investment.

\begin{figure}[ht]
    \centering
    \begin{subfigure}{0.48\textwidth}
        \centering
        \includegraphics[width=\linewidth]{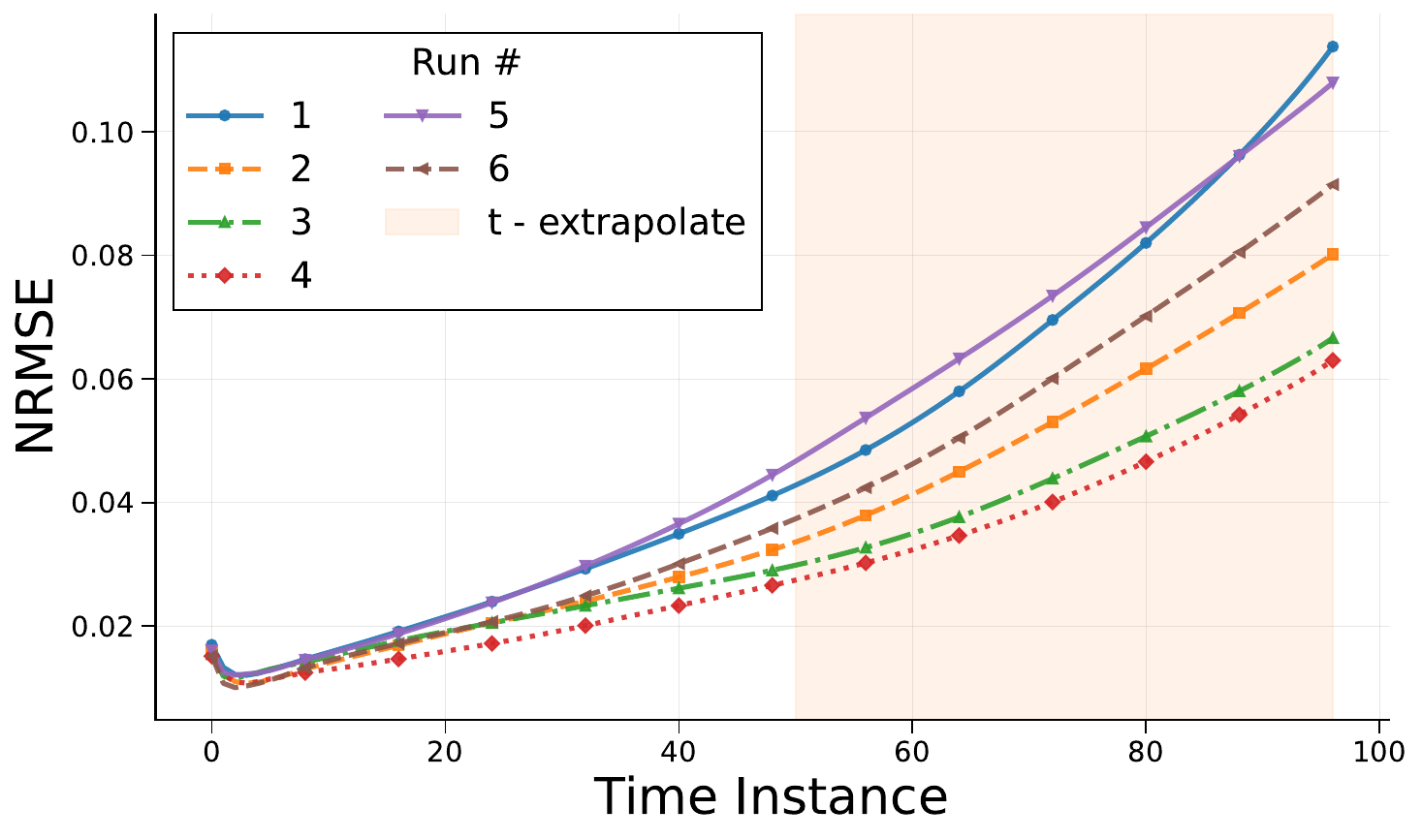}
        \caption{in-distribution}
        \label{fig:NS_incomp_opssplit_ablations_ID}
    \end{subfigure}
    \hfill
    \begin{subfigure}{0.48\textwidth}
        \centering
        \includegraphics[width=\linewidth]{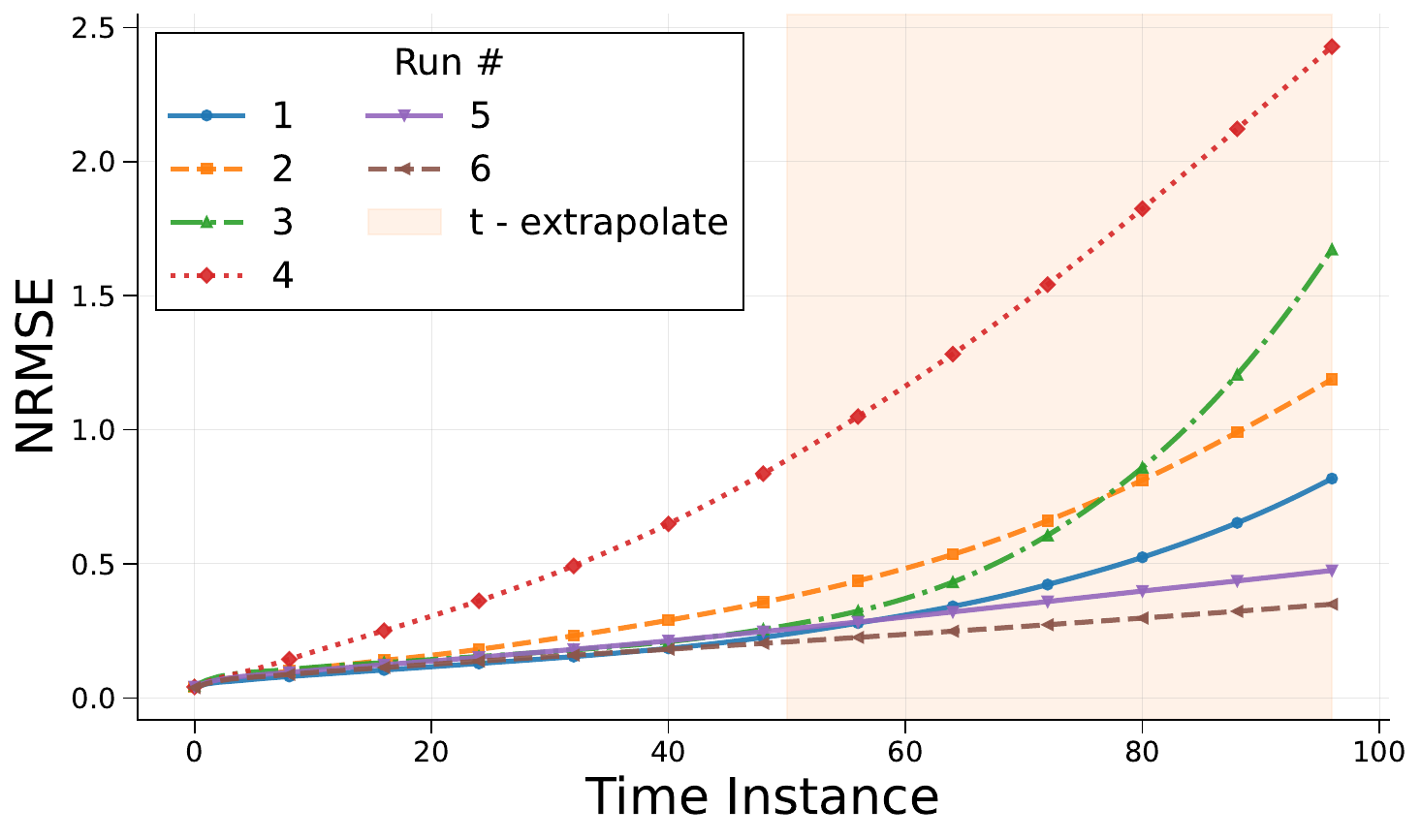}
        \caption{out-of-distribution}
        \label{fig:NS_incomp_opssplit_ablations_OOD}
    \end{subfigure}
    \caption{Incompressible Navier--Stokes: Rollout Error of the the various OpsSplit ablations as described in \cref{tab:incomp_ns_opsplit}. We experiment each ablation across both in and out-of-distribution. While we notice that higher order finite-difference stencils deployed to model the linear operators perform best within distribution, they fail to extend this to outside distribution. Learnt approximations of the linear diffusion operator either using a linear convolutional operator (5) or a neural operator (6) allows for better generalisation.}   
    \label{fig:NS_incomp_opssplit_ablations}
\end{figure}

In \cref{tab:incomp_ns_opsplit}, we experiment with different architectures deployed within our OpSplit method. $FD (n)$ represents a fixed finite difference kernel of order $n$, Linear corresponds to a learnable linear convolution, and $NO$ represents an FNO. We observe that higher-order finite-difference stencils modelling linear operators perform best within the distribution, but fail to generalise effectively to out-of-distribution contexts. Learnt approximations of the linear diffusion operator, whether using a linear convolutional operator (5) or a neural operator (6), enable superior generalisation for the PDE. This is further validated within the temporal rollout plots in  \cref{fig:NS_incomp_opssplit_ablations}.

\begin{table*}[ht]
\centering
\resizebox{\textwidth}{!}{
\begin{tabular}{ccccccccc}
\toprule
\textbf{Ablation} & \textbf{Convection} & \textbf{Diffusion} & & \textbf{Train Time} & \multicolumn{4}{c}{\textbf{NRMSE}} \\
\cmidrule(lr){6-9}
& \textbf{$(\mathbf{v} \cdot \nabla) \mathbf{v}$} & \textbf{$\nabla^2 \mathbf{v}$} & \textbf{Parameters} & \textbf{(hrs:mins)} & \textbf{Test} & \textbf{t-extrapolate} & \textbf{OOD} & \textbf{OOD+t-extrapolate} \\
\midrule
1 & NO & FD (2) & 13437321 & 1:11  & 0.0270  & 0.0796  & 0.1471 & 0.4731 \\
2 & NO & FD (4) & 13437339 & 1:12  & 0.0244 & 0.0600 & 0.1993 & 0.6995 \\
3 & NO & FD (6) & 13437350 & 1:13  & 0.0221 & \textbf{0.0504} & 0.1576 &  0.5303   \\
4 & NO & FD (8) & 13437378 & 1:13 & \textbf{0.0214} & 0.0568 &  0.4195 & 1.5992 \\
5 & NO & Linear & 13438196 & 1:13 & 0.0283 & 0.0823 & 0.1581 & 0.3693\\
6 & NO & NO & 26874628 & 2:18 & 0.0241 & 0.0670 & \textbf{0.1385} & \textbf{0.2815}\\
\bottomrule
\end{tabular}
}
\caption{Incompressible Navier--Stokes: Performance comparison across different methods of operator splitting and subsequent operator choices. $FD (n)$ represents a fixed finite difference kernel of order $n$, Linear corresponds to a learnable linear convolution, and $NO$ represents an FNO.}
\label{tab:incomp_ns_opsplit}
\end{table*}

\begin{figure}[ht]
    \centering
    \begin{subfigure}{0.45\textwidth}
        \centering
        \includegraphics[width=\linewidth]{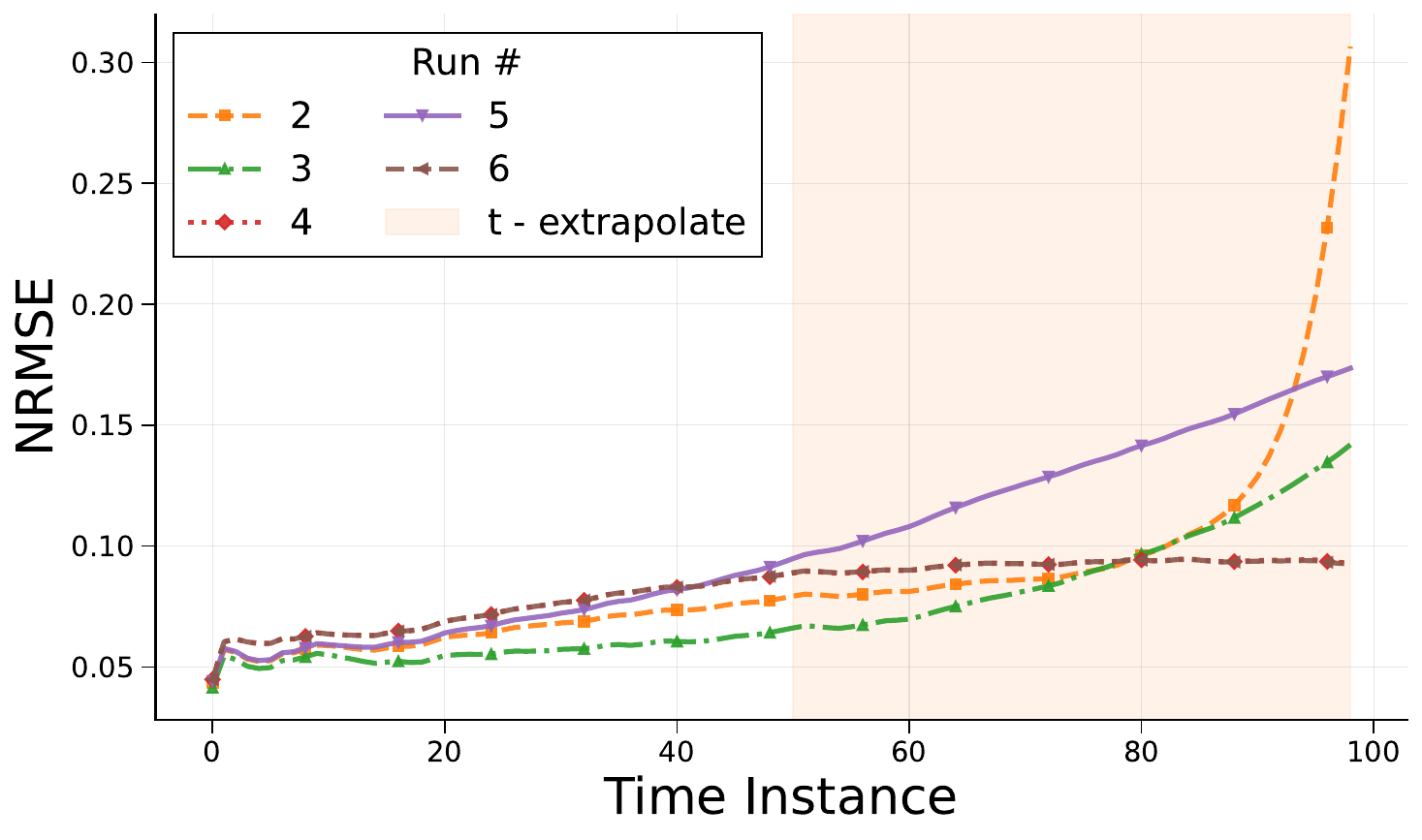}
        \caption{in-distribution}
        \label{fig:NS_comp_opssplit_ablations_ID}
    \end{subfigure}
    \hfill
    \begin{subfigure}{0.48\textwidth}
        \centering
        \includegraphics[width=\linewidth]{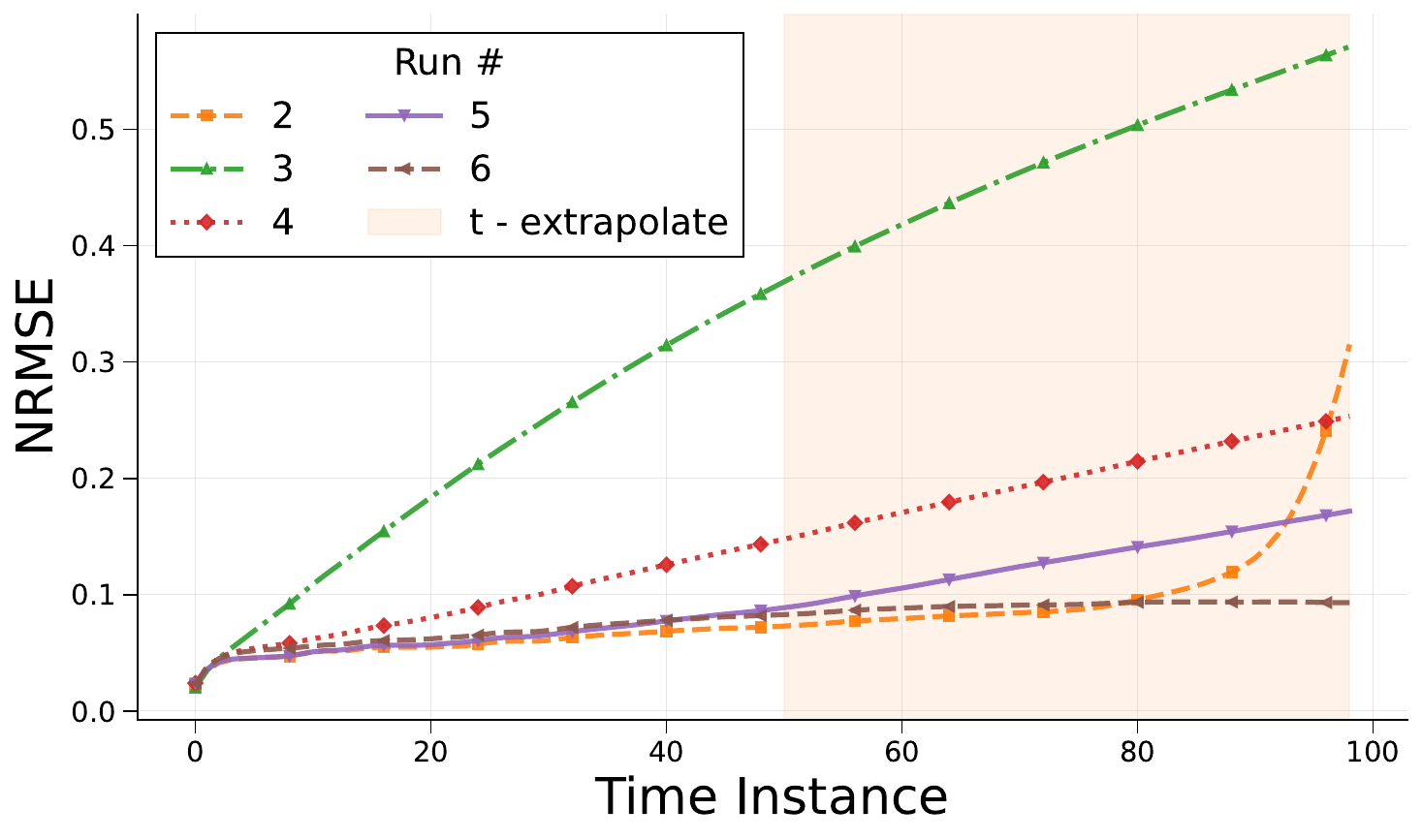}
        \caption{out-of-distribution}
        \label{fig:NS_comp_opssplit_ablations_OOD}
    \end{subfigure}
    \begin{subfigure}{0.48\textwidth}
        \centering
        \includegraphics[width=\linewidth]{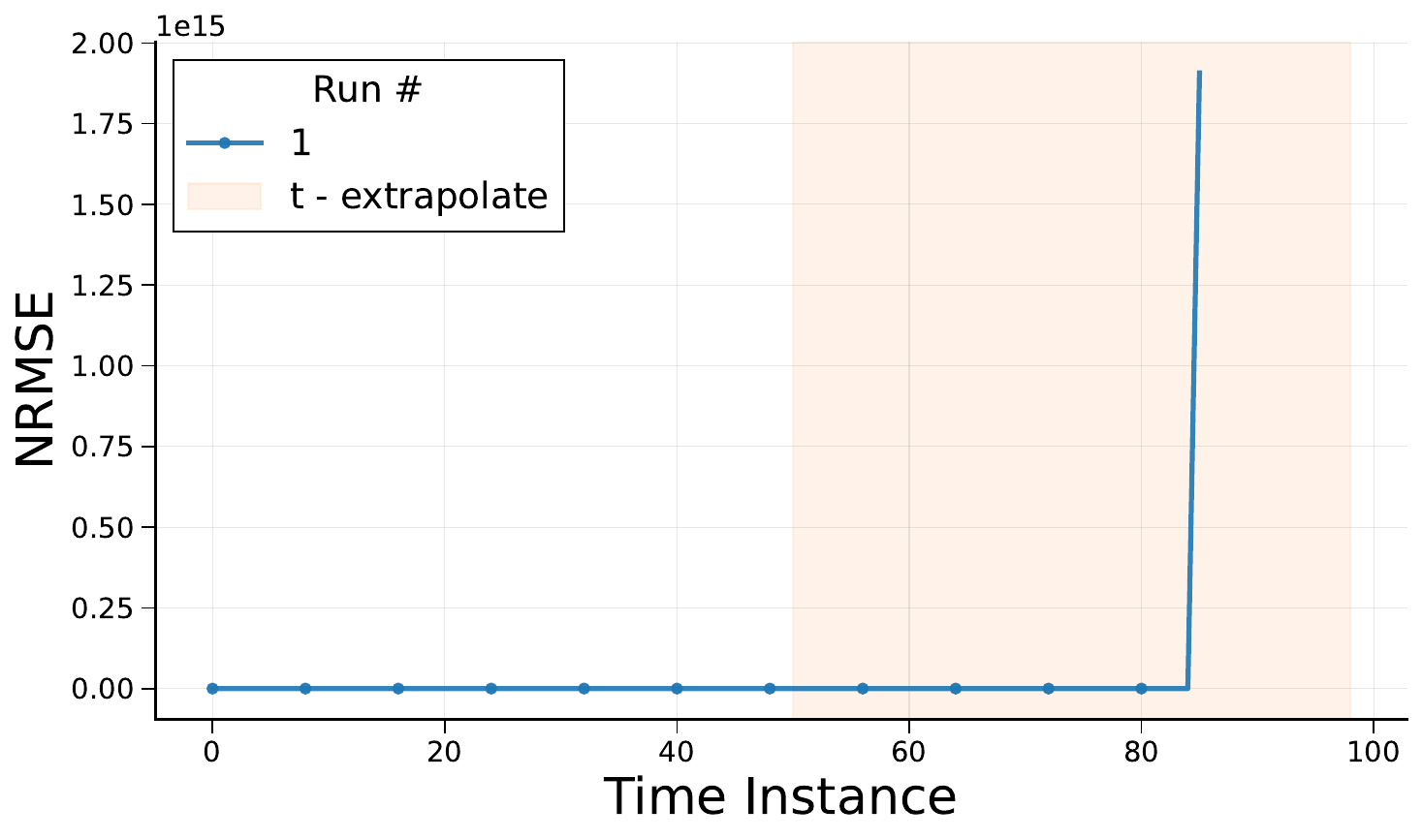}
        \caption{Training Times}
        \label{fig:NS_comp_opssplit_ablations_ID_explode}
    \end{subfigure}
    \hfill
    \begin{subfigure}{0.48\textwidth}
        \centering
        \includegraphics[width=\linewidth]{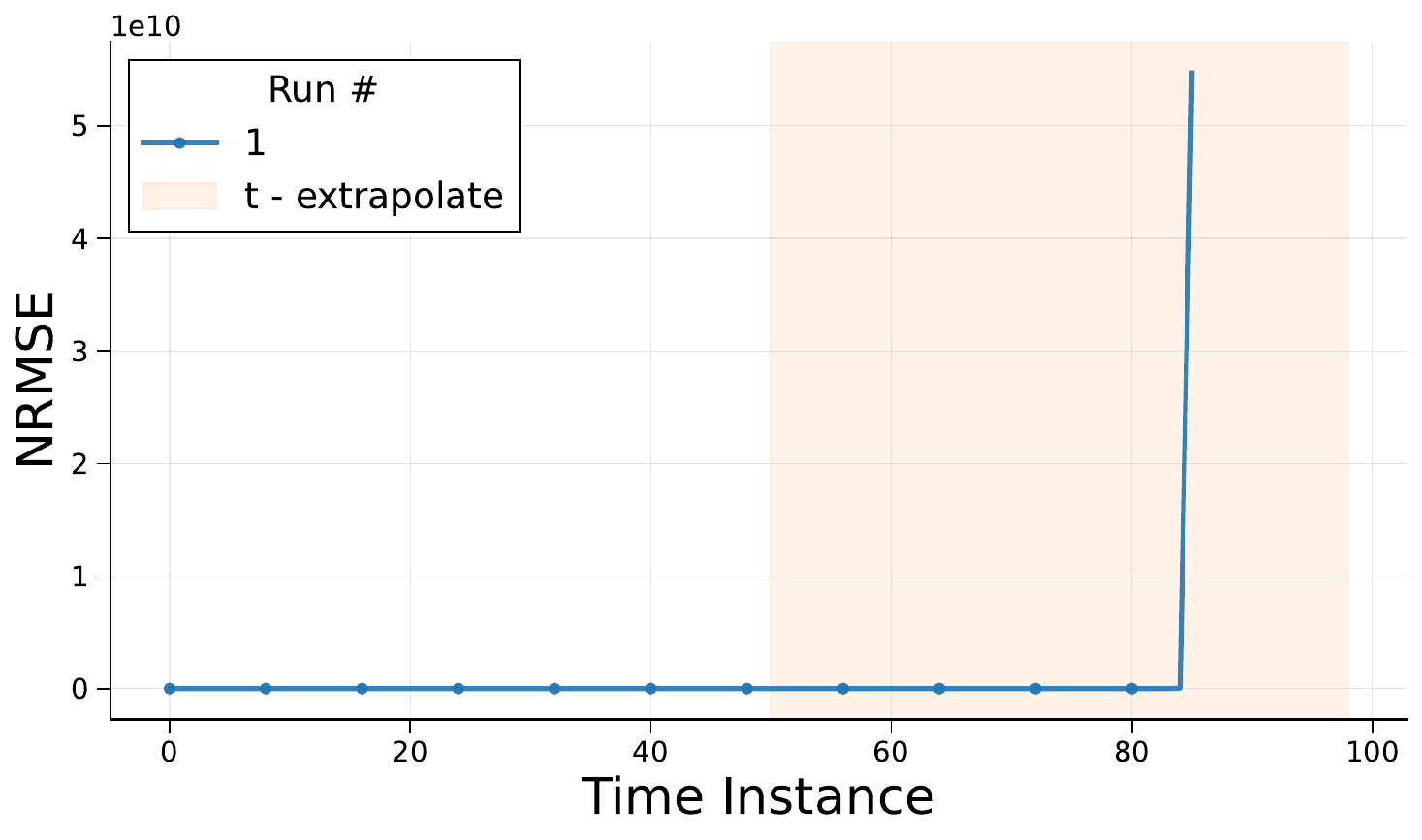}
        \caption{Training Times}
        \label{fig:NS_comp_opssplit_ablations_OOD_explode}
    \end{subfigure}
    \caption{Compressible Navier--Stokes: Rollout Error of the the various OpsSplit ablations as described in \cref{tab:comp_ns_opsplit}. We experiment each ablation across both in and out-of-distribution. We notice that when we approximate nonlinear operators using FD stencils the prediction blows up when extrapolated in time. The general trend we notice is that approximating using neural operators offer better stability both in and out-of-distribution.}   
    \label{fig:NS_comp_opssplit_ablations}
\end{figure}
We performed a similar ablation of OpsSplit strategies across the physical operators within Compressible Navier-Stokes, as outlined in Section 5.2, where $NO_i$ represents the $i^{th}$ neural operator within the system. We observe that using a fixed linear operator over nonlinear operators tends to result in instability, leading to divergent predictions when extrapolating further in time. The general trend suggests that employing neural operators within the splitting strategies offers better stability both in and out of distribution. As indicated in the table, these performance improvements come at the cost of increased parameterisation and/or training time. We also note that using explicitly defined fixed kernels offers better performance than using learnable linear operators. \Cref{fig:NS_comp_opssplit_ablations} demonstrate the temporal rollout of error for each OpsSplit strategy. 

\begin{table*}[ht]
\centering
\resizebox{\textwidth}{!}{
\begin{tabular}{cccccccccccc}
\toprule
\textbf{Ablation} & \textbf{Continuity} & \textbf{Convection} & \textbf{Pressure} & \textbf{Pressure} & \textbf{Pressure} & & \textbf{Train Time} & \multicolumn{4}{c}{\textbf{NRMSE}} \\
& & & \textbf{Gradient} & \textbf{Advection} & \textbf{Expansion} & & & & & & \\
\cmidrule(lr){9-12}
& \textbf{$\nabla \cdot (\rho\mathbf{v})$} & \textbf{$(\mathbf{v} \cdot \nabla) \mathbf{v}$} & \textbf{$\nabla \mathbf{P}$} & \textbf{$\mathbf{v} \cdot \nabla P$} & \textbf{$P(\nabla \cdot \mathbf{v})$} & \textbf{Parameters} & \textbf{(hrs:mins)} & \textbf{Test} & \textbf{t-extrapolate} & \textbf{OOD} & \textbf{OOD+t-extrapolate} \\
\midrule
1 & NO$_1$ & NO$_2$ & FD (4) & FD (4) & FD (4) & 26874627  & 3:56 & 0.0824 & - & 0.0851 & - \\
2 & NO$_1$ & NO$_2$ & FD (4) & FD (4) & NO$_3$ & 40311940 & 5:43 & 0.0762 & 0.1441 & \textbf{0.0785}  & 0.1564 \\
3 & NO$_1$ & NO$_2$ & Linear & Linear & Linear & 26874666 &  3:52  & \textbf{0.0639} & \textbf{ 0.1115} & 0.2061 & 0.4201 \\
4 & NO$_1$ & NO$_2$ & Linear & Linear & NO$_3$ &  40311979
&  5:41 & 0.0866  & 0.1177  & 0.1329  & 0.2839 \\
5 & NO$_1$ & NO$_2$ & FD (4) & NO$_1$ & NO$_1$ & 26874627 & 5:39 & 0.0822  & 0.1691 & 0.0852 & 0.1812\\
6 & NO$_1$ & NO$_1$ & FD (4) & NO$_1$ & NO$_1$ & 13437313 & 7:29 & 0.0865 & 0.1173 &  0.0887 & \textbf{0.1252} \\
\bottomrule 
\end{tabular}
}
\caption{Compressible Navier--Stokes: Performance comparison across different methods of operator splitting and subsequent operator choices. $FD (n)$ represents a fixed finite difference kernel of order $n$, Linear corresponds to a learnable linear convolution, and $NO$ represents an FNO. $NO_i$ represents the $i^{th}$ neural operator within the system.}
\label{tab:comp_ns_opsplit}
\end{table*}

\newpage
\section{Rollout Length}
\label{appendix:rollout_length}

\begin{figure}[H]
    \centering
    \begin{subfigure}{0.5\textwidth}
        \centering
        \includegraphics[width=0.9\linewidth]{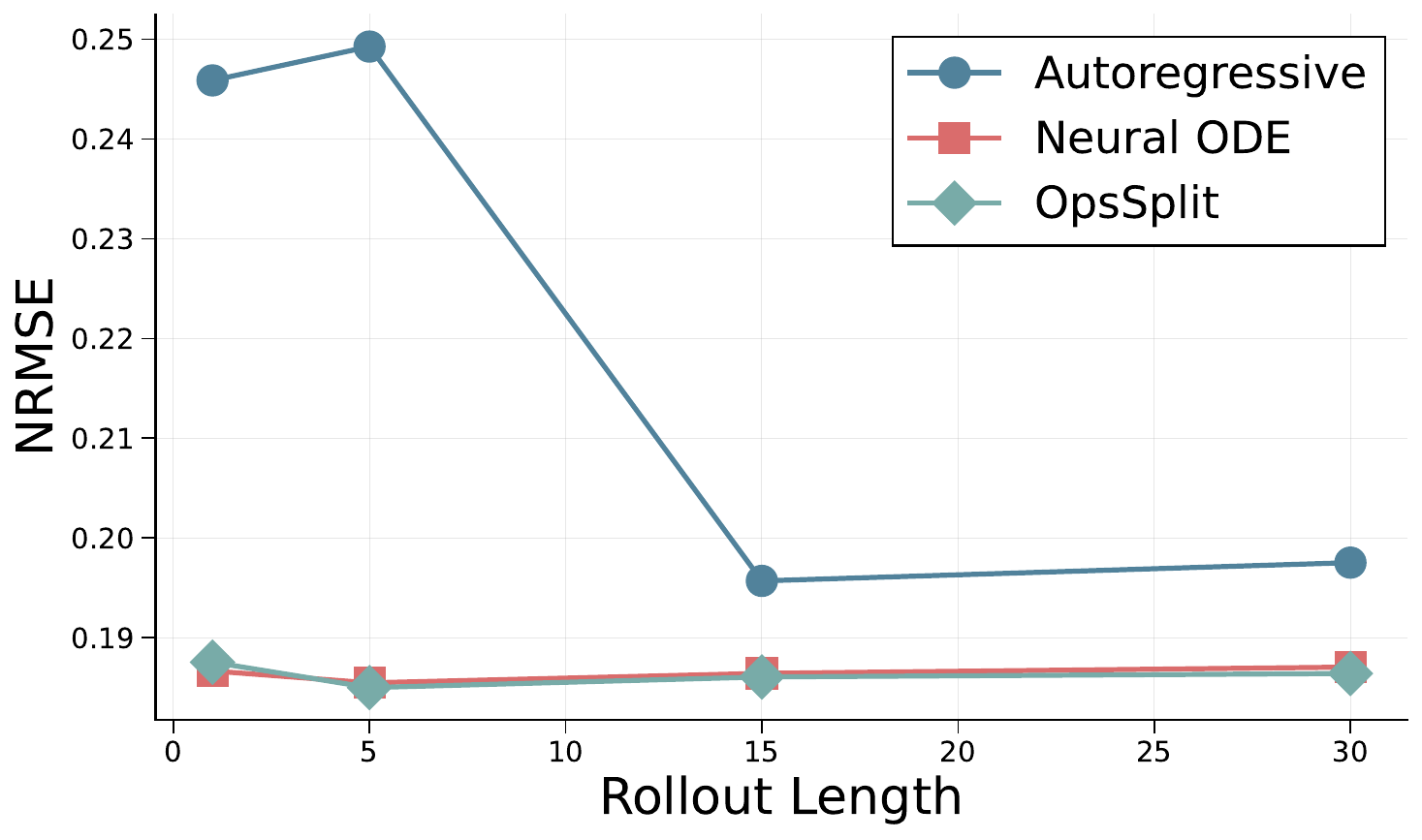}
        \caption{in-distribution}
        \label{fig:NS_incomp_rollout_in}
    \end{subfigure}
    \begin{subfigure}{0.5\textwidth}
        \centering
        \includegraphics[width=0.9\linewidth]{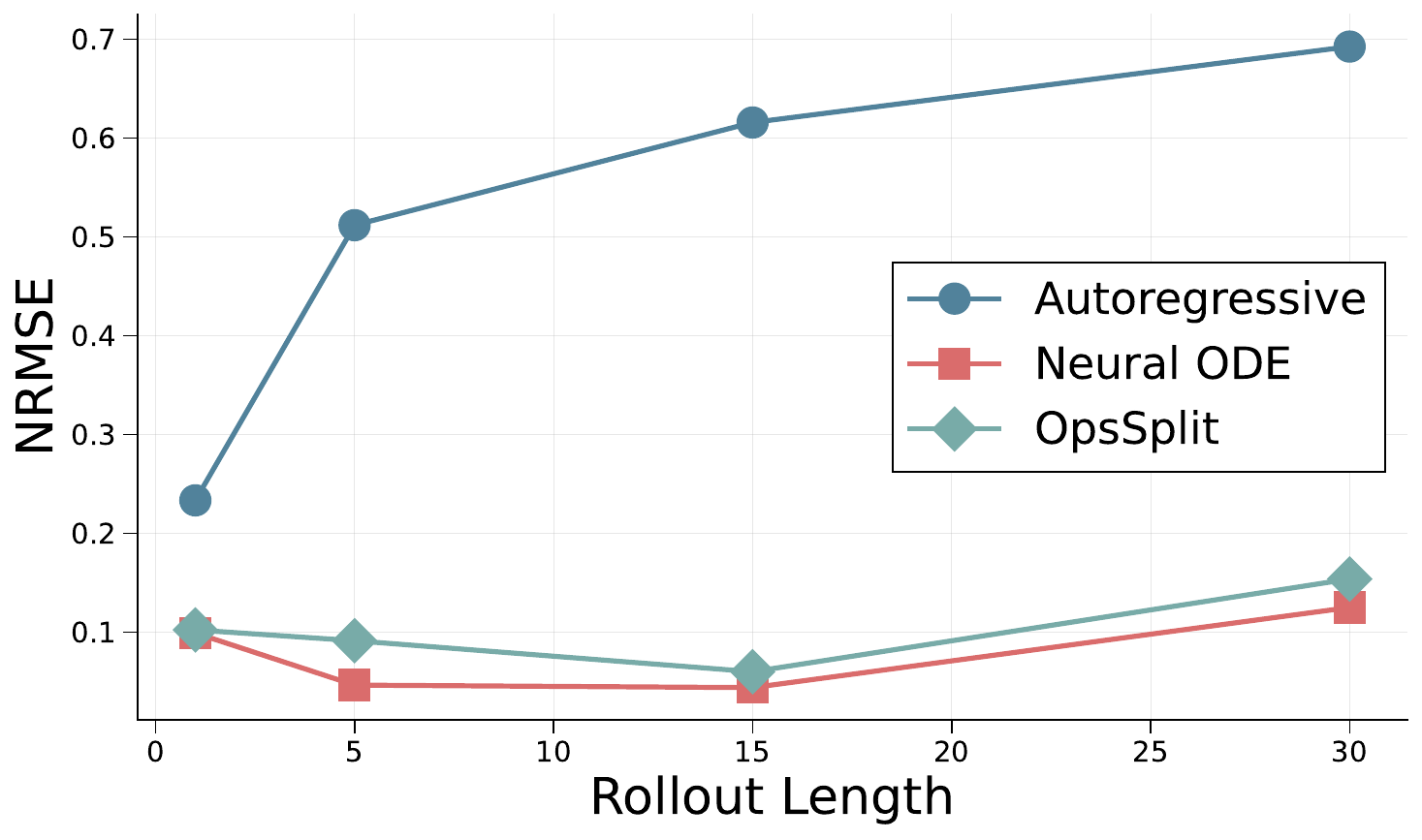}
        \caption{out-of-distribution}
        \label{fig:NS_incomp_rollout_out}
    \end{subfigure}
    \begin{subfigure}{0.5\textwidth}
        \centering
        \includegraphics[width=0.9\linewidth]{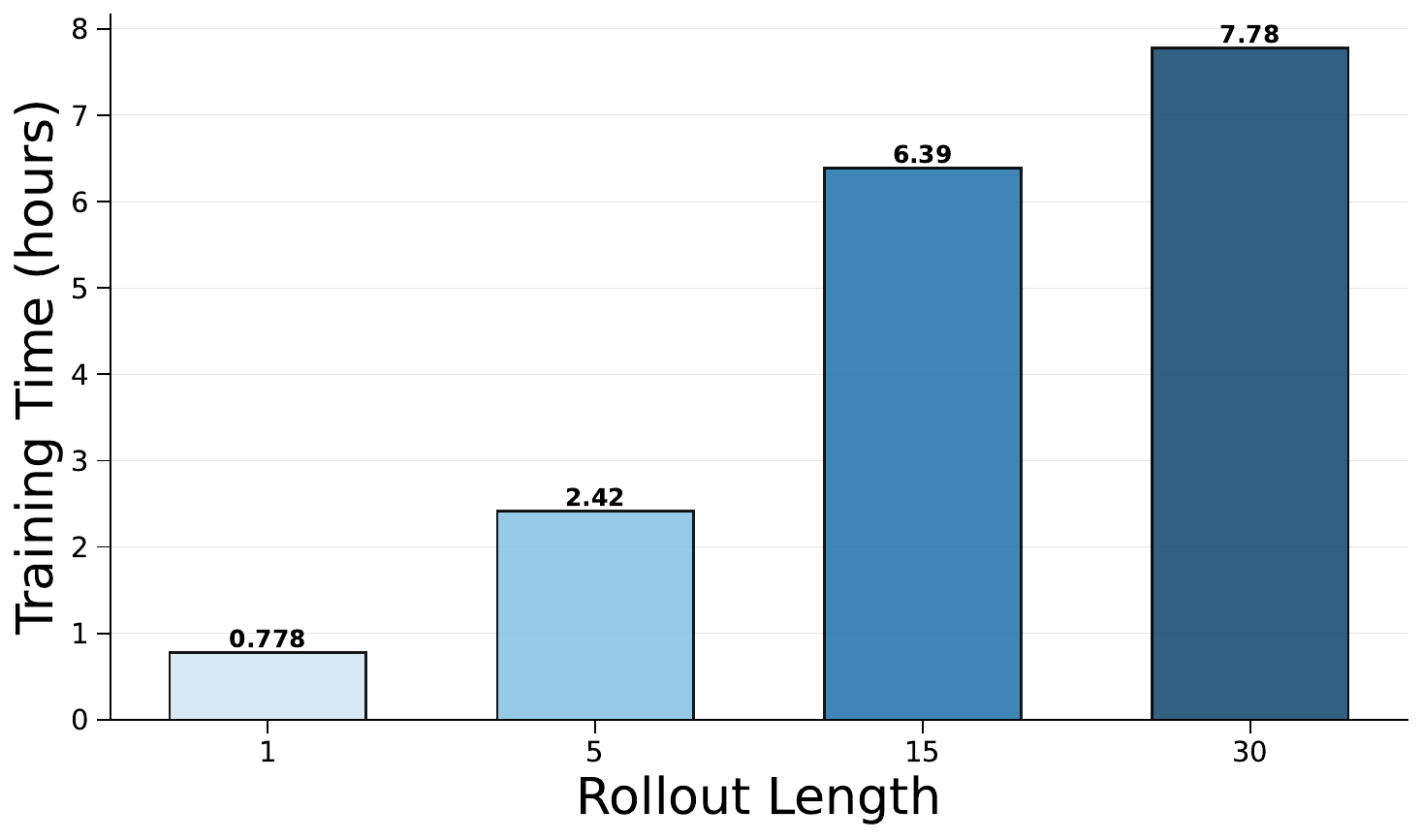}
        \caption{Training Times}
        \label{fig:NS_incomp_rollout_time}
    \end{subfigure}
    \caption{Incompressible Navier--Stokes: Rollout length of the FNO deployed across each method for both in and out-of-distribution.}   \label{fig:NS_incomp_rollout}
\end{figure}

As demonstrated by \citet{koehler2024apebench}, training autoregressive neural PDE models with longer rollout lengths yields improved temporal stability. Our ablation studies corroborate this finding, with \cref{fig:NS_incomp_rollout_in} showing that extended rollout lengths enhance the performance of autoregressive models. In contrast, temporally continuous methods, specifically neural ODEs and our OpsSplit approach, exhibit marginal sensitivity to rollout length for both in-distribution and out-of-distribution scenarios. This insensitivity enables more memory and computationally efficient training paradigms. Notably, for out-of-distribution evaluation (\cref{fig:NS_incomp_rollout_out}), autoregressive models demonstrate degraded performance with longer rollouts, likely attributable to the increasing divergence between novel physics and the training distribution as temporal evolution progresses.

\section{Model Efficiency}
\label{appendix:model_efficiency}
\begin{figure}[H]
    \centering
    \begin{subfigure}{0.5\textwidth}
        \centering
        \includegraphics[width=0.9\linewidth]{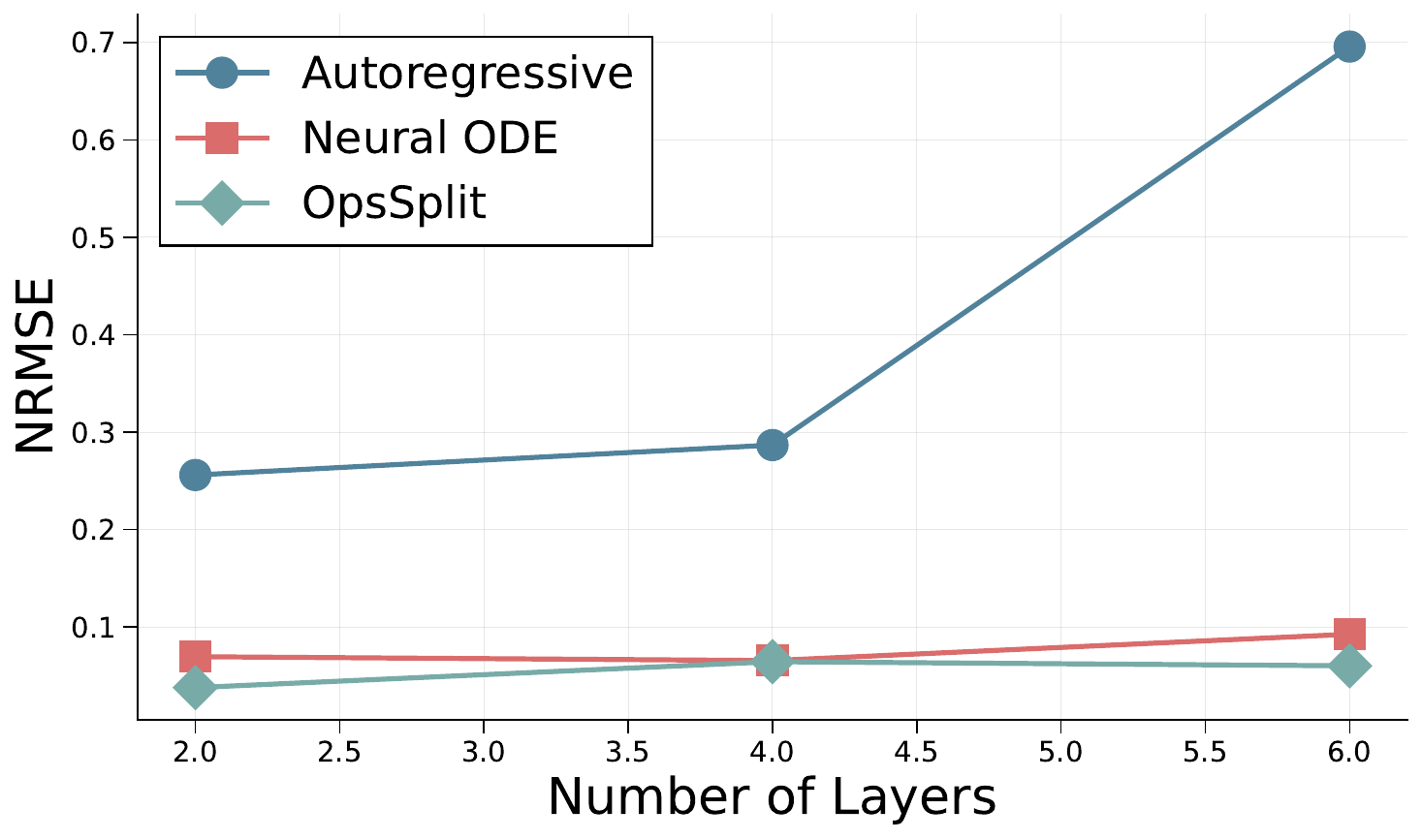}
        \caption{in-distribution}
        \label{fig:NS_incomp_model_efficiency_in}
    \end{subfigure}
    \begin{subfigure}{0.5\textwidth}
        \centering
        \includegraphics[width=0.9\linewidth]{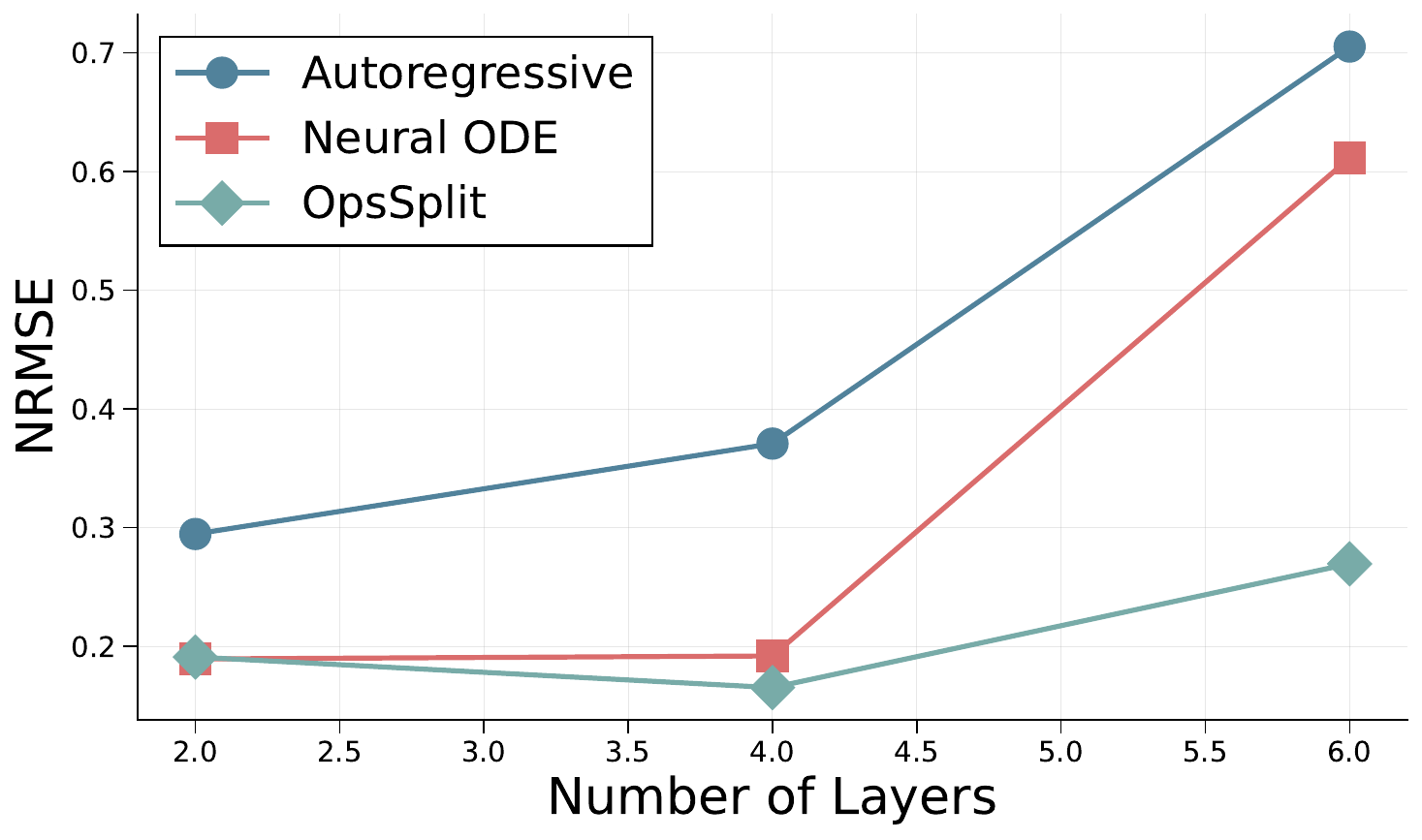}
        \caption{out-of-distribution}
        \label{fig:NS_incomp_model_efficiency_out}
    \end{subfigure}
    \begin{subfigure}{0.5\textwidth}
        \centering
        \includegraphics[width=0.9\linewidth]{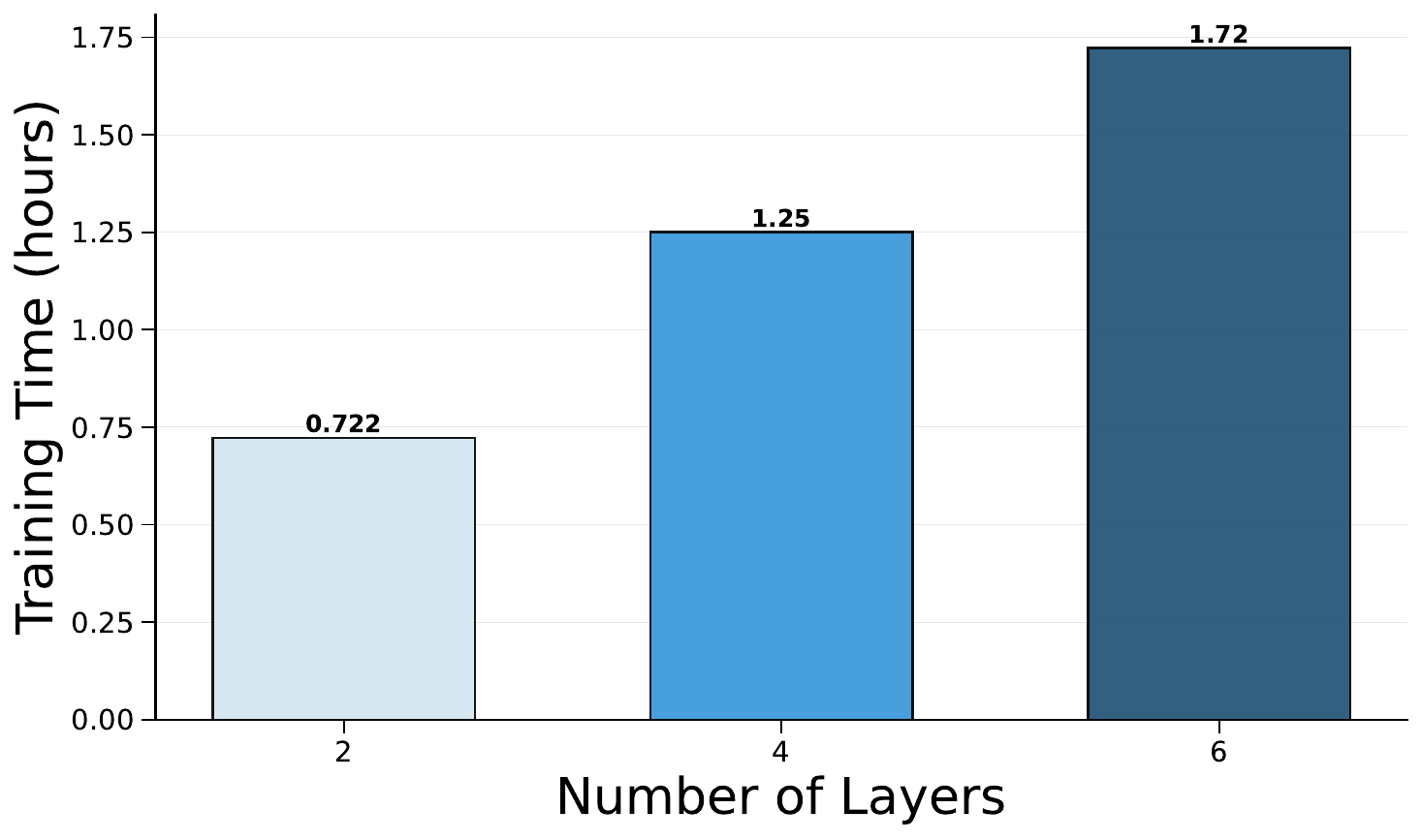}
        \caption{Training Times}
        \label{fig:NS_incomp_model_efficiency_time}
    \end{subfigure}
    \caption{Incompressible Navier--Stokes: Data Efficiency of the FNO deployed across each method for both in and out-of-distribution.}   \label{fig:NS_incomp_model_efficiency}
\end{figure}

\Cref{fig:NS_incomp_model_efficiency} demonstrates that the OpsSplit method exhibits exceptional parameter efficiency, achieving comparable performance across varying FNO depths. Conversely, the autoregressive approach shows deteriorating performance with increased model capacity, likely due to optimisation challenges in high-dimensional parameter spaces that demand greater computational resources to locate satisfactory minima. Across all model sizes, OpsSplit consistently outperforms competing methods for both in-distribution and out-of-distribution evaluation scenarios.

\section{Data Efficiency}
\label{appendix:data_efficiency}

\begin{figure}[H]
    \centering
    \begin{subfigure}{0.5\textwidth}
        \centering
        \includegraphics[width=0.9\linewidth]{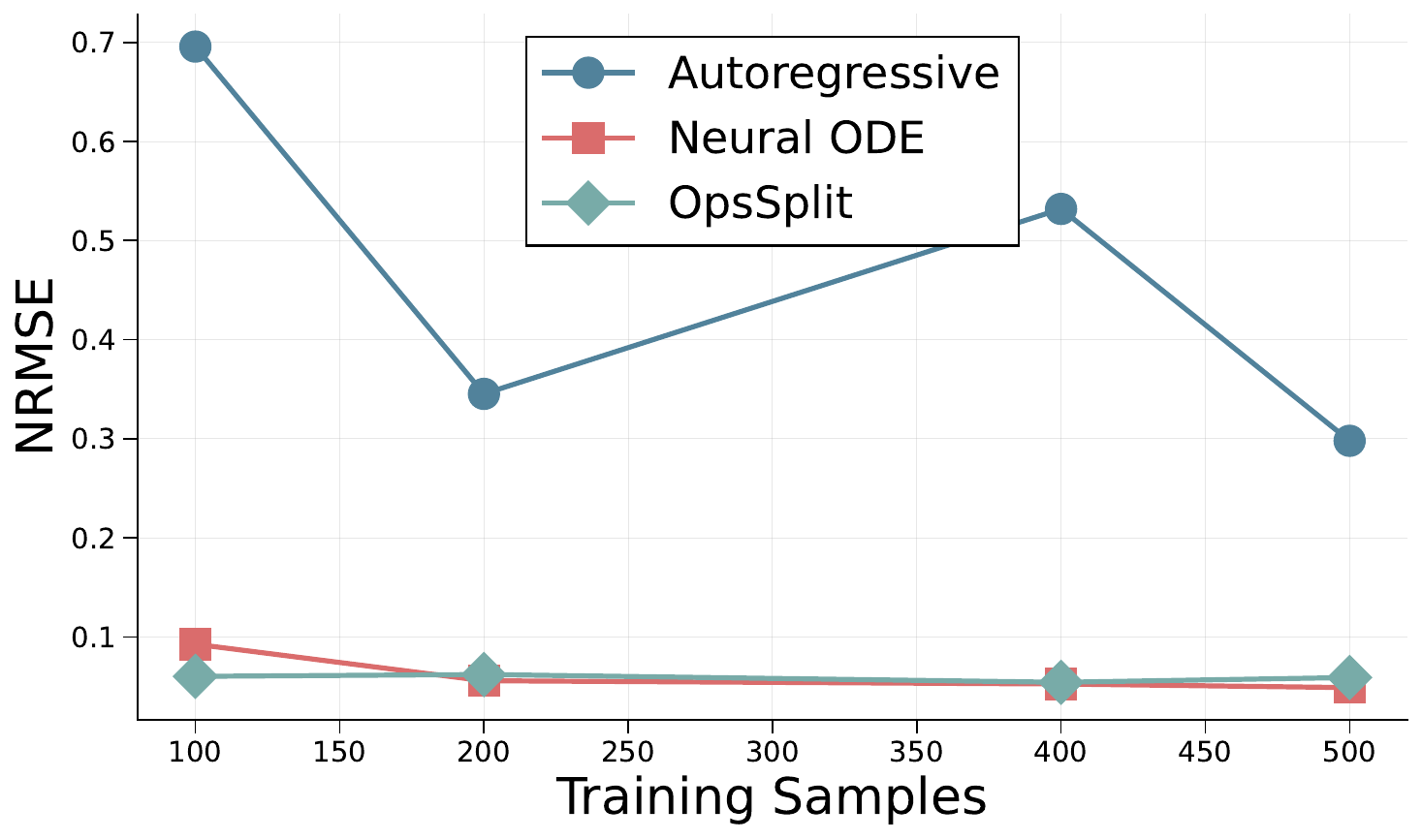}
        \caption{in-distribution}
        \label{fig:NS_incomp_data_efficiency_in}
    \end{subfigure}
    \begin{subfigure}{0.5\textwidth}
        \centering
        \includegraphics[width=0.9\linewidth]{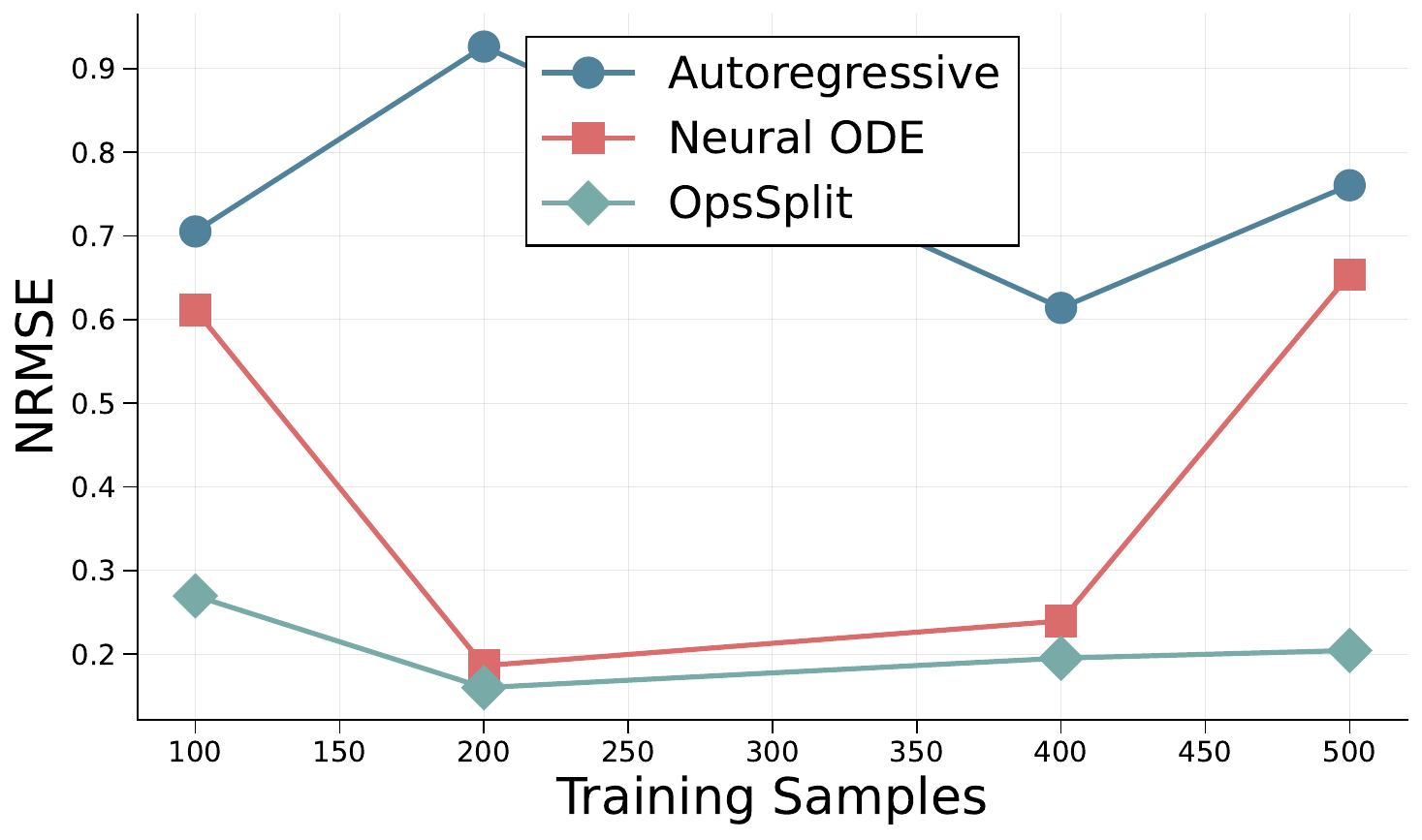}
        \caption{out-of-distribution}
        \label{fig:NS_incomp_data_efficiency_out}
    \end{subfigure}
    \begin{subfigure}{0.5\textwidth}
        \centering
        \includegraphics[width=0.9\linewidth]{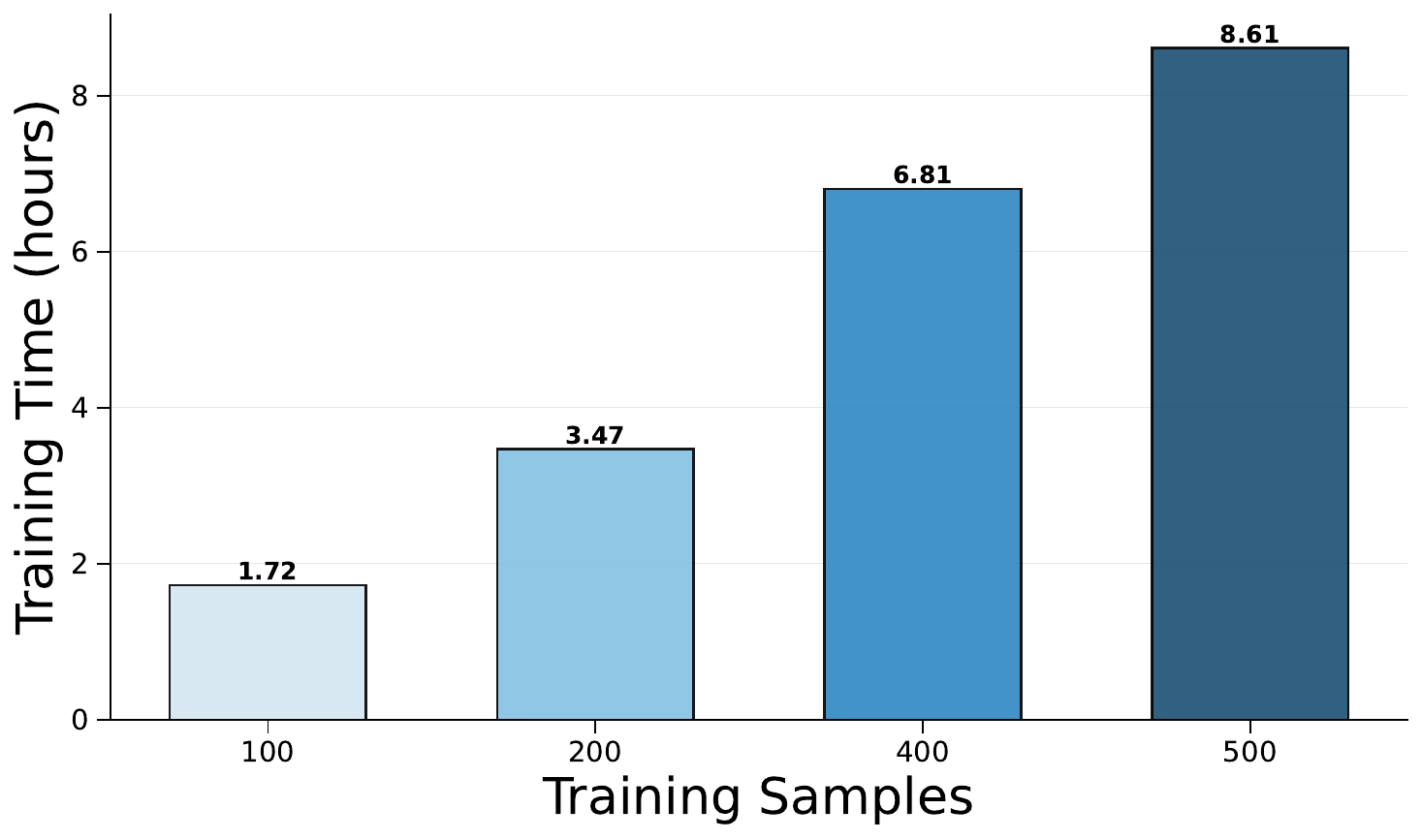}
        \caption{Training Times}
        \label{fig:NS_incomp_data_efficiency_time}
    \end{subfigure}
    \caption{Incompressible Navier--Stokes: Data Efficiency of the FNO deployed across each method for both in and out-of-distribution.}   \label{fig:NS_incomp_data_efficiency}
\end{figure}

\Cref{fig:NS_incomp_data_efficiency} demonstrates that the OpsSplit method provides a data-efficient training paradigm, exhibiting minimal sensitivity to dataset size variations. While the neural ODE approach achieves comparable performance on in-distribution data, it lacks explicit PDE structure encoding and consequently struggles to generalise to novel physical regimes. The autoregressive method benefits from additional training data, showing monotonic improvement with dataset size. However, OpsSplit consistently achieves superior performance with substantially reduced data requirements compared to baseline approaches.


\newpage
\section{Temporal Integration Method}
\label{appendix:odesolve}

\begin{figure}[H]
    \centering
    \begin{subfigure}{0.5\textwidth}
        \centering
        \includegraphics[width=0.9\linewidth]{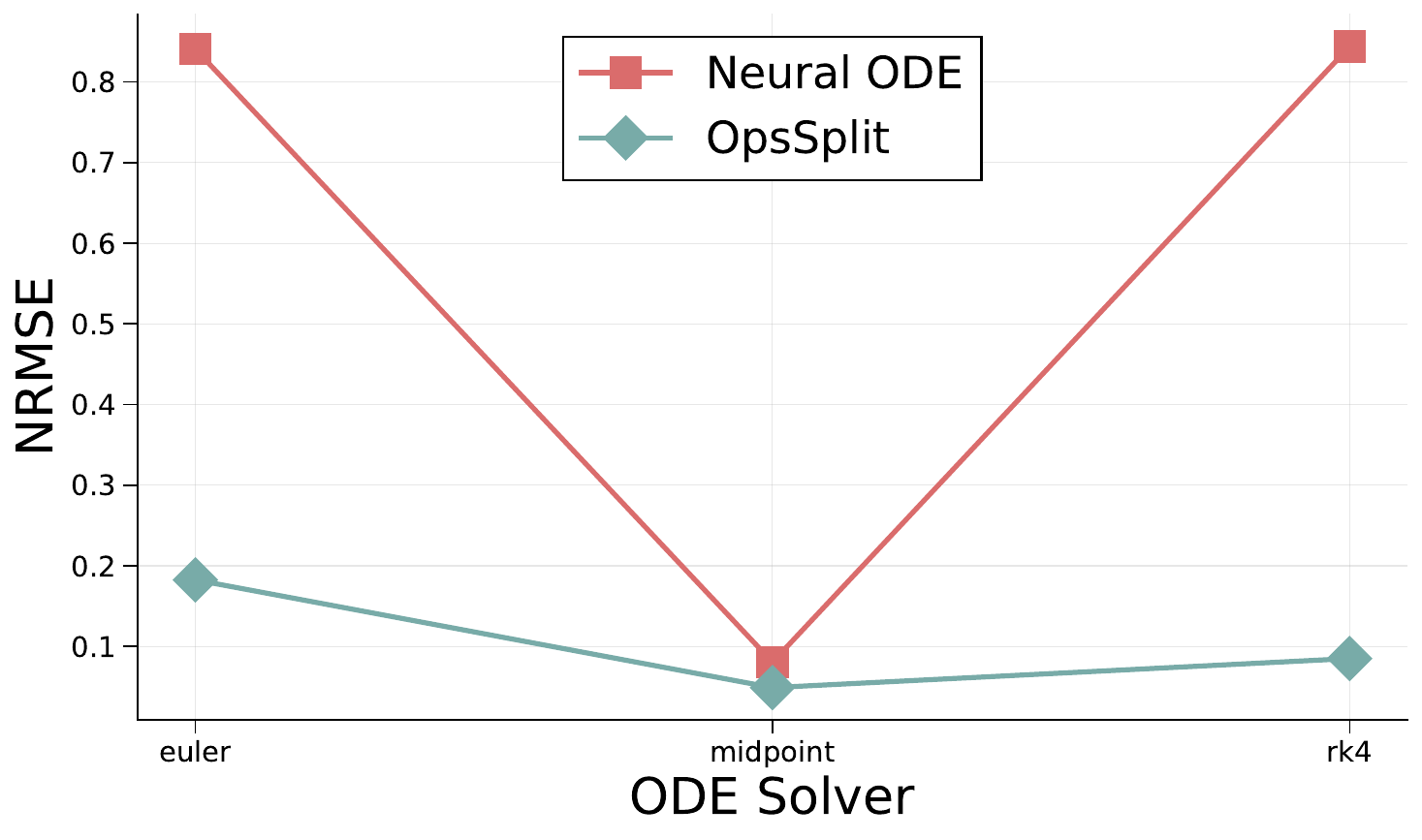}
        \caption{in-distribution}
        \label{fig:NS_incomp_odesolve_in}
    \end{subfigure}
    \begin{subfigure}{0.5\textwidth}
        \centering
        \includegraphics[width=0.9\linewidth]{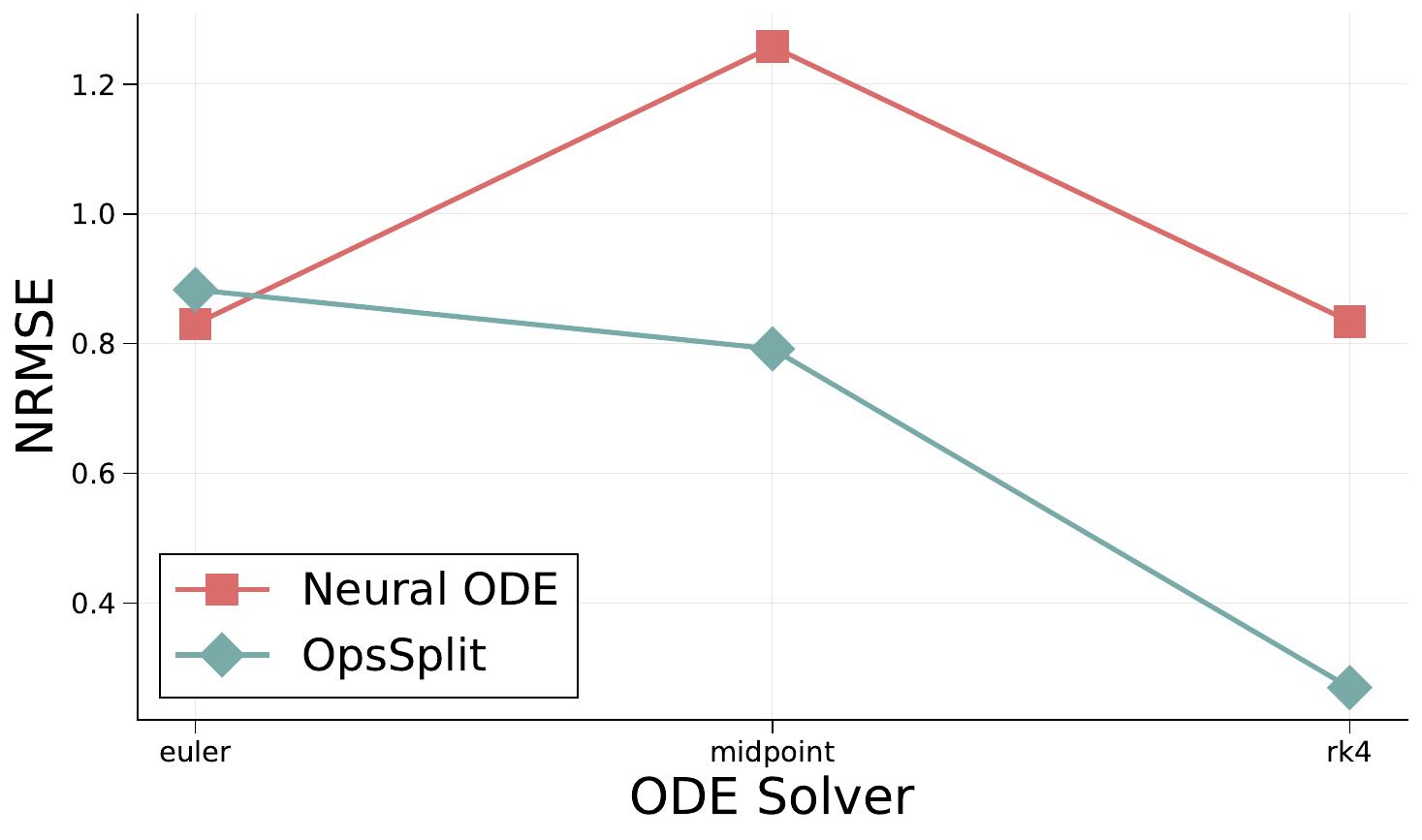}
        \caption{out-of-distribution}
        \label{fig:NS_incomp_odesolve_out}
    \end{subfigure}
    \begin{subfigure}{0.5\textwidth}
        \centering
        \includegraphics[width=0.9\linewidth]{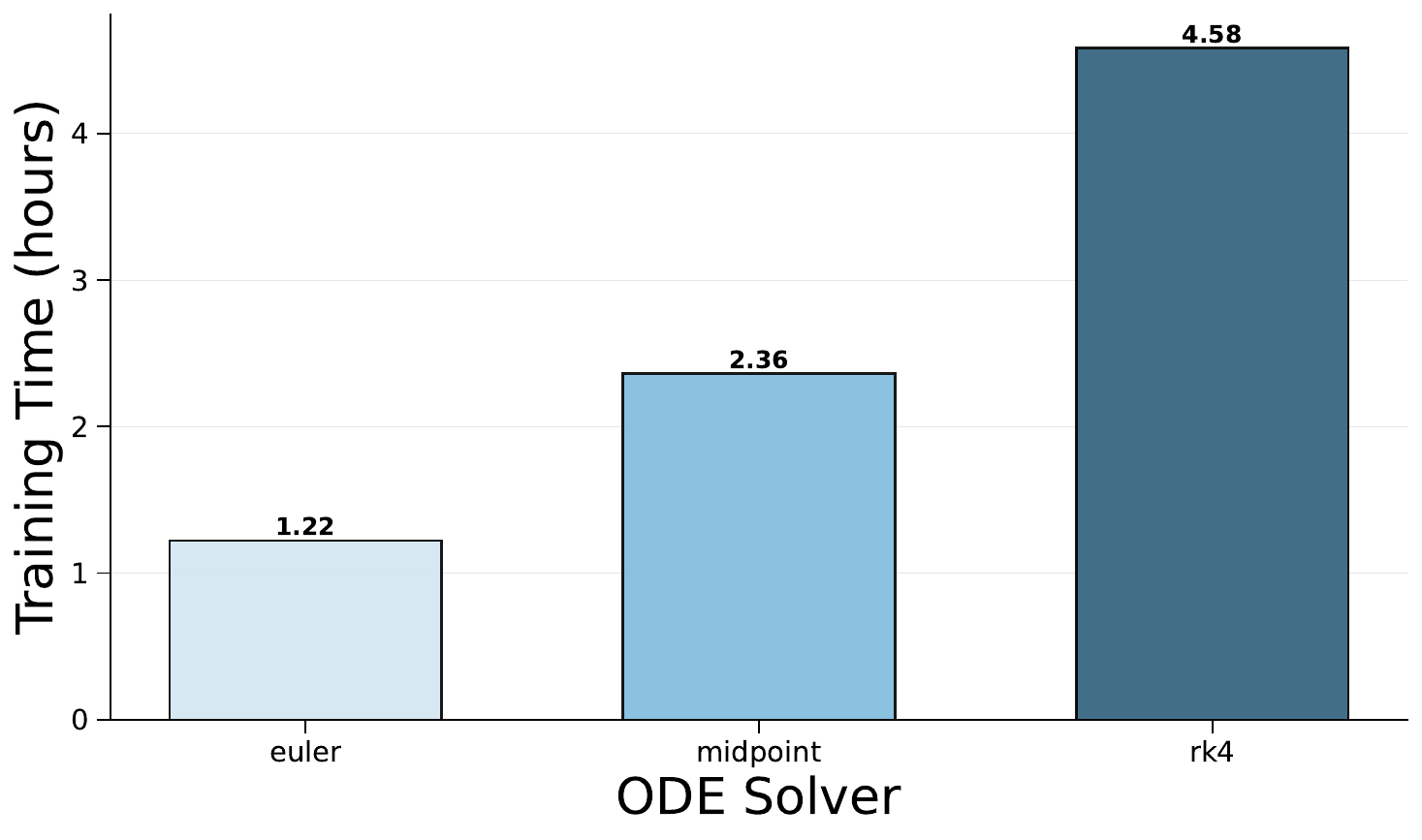}
        \caption{Training Times}
        \label{fig:NS_incomp_odesolve_time}
    \end{subfigure}
    \caption{Incompressible Navier--Stokes: Impact of finesse of ODE solver used for temporal integration of the FNO deployed across each method for both in and out-of-distribution.}   \label{fig:NS_incomp_odesolve}
\end{figure}


\newpage
\section{Convergence}
\label{appendix:convergence}

\subsection{Pre-training: Training from Scratch}

\begin{figure}[H]
    \centering
    \begin{subfigure}{0.48\textwidth}
        \centering
        \includegraphics[width=0.9\linewidth]{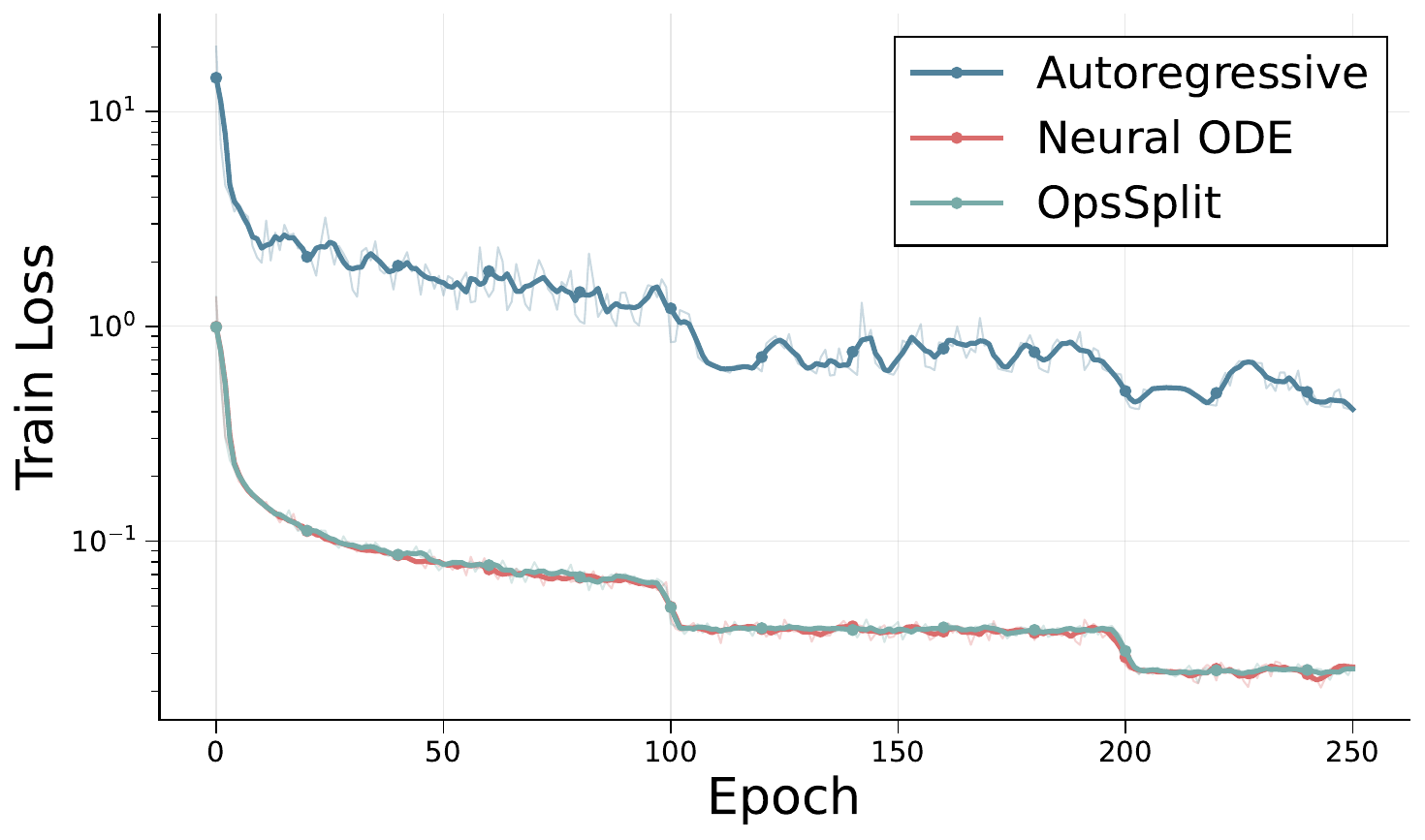}
        \caption{Train loss across methods: FNO}
        \label{fig:NS_incomp_train}
    \end{subfigure}
    \begin{subfigure}{0.48\textwidth}
        \centering
        \includegraphics[width=0.9\linewidth]{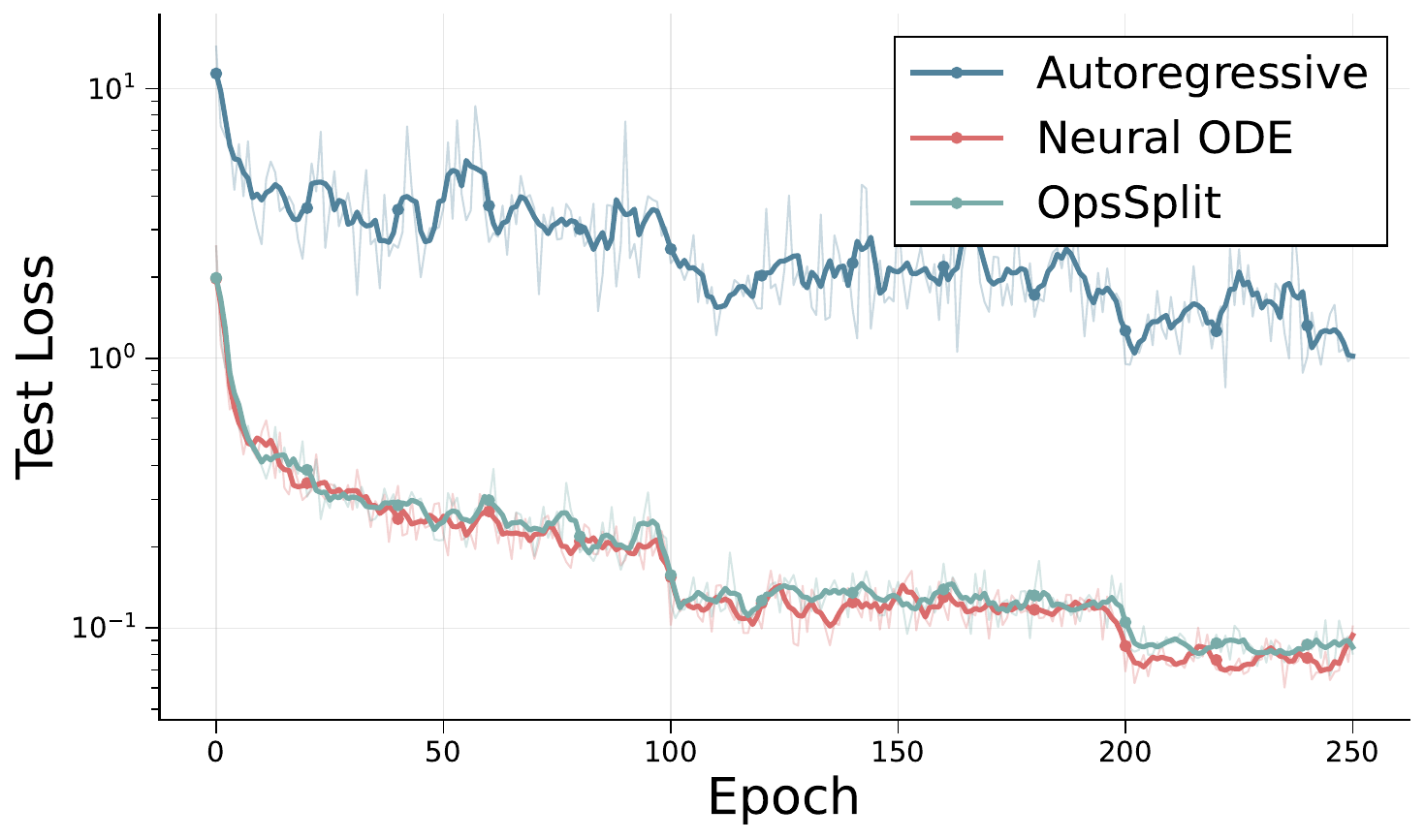}
        \caption{Test loss across methods: FNO}
        \label{fig:NS_incomp_test}
    \end{subfigure}
    \caption{Incompressible Navier--Stokes: Train and test loss convergences across training methods for the FNO}   \label{fig:NS_incomp_train_test}
\end{figure}

\begin{figure}[H]
    \centering
    \begin{subfigure}{0.48\textwidth}
        \centering
            \includegraphics[width=0.9\linewidth]{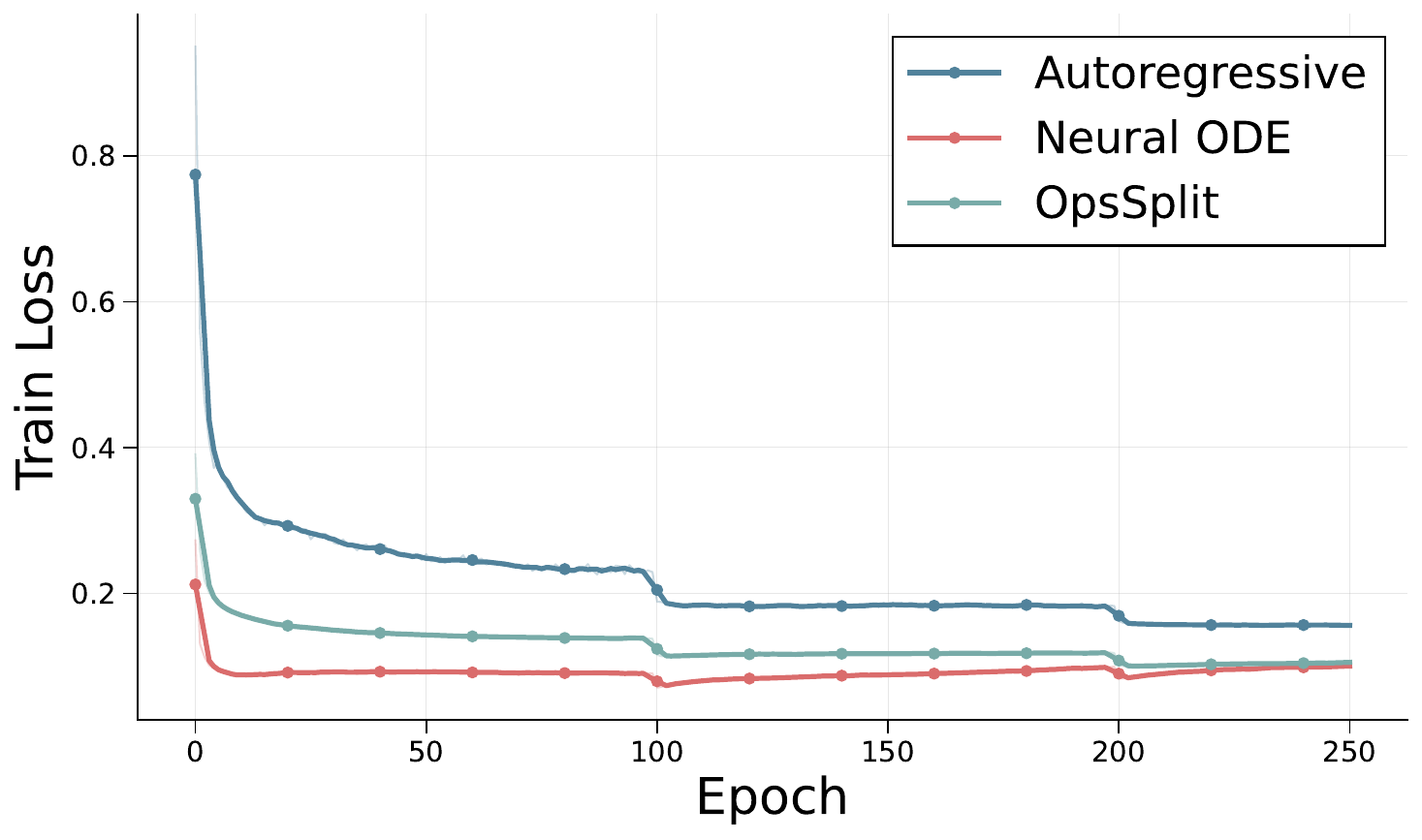}
        \caption{Train loss across methods: FNO}
        \label{fig:NS_comp_train}
    \end{subfigure}
    \begin{subfigure}{0.48\textwidth}
        \centering
        \includegraphics[width=0.9\linewidth]{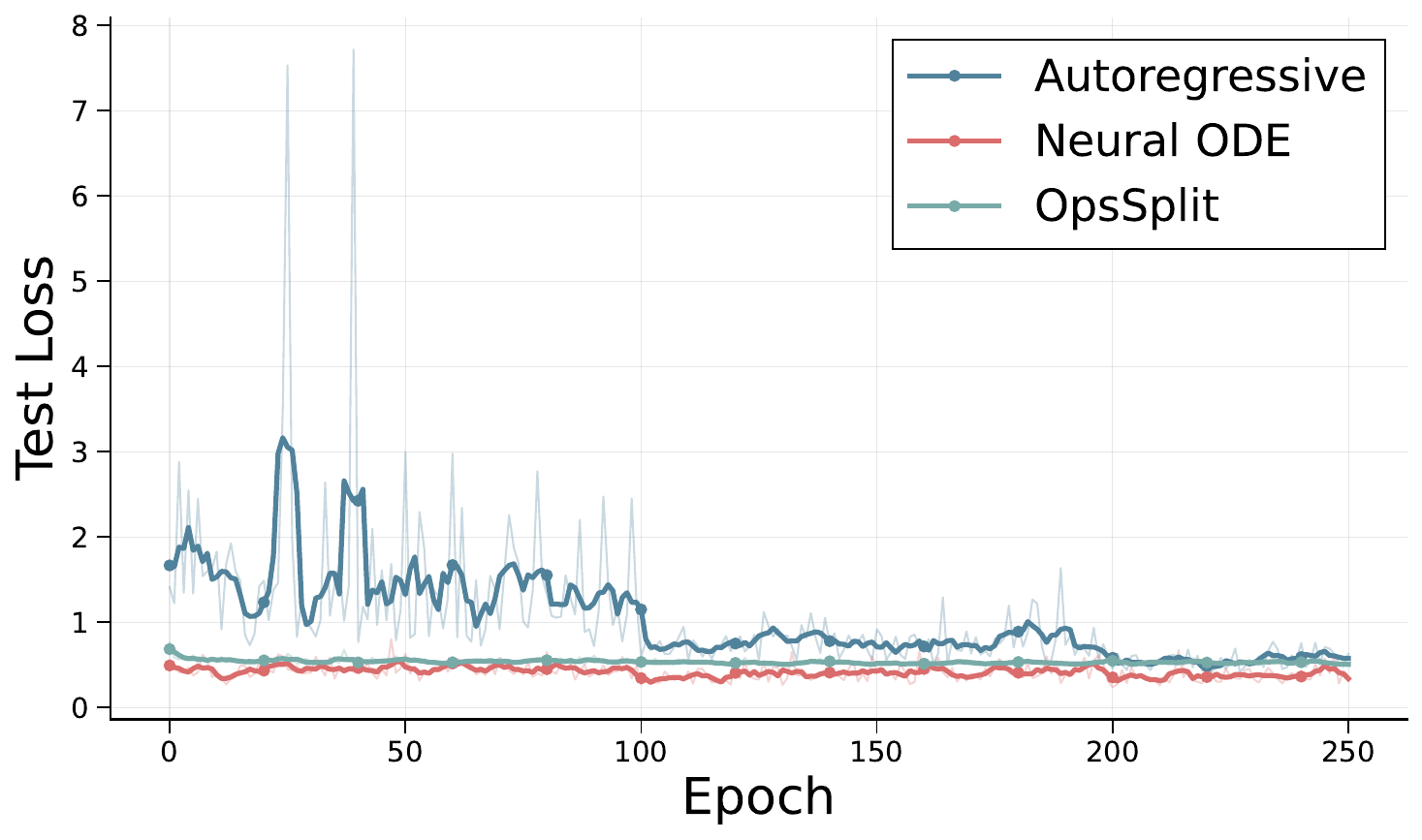}
        \caption{Test loss across methods: FNO}
        \label{fig:NS_comp_test}
    \end{subfigure}
    \caption{Compressible Navier--Stokes: Train and test loss convergences across training methods for the FNO}   
    \label{fig:NS_comp_train_test}
\end{figure}

\Cref{fig:NS_incomp_train_test,fig:NS_comp_train_test} illustrate the training and test loss convergence characteristics for three distinct neural operator deployment strategies applied to fluid dynamics problems. All models are trained from scratch using random initialisation. The comparison reveals fundamental differences in how each approach learns to approximate partial differential equations. \Cref{fig:NS_incomp_train_test} (Incompressible Navier--Stokes) exhibits dramatic differences in convergence behaviour across methods. The autoregressive approach (blue) maintains relatively elevated loss values throughout training, as evidenced by the logarithmic scale, indicating difficulty in learning the solution operator mapping directly. In contrast, the neural ODE method (red) achieves superior convergence by learning the dynamics operator, while OpsSplit (green) demonstrates comparable convergence characteristics to the neural ODE approach.

\Cref{fig:NS_comp_train_test} (Compressible Navier--Stokes) reveals analogous convergence patterns with quantitatively different magnitudes. The compressible formulation presents greater complexity due to coupled density-velocity-pressure dynamics, resulting in a more challenging learning problem than the incompressible case. While the neural ODE exhibits earlier convergence during training, the OpsSplit method achieves comparable final performance by the end of the optimisation process. The central insight from these convergence analyses is that both neural ODE and OpsSplit methods attain faster and more stable convergence by learning spatial dynamics explicitly and integrating predictions temporally, in contrast to the direct solution operator mapping employed by autoregressive approaches.

\newpage
\subsection{Fine-tuning: Transfer learning operators across PDEs}
The modular structure of the OpsSplit framework enables a unique advantage: physical operators learned for one PDE system can be transferred and fine-tuned for related systems. This section investigates whether pre-trained convection operators from one fluid dynamics regime can accelerate convergence when applied to a different regime.
We conduct bidirectional transfer learning experiments between the incompressible and compressible Navier--Stokes equations. For the incompressible case, we initialise the convection operator $\mathbb{NO}_{conv}$ in \cref{eq:ns_mom_NO} using weights pre-trained on the compressible system (\cref{eqn:comp_momentum_no}). Conversely, for the compressible case, we initialise the convection operator in \cref{eqn:comp_momentum_no} with weights from the incompressible formulation (\cref{eq:ns_mom_NO}). These fine-tuned models are compared against pre-trained models where all operators are learned from scratch using random initialisation.

\begin{figure}[H]
\centering
\begin{subfigure}{0.48\textwidth}
\centering
\includegraphics[width=0.9\linewidth]{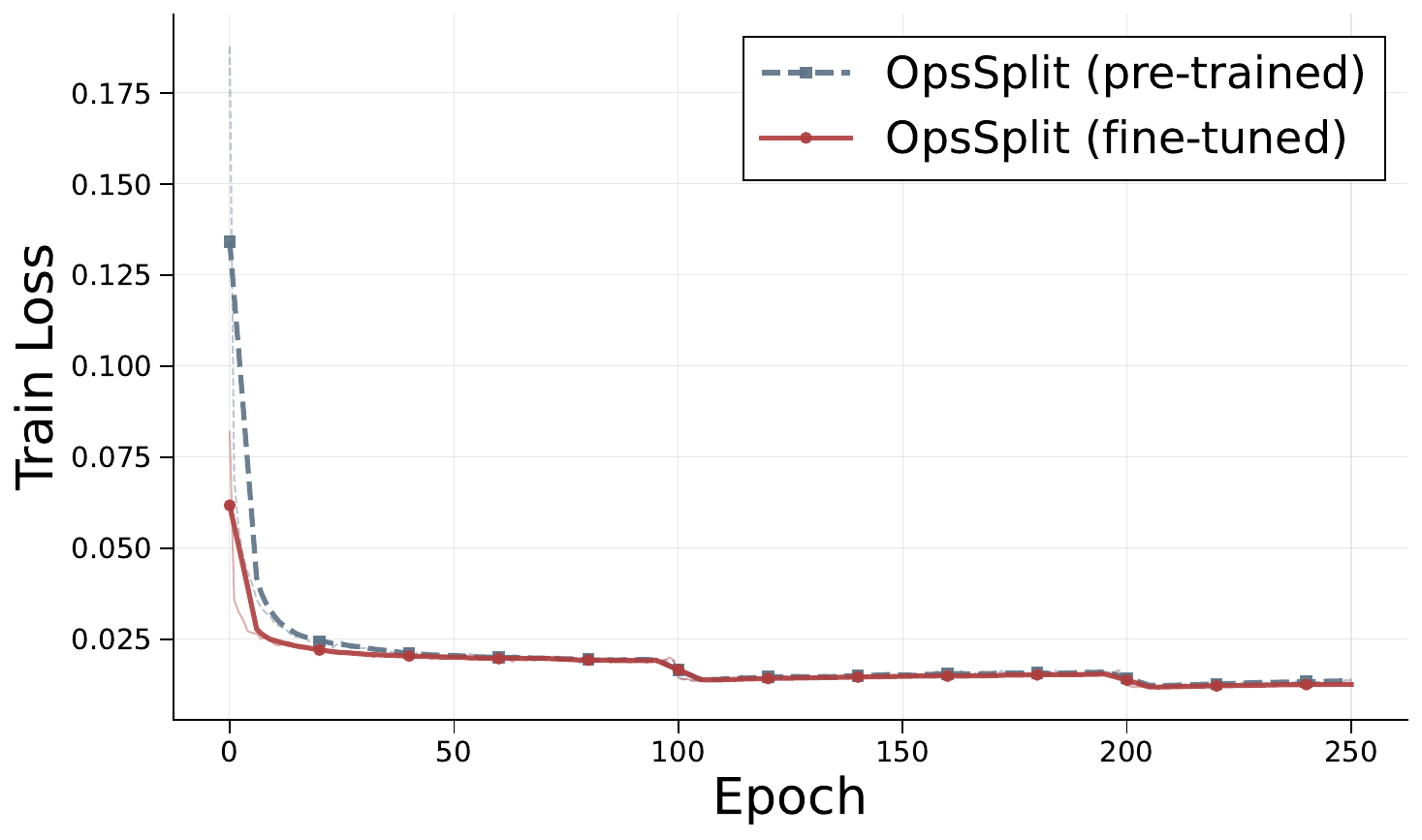}
\caption{Training loss convergence}
\label{fig:NS_incomp_ft_train}
\end{subfigure}
\begin{subfigure}{0.48\textwidth}
\centering
\includegraphics[width=0.9\linewidth]{Images/Ablations/test_losses_transfer_learn_incompressible.pdf}
\caption{Test loss convergence}
\label{fig:NS_incomp_ft_test}
\end{subfigure}
\caption{Incompressible Navier--Stokes: Comparison of training and test loss convergence for FNO-based OpsSplit models. The \textbf{pre-trained} baseline learns all operators from scratch with random initialisation. The \textbf{fine-tuned} model initialises the convection operator using weights pre-trained on the compressible Navier--Stokes equations (\cref{eqn:comp_momentum_no}), demonstrating accelerated convergence through cross-PDE transfer learning.}
\label{fig:NS_incomp_ft_train_test}
\end{figure}

\Cref{fig:NS_incomp_ft_train_test} presents convergence results for the incompressible Navier--Stokes equations. The training loss curves (\cref{fig:NS_incomp_ft_train}) show minimal difference between fine-tuned and pre-trained approaches, suggesting both methods achieve similar optimisation trajectories on the training data. However, the test loss (\cref{fig:NS_incomp_ft_test}) reveals a substantial advantage for transfer learning: the fine-tuned model achieves significantly lower test error and demonstrates faster convergence with identical computational resources. This indicates that the convection operator learned from compressible flow contains generalizable features that transfer effectively to incompressible flow physics.

\begin{figure}[H]
\centering
\begin{subfigure}{0.48\textwidth}
\centering
\includegraphics[width=0.9\linewidth]{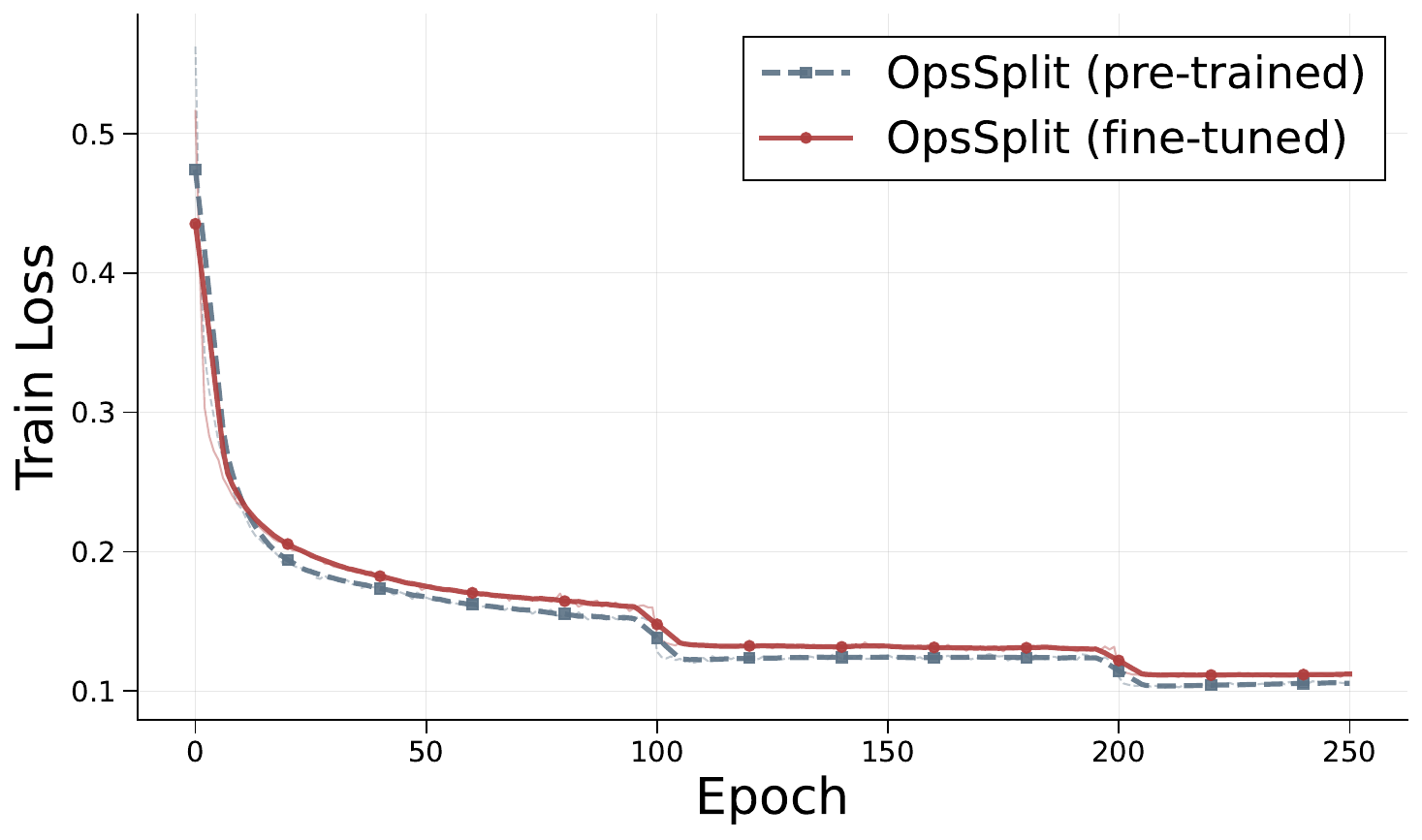}
\caption{Training loss convergence}
\label{fig:NS_comp_ft_train}
\end{subfigure}
\begin{subfigure}{0.48\textwidth}
\centering
\includegraphics[width=0.9\linewidth]{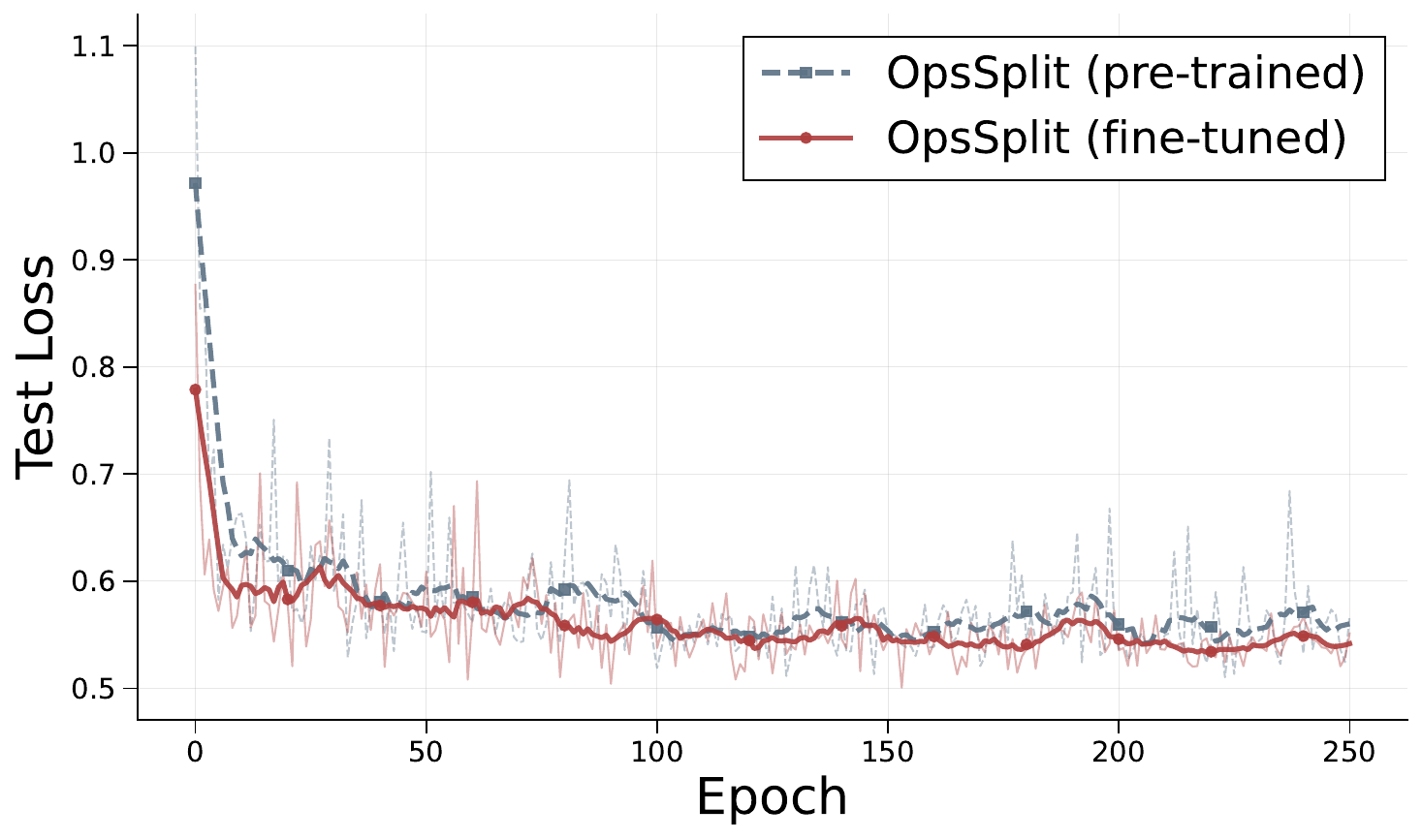}
\caption{Test loss convergence}
\label{fig:NS_comp_ft_test}
\end{subfigure}
\caption{Compressible Navier--Stokes: Comparison of training and test loss convergence for FNO-based OpsSplit models. The \textbf{pre-trained} baseline learns all operators from scratch with random initialisation. The \textbf{fine-tuned} model initialises the convection operator using weights pre-trained on the incompressible Navier--Stokes equations (\cref{eq:ns_mom_NO}), demonstrating that transfer learning benefits are bidirectional across fluid dynamics regimes.}
\label{fig:NS_comp_ft_train_test}
\end{figure}

\Cref{fig:NS_comp_ft_train_test} demonstrates similar behavior for the compressible case. The fine-tuned model, initialised with the incompressible convection operator, exhibits markedly improved test loss convergence (\cref{fig:NS_comp_ft_test}) compared to training from scratch, despite comparable training loss trajectories (\cref{fig:NS_comp_ft_train}). This bidirectional transferability confirms that the learned representations capture fundamental fluid dynamics principles that transcend specific formulations.

These results establish that neural operators trained as physical operators offer benefits beyond improved generalisation within a single PDE family. The learned operator representations can be leveraged for transfer learning across related PDE systems, enabling faster convergence and better generalisation when adapting to new physical regimes. This modularity and transferability distinguish the OpsSplit approach from monolithic neural PDE solvers, opening pathways for building reusable libraries of learned physical operators.

\newpage
\section{Incompressible Navier--Stokes Equations}
\label{appendix:incomp_ns}


\subsection{Physics}
Consider the two-dimensional incompressible Navier--Stokes equations with a fixed density:
\begin{align*}
    \label{eq: appendix_ns_equations}
    \nabla \cdot \mathbf{v} &= 0 ,  \\
    \pdv{\mathbf{v}}{t} + (\mathbf{v} \cdot \nabla) \mathbf{v}  &= \nu \nabla^2 \mathbf{v} - \nabla P, 
\end{align*}

where under constant density the pressure is given via the pressure Poisson formulation:

\begin{align}
\nabla^2 P = -\nabla \cdot [(\mathbf{v} \cdot \nabla) \mathbf{v}] \\
\label{eq:pressure_poisson}
\end{align}

with initial conditions: 

\begin{align}
        u(x,y,t=0) &=  -\sin(2 \pi \alpha y) \quad y \in [-1,1],  \\
        v(x,y,t=0) &=  -\sin(4\pi \beta  x) \quad x \in [-1,1], 
\end{align}

where $u$ defines the x-component of velocity, $v$ defines the y-component of velocity. The Navier--Stokes equations solve the flow of an incompressible fluid with a kinematic viscosity $\nu$. The system is bounded with periodic boundary conditions within the domain. The dataset is built by performing a Latin hypercube scan across the defined domain for the parameters $\alpha, \beta$,  which parametrise the initial velocity fields for each simulation. We generate 500 simulation points, each with its initial condition and use them for training. The solver is built using a spectral method outlined in \href{https://github.com/pmocz/navier-stokes-spectral-python}{Philip Mocz's code}.

Each data point, as in each simulation, is generated with a different initial condition as described above. The parameters of the initial conditions are sampled from within the domain as given in \cref{table: data_generation_ns_combined}. Each simulation is run up until wallclock time reaches $0.5$ $\Delta t = 0.001$. The spatial domain is uniformly discretised into 400 spatial units in the x and y axes. The temporal domain is subsampled to factor in every $10^{th}$ time instance, and the spatial domain is downsampled to factor every $4^{th}$ time instance, leading to a $100\times 100$ grid for the neural PDE. Parameterisation of the initial conditions and the kinematic viscosity used for both training and testing can be found in \cref{table: data_generation_ns_combined}. 

\begin{table}[h!]
\caption{Parameterisations of the 2D incompressible Navier--Stokes equations utilised for training and OOD testing}
\label{table: data_generation_ns_combined}
\vspace{0.5cm}
  \centering
  \begin{tabular}{llll}
  \hline 
  Parameter & Training & OOD Testing & Type \\ 
  \hline\\
    Velocity x-axis  $(u_0)$ & $[0.5, 1.0]$ & $[0.1, 0.5]$ & Continuous  \\
    Velocity y-axis $(v_0)$ & $[0.5, 1.0]$ & $[0.1, 0.5]$ & Continuous \\
    viscosity $\nu$ & 0.001 & 0.01 & Discrete \\ 
  \hline
  \end{tabular}
\end{table}

\subsection{Model Details}

We evaluated five neural operator architectures within each deployment method: Fourier Neural Operator (FNO) \citep{Li2021fourier}, U-Net \citep{ronneberger2015unet, gupta2023towards}, Vision Transformer (ViT) \citep{dosovitskiy2021imageworth16x16words, herde2024poseidon}, U-shaped Neural Operator (UNO) \citep{rahman2023unoushapedneuraloperators}, and Convolutional Neural Operator (CNO) \citep{bartolucci2023representation}. All models processed 2-channel velocity field inputs and outputs with configurations detailed in Table below. For each model, the hyperparameters were chosen inspired from the literature and constructed to maximise the GPU utilisation within a single H100 GPU. 

\begin{table}[H]
\centering
\begin{tabular}{|l|l|}
\hline
\textbf{Model} & \textbf{Configuration Details - AR, NODE, OpsSplit} \\
\hline
FNO & \begin{tabular}[t]{@{}l@{}}
in\_channels: 2 \\
out\_channels: 2 \\
modes: 32 \\
width: 64 \\
n\_layers: 6 \\
activation\_func: GeLU
\end{tabular} \\
\hline
U-Net & \begin{tabular}[t]{@{}l@{}}
in\_channels: 2 \\
out\_channels: 2 \\
initial\_width: 64 \\
activation\_func: Tanh
\end{tabular} \\
\hline
ViT & \begin{tabular}[t]{@{}l@{}}
patch\_size: 4 \\
embed\_dim: 512 \\
in\_channels: 2 \\
out\_channels: 2 \\
time\_channels: 1 \\
depth: 12 \\
num\_heads: 10 \\
activation\_func: GeLU
\end{tabular} \\
\hline
CNO & \begin{tabular}[t]{@{}l@{}}
Nx: 100 \\
N\_layers: 4 \\
N\_res: 4 \\
N\_res\_neck: 2 \\
in\_channels: 2 \\
out\_channels: 2 \\
channel\_multiplier: 16 \\
activation\_func: LReLU
\end{tabular} \\
\hline
UNO & \begin{tabular}[t]{@{}l@{}}
arch: UNO \\
in\_channels: 2 \\
out\_channels: 2 \\
width: 64 \\
projection\_channels: 256 \\
n\_layers: 5 \\
uno\_out\_channels: [32,64,64,64,32] \\
uno\_n\_modes: [[16,16],[8,8],[8,8],[8,8],[16,16]] \\
domain\_padding: 0.2 \\
norm: group\_norm \\
activation\_func: GeLU
\end{tabular} \\
\hline
\end{tabular}
\label{tab:model_config_incomp_NS}
\caption{Model configuration details of neural operators used for modelling the incompressible Navier--Stokes equations. All three methods - AR, NODE, and OpsSplit utilised NO models with the same configuration as mentioned above.}
\end{table}

\subsection{Performance Plots}
In this section, we showcase figures to provide a qualitative comparison of model predictions across different deployment methods and neural operator architectures. Each visualisation displays the ground truth (top row) and model predictions (bottom row) at three representative time instances: t=1 (early-stage dynamics within training temporal resolution), t=25 (mid-simulation behaviour testing model stability), and t=50 (long-term temporal extrapolation). Velocity field visualisations are represented in the physical space within the Cartesian domain. These visualisations enable assessment of spatial accuracy (how well the model captures field patterns and structures), temporal stability (whether predictions maintain physical consistency over time), error accumulation (how prediction errors grow during autoregressive rollout or temporal extrapolation), and method comparison (relative performance of Autoregressive, Neural ODE, and OpsSplit approaches across different architectures). For quantitative metrics corresponding to these visualisations, refer to \cref{tab:incomp_ns} in the main text.

\subsubsection{FNO}
\begin{figure}[H]
    \centering
    \begin{subfigure}{0.48\textwidth}
        \centering
        \includegraphics[width=0.9\linewidth]{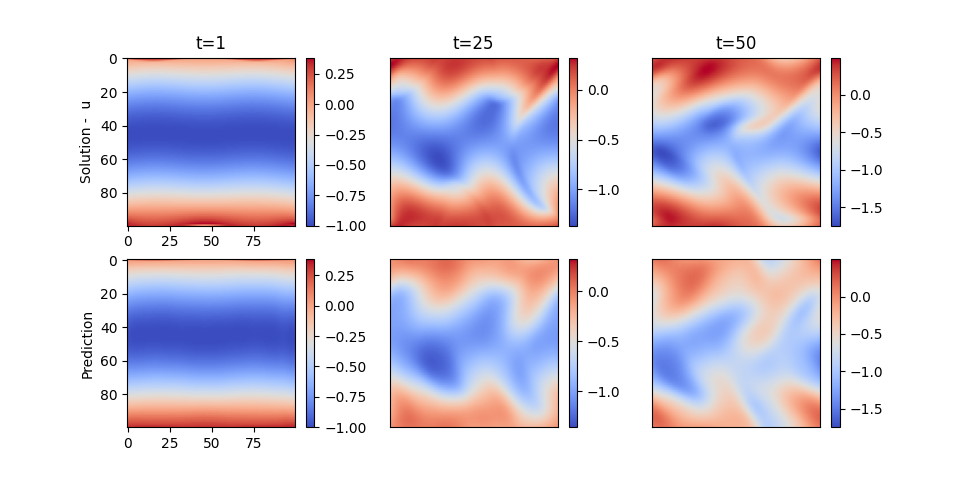}
        \caption{Horizontal velocity - Autoregressive FNO}
        \label{fig:NS_incomp_u_FNO_AR}
    \end{subfigure}
    \begin{subfigure}{0.48\textwidth}
        \centering
        \includegraphics[width=0.9\linewidth]{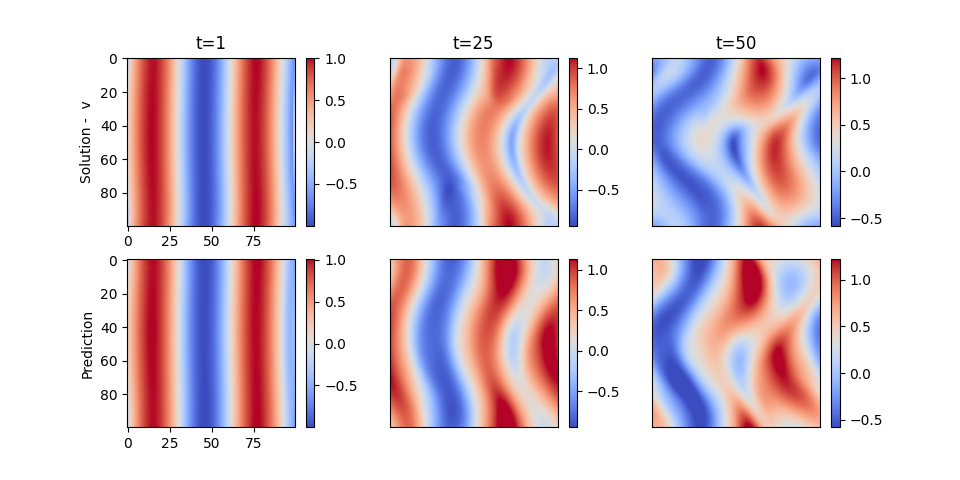}
        \caption{Vertical velocity - Autoregressive FNO}
        \label{fig:NS_incomp_v_FNO_AR}
    \end{subfigure}
    \caption{Incompressible Navier--Stokes: Model prediction within test distribution for Autoregressive FNO}
    \label{fig:NS_incomp_FNO_AR}
\end{figure}

\begin{figure}[H]
    \centering
    \begin{subfigure}{0.48\textwidth}
        \centering
        \includegraphics[width=0.9\linewidth]{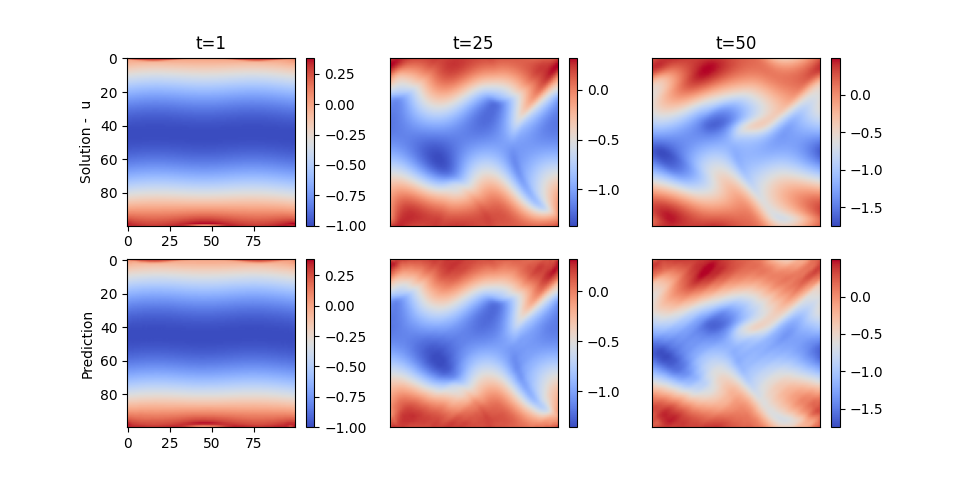}
        \caption{Horizontal velocity - NODE FNO}
        \label{fig:NS_incomp_u_FNO_NODE}
    \end{subfigure}
    \begin{subfigure}{0.48\textwidth}
        \centering
        \includegraphics[width=0.9\linewidth]{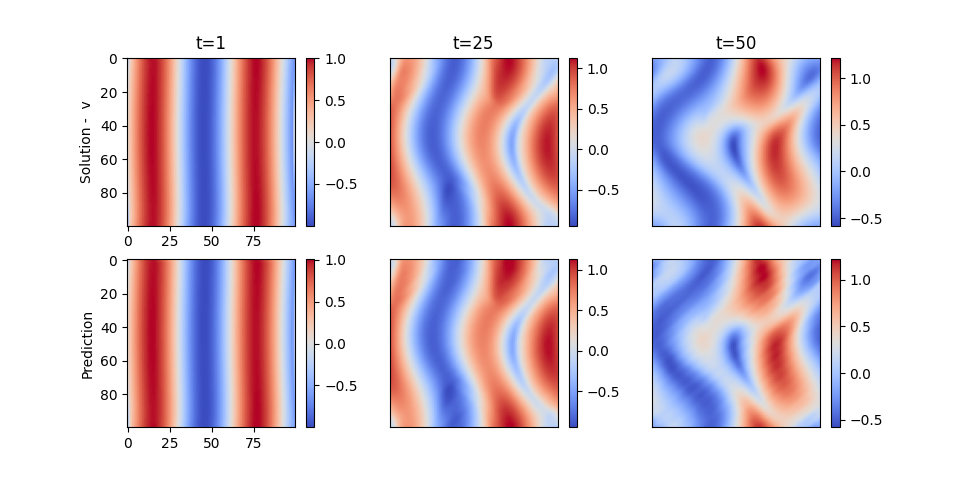}
        \caption{Vertical velocity - NODE FNO}
        \label{fig:NS_incomp_v_FNO_NODE}
    \end{subfigure}
    \caption{Incompressible Navier--Stokes: Model prediction within test distribution for Neural-ODE FNO}
    \label{fig:NS_incomp_FNO_node}
\end{figure}

\begin{figure}[H]
    \centering
    \begin{subfigure}{0.48\textwidth}
        \centering
        \includegraphics[width=0.9\linewidth]{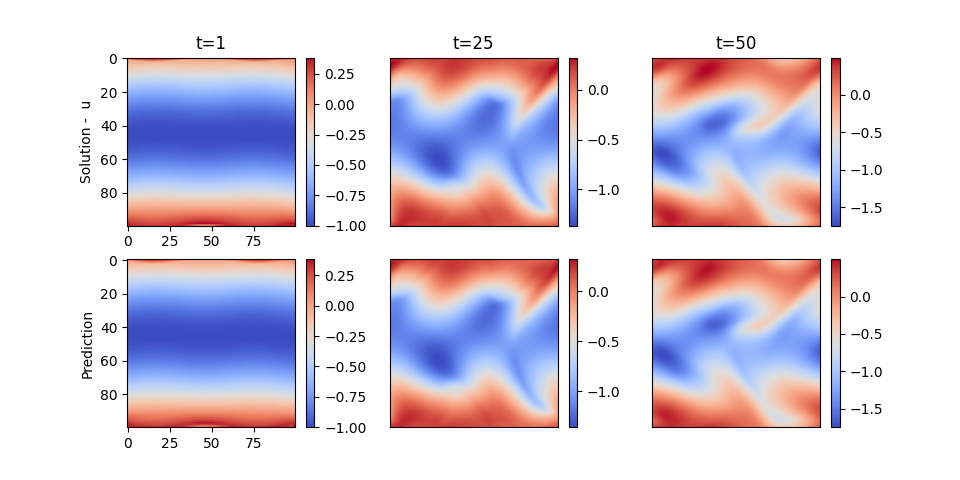}
        \caption{Horizontal velocity - OpsSplit FNO}
        \label{fig:NS_incomp_u_FNO_opsplit}
    \end{subfigure}
    \begin{subfigure}{0.48\textwidth}
        \centering
        \includegraphics[width=0.9\linewidth]{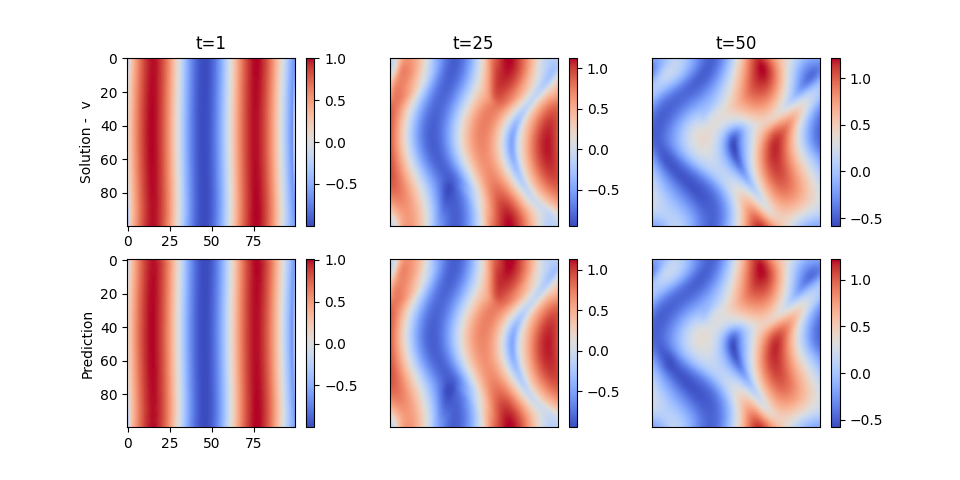}
        \caption{Vertical velocity - OpsSplit FNO}
        \label{fig:NS_incomp_v_FNO_opssplit}
    \end{subfigure}
    \caption{Incompressible Navier--Stokes: Model prediction within test distribution for OpsSplit FNO}
    \label{fig:NS_incomp_FNO_opssplit}
\end{figure}

\subsubsection{U-Net}

\begin{figure}[H]
    \centering
    \begin{subfigure}{0.48\textwidth}
        \centering
        \includegraphics[width=0.9\linewidth]{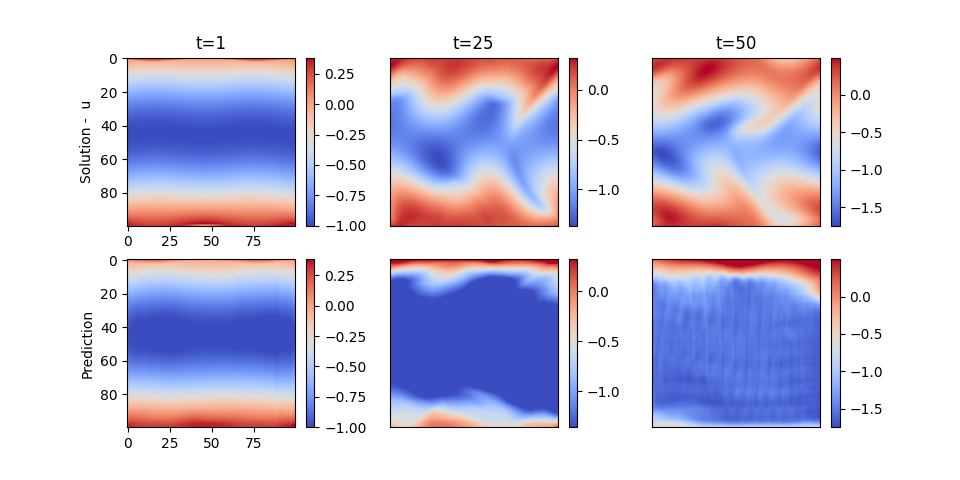}
        \caption{Horizontal velocity - Autoregressive U-Net}
        \label{fig:NS_incomp_u_UNet_AR}
    \end{subfigure}
    \begin{subfigure}{0.48\textwidth}
        \centering
        \includegraphics[width=0.9\linewidth]{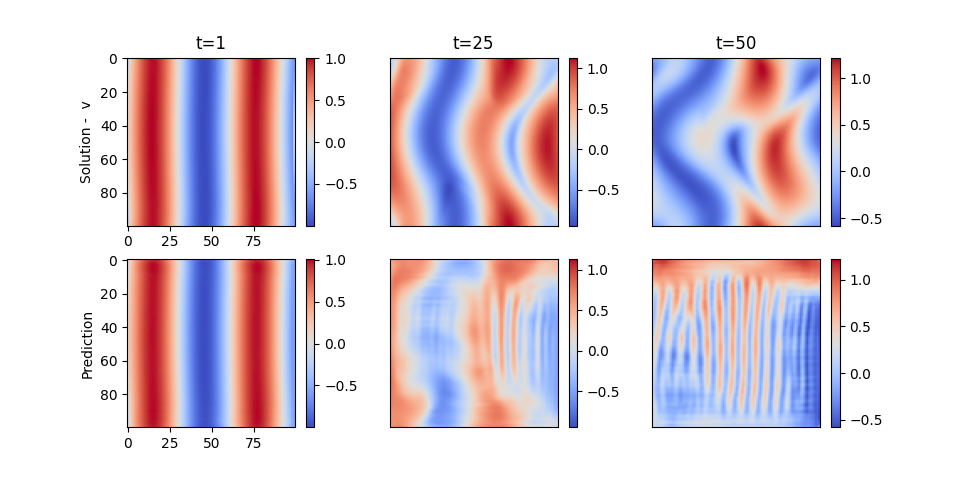}
        \caption{Vertical velocity - Autoregressive U-Net}
        \label{fig:NS_incomp_v_UNet_AR}
    \end{subfigure}
    \caption{Incompressible Navier--Stokes: Model prediction within test distribution for Autoregressive U-Net}
    \label{fig:NS_incomp_UNet_AR}
\end{figure}

\begin{figure}[H]
    \centering
    \begin{subfigure}{0.48\textwidth}
        \centering
        \includegraphics[width=0.9\linewidth]{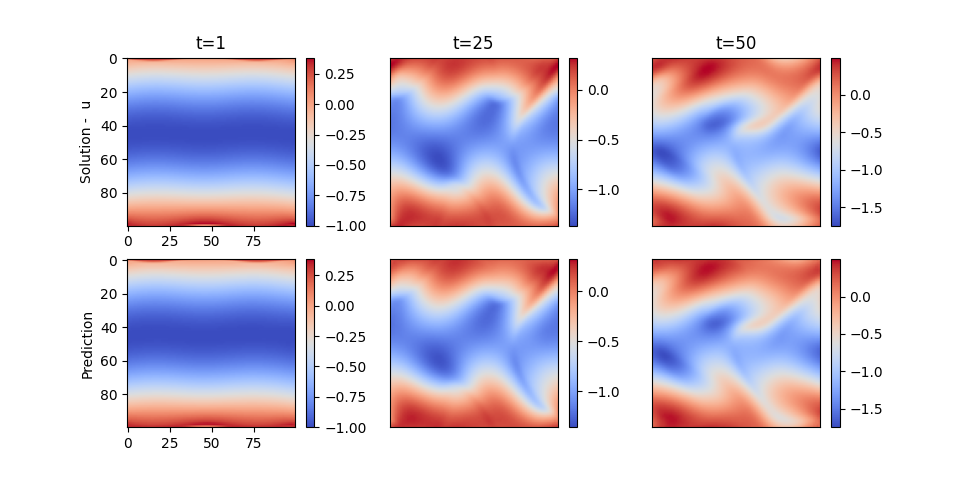}
        \caption{Horizontal velocity - NODE U-Net}
        \label{fig:NS_incomp_u_UNet_NODE}
    \end{subfigure}
    \begin{subfigure}{0.48\textwidth}
        \centering
        \includegraphics[width=0.9\linewidth]{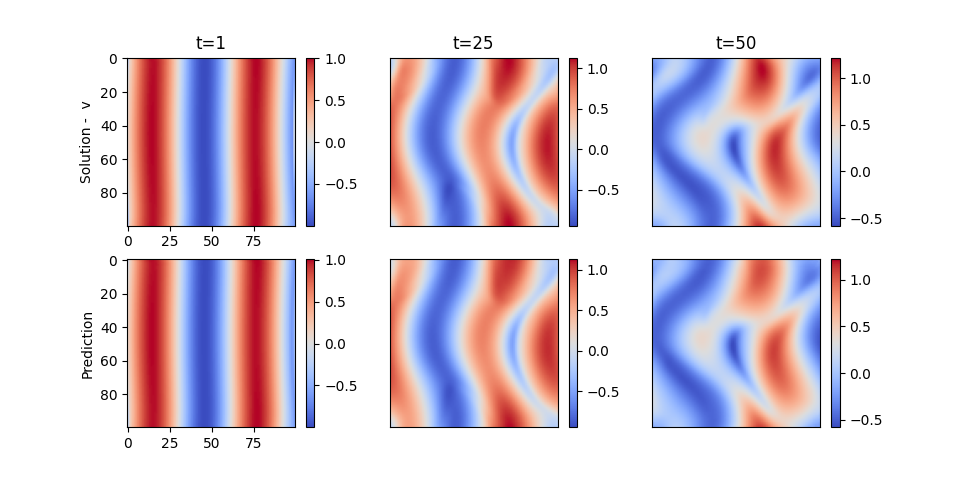}
        \caption{Vertical velocity - NODE U-Net}
        \label{fig:NS_incomp_v_UNet_NODE}
    \end{subfigure}
    \caption{Incompressible Navier--Stokes: Model prediction within test distribution for Neural-ODE U-Net}
    \label{fig:NS_incomp_Unet_node}
\end{figure}

\begin{figure}[H]
    \centering
    \begin{subfigure}{0.48\textwidth}
        \centering
        \includegraphics[width=0.9\linewidth]{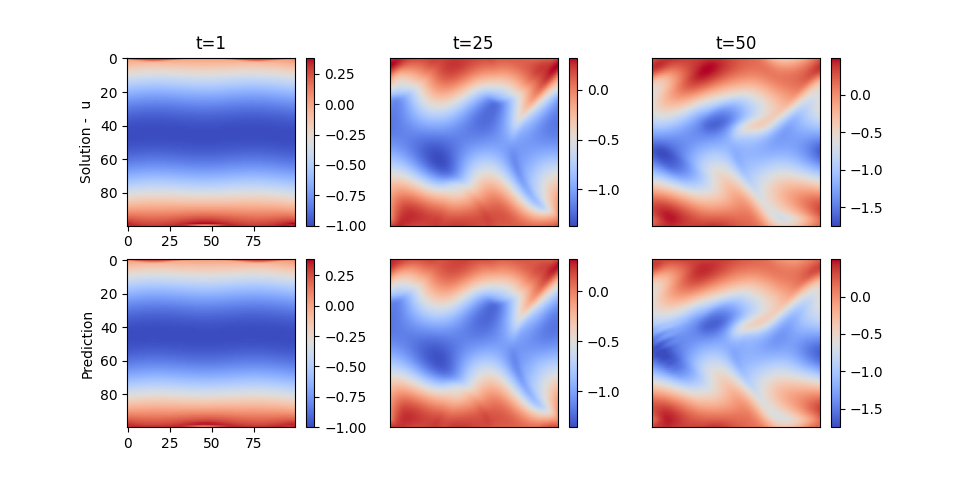}
        \caption{Horizontal velocity - OpsSplit U-Net}
        \label{fig:NS_incomp_u_UNet_opsplit}
    \end{subfigure}
    \begin{subfigure}{0.48\textwidth}
        \centering
        \includegraphics[width=0.9\linewidth]{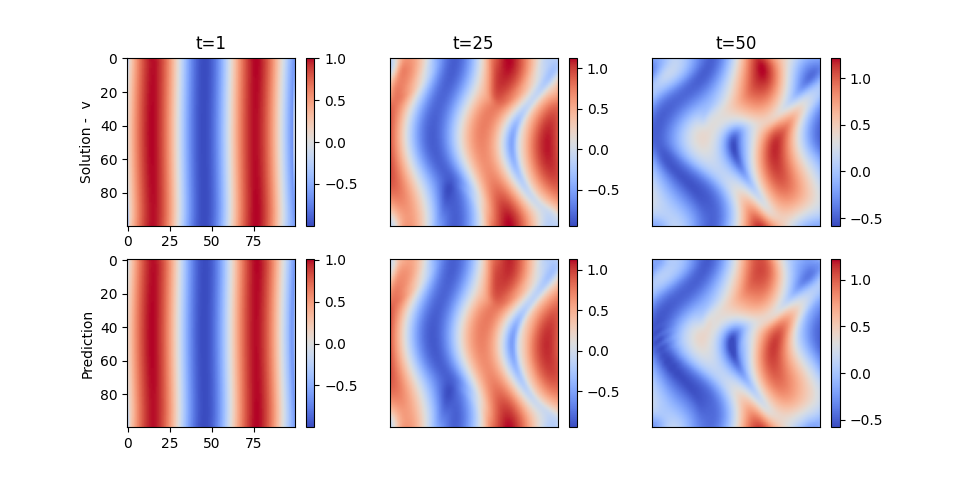}
        \caption{Vertical velocity - OpsSplit U-Net}
        \label{fig:NS_incomp_v_UNet_opssplit}
    \end{subfigure}
    \caption{Incompressible Navier--Stokes: Model prediction within test distribution for OpsSplit U-Net}
    \label{fig:NS_incomp_UNet_opssplit}
\end{figure}

\subsubsection{ViT}

\begin{figure}[H]
    \centering
    \begin{subfigure}{0.48\textwidth}
        \centering
        \includegraphics[width=0.9\linewidth]{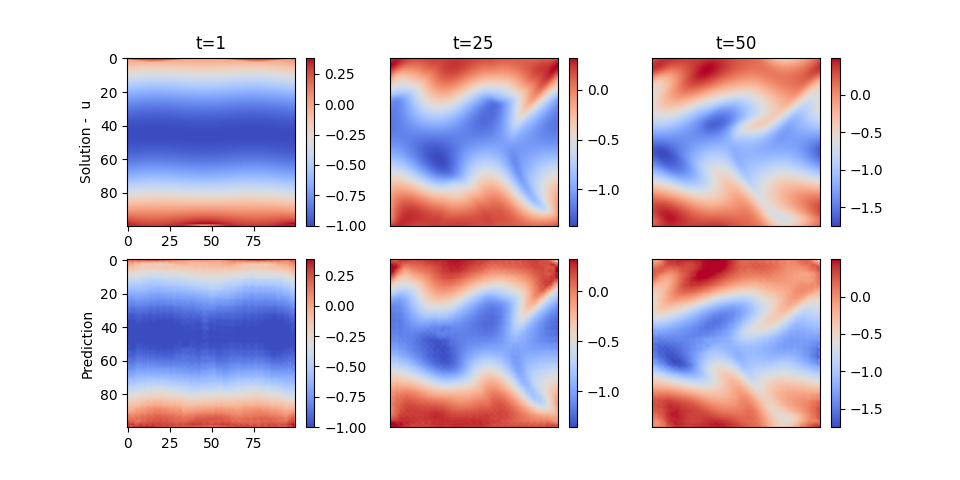}
        \caption{Horizontal velocity - Autoregressive ViT}
        \label{fig:NS_incomp_u_ViT_AR}
    \end{subfigure}
    \begin{subfigure}{0.48\textwidth}
        \centering
        \includegraphics[width=0.9\linewidth]{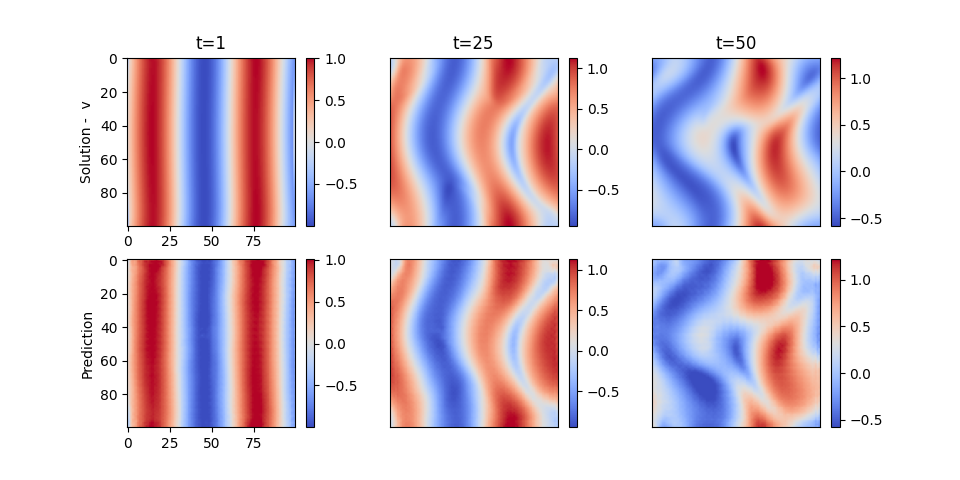}
        \caption{Vertical velocity - Autoregressive ViT}
        \label{fig:NS_incomp_v_ViT_AR}
    \end{subfigure}
    \caption{Incompressible Navier--Stokes: Model prediction within test distribution for Autoregressive ViT}
    \label{fig:NS_incomp_ViT_AR}
\end{figure}

\begin{figure}[H]
    \centering
    \begin{subfigure}{0.48\textwidth}
        \centering
        \includegraphics[width=0.9\linewidth]{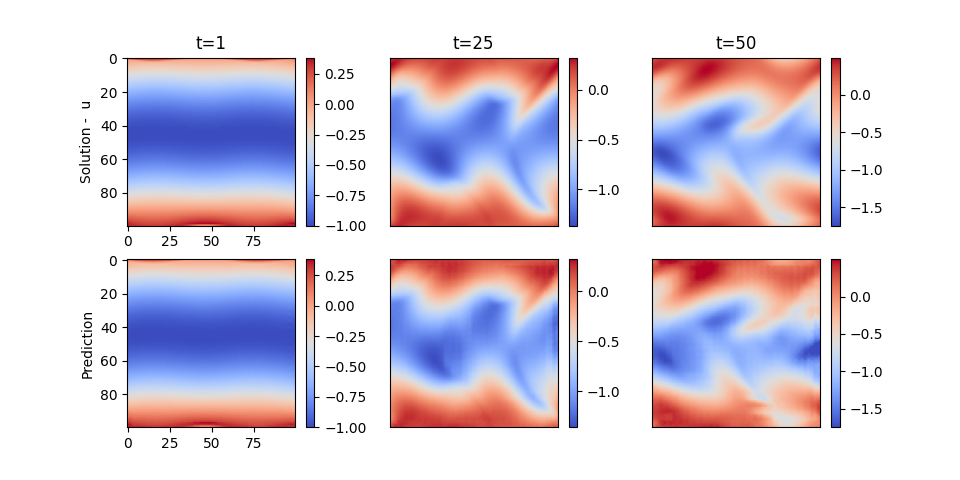}
        \caption{Horizontal velocity - NODE ViT}
        \label{fig:NS_incomp_u_ViT_NODE}
    \end{subfigure}
    \begin{subfigure}{0.48\textwidth}
        \centering
        \includegraphics[width=0.9\linewidth]{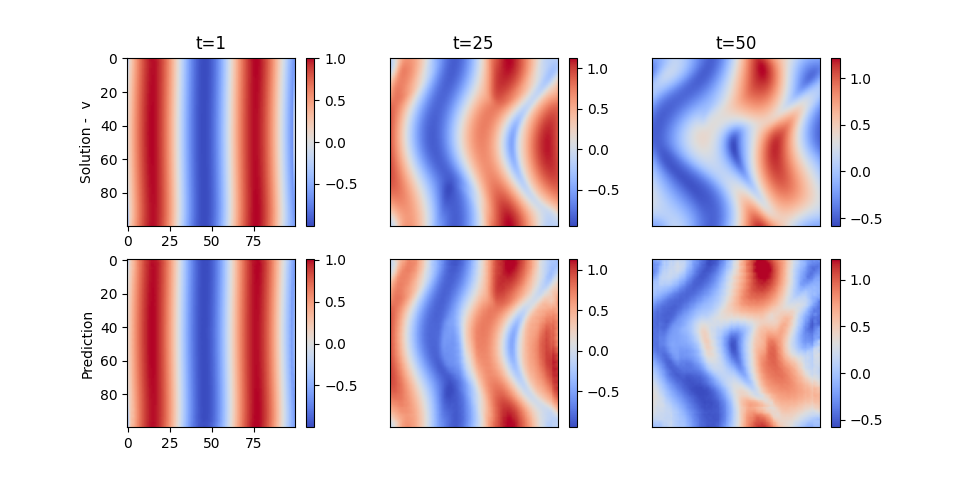}
        \caption{Vertical velocity - NODE ViT}
        \label{fig:NS_incomp_v_ViT_NODE}
    \end{subfigure}
    \caption{Incompressible Navier--Stokes: Model prediction within test distribution for Neural-ODE ViT}
    \label{fig:NS_incomp_ViT_node}
\end{figure}

\begin{figure}[H]
    \centering
    \begin{subfigure}{0.48\textwidth}
        \centering
        \includegraphics[width=0.9\linewidth]{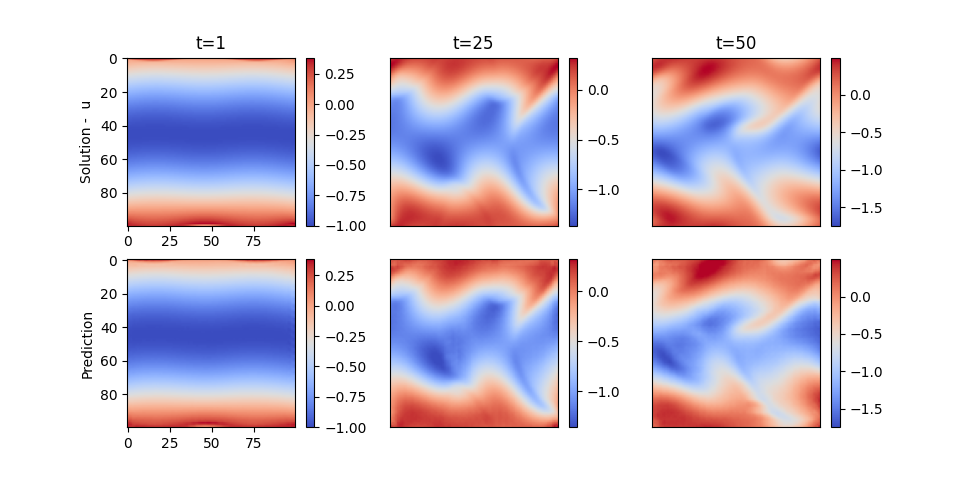}
        \caption{Horizontal velocity - OpsSplit ViT}
        \label{fig:NS_incomp_u_ViT_opsplit}
    \end{subfigure}
    \begin{subfigure}{0.48\textwidth}
        \centering
        \includegraphics[width=0.9\linewidth]{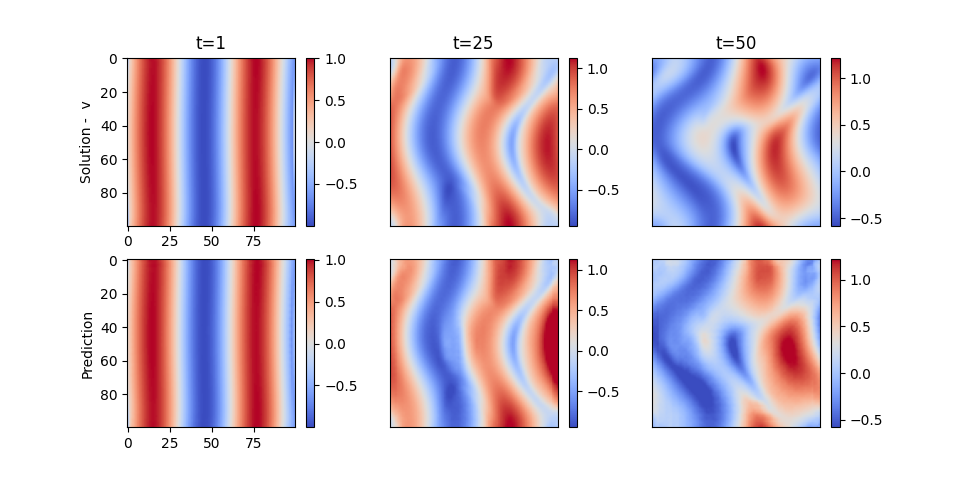}
        \caption{Vertical velocity - OpsSplit ViT}
        \label{fig:NS_incomp_v_ViT_opssplit}
    \end{subfigure}
    \caption{Incompressible Navier--Stokes: Model prediction within test distribution for OpsSplit ViT}
    \label{fig:NS_incomp_ViT_opssplit}
\end{figure}

\subsubsection{CNO}

\begin{figure}[H]
    \centering
    \begin{subfigure}{0.48\textwidth}
        \centering
        \includegraphics[width=0.9\linewidth]{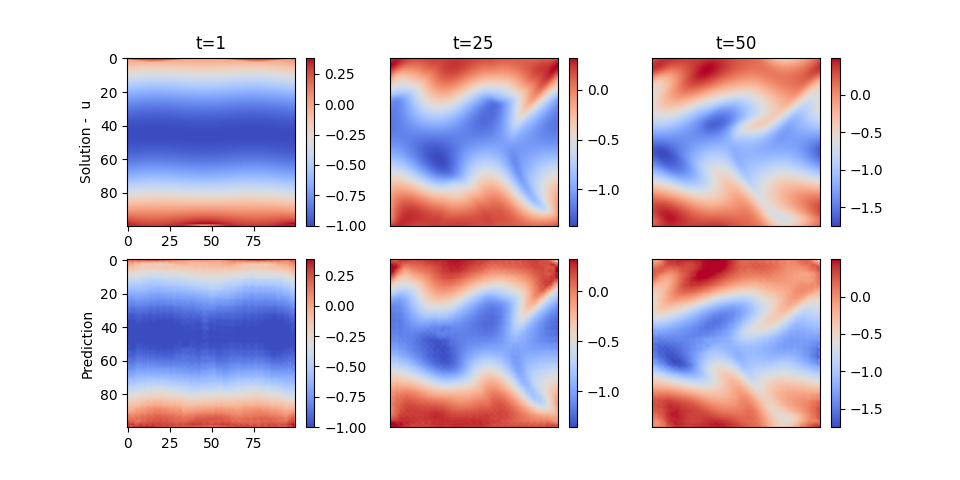}
        \caption{Horizontal velocity - Autoregressive CNO}
        \label{fig:NS_incomp_u_CNO_AR}
    \end{subfigure}
    \begin{subfigure}{0.48\textwidth}
        \centering
        \includegraphics[width=0.9\linewidth]{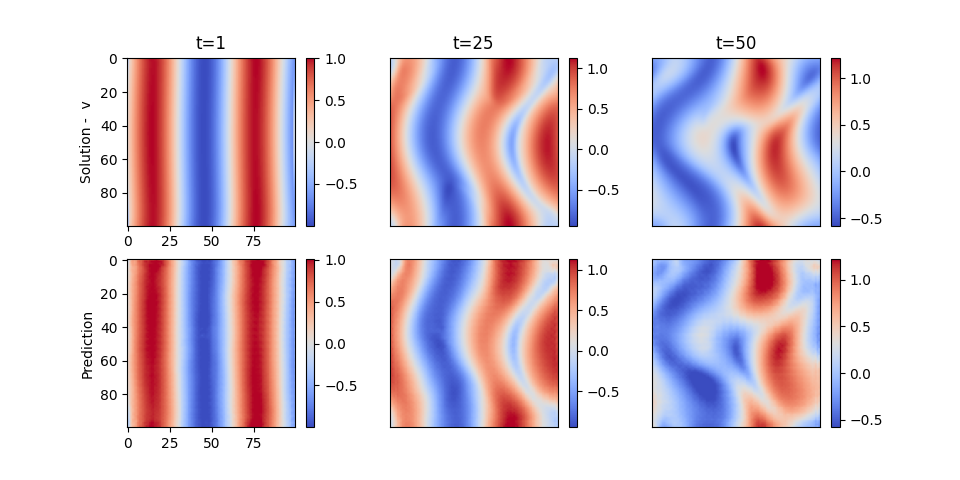}
        \caption{Vertical velocity - Autoregressive CNO}
        \label{fig:NS_incomp_v_CNO_AR}
    \end{subfigure}
    \caption{Incompressible Navier--Stokes: Model prediction within test distribution for Autoregressive CNO}
    \label{fig:NS_incomp_CNO_AR}
\end{figure}

\begin{figure}[H]
    \centering
    \begin{subfigure}{0.48\textwidth}
        \centering
        \includegraphics[width=0.9\linewidth]{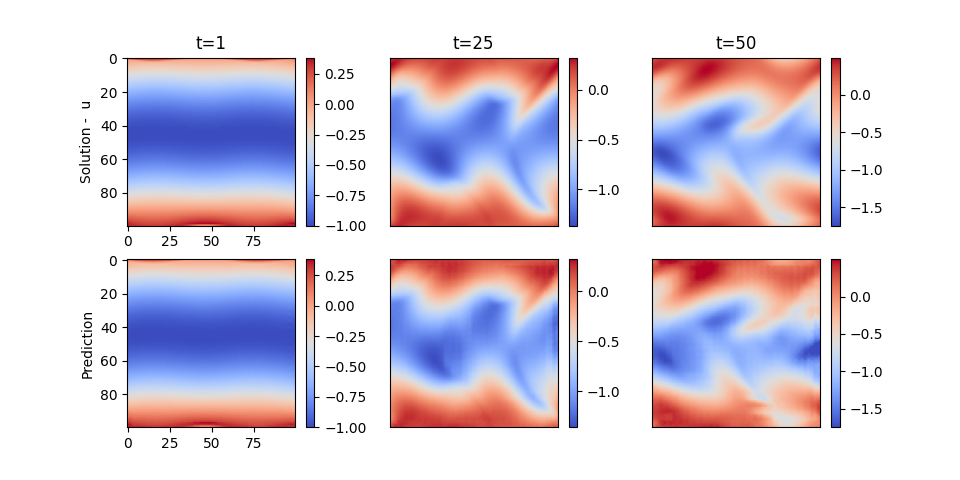}
        \caption{Horizontal velocity - NODE CNO}
        \label{fig:NS_incomp_u_CNO_NODE}
    \end{subfigure}
    \begin{subfigure}{0.48\textwidth}
        \centering
        \includegraphics[width=0.9\linewidth]{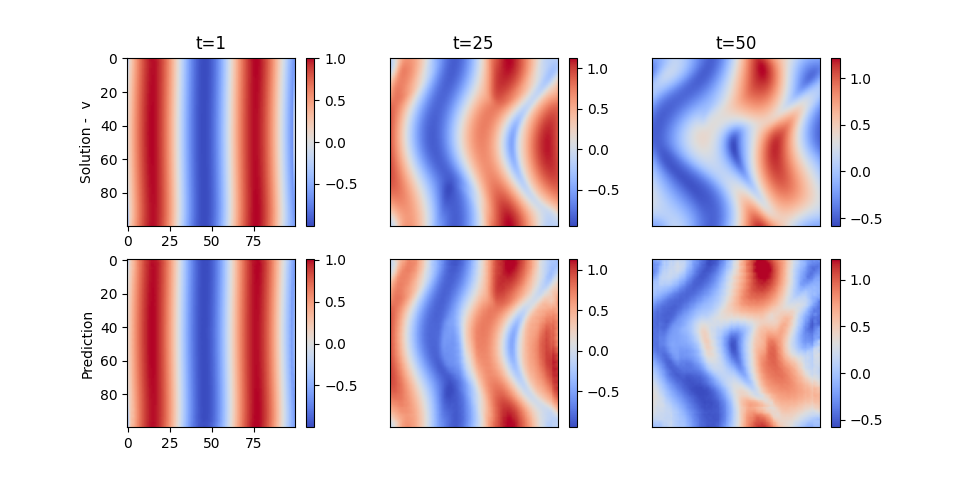}
        \caption{Vertical velocity - NODE CNO}
        \label{fig:NS_incomp_v_CNO_NODE}
    \end{subfigure}
    \caption{Incompressible Navier--Stokes: Model prediction within test distribution for Neural-ODE CNO}
    \label{fig:NS_incomp_CNO_node}
\end{figure}

\begin{figure}[H]
    \centering
    \begin{subfigure}{0.48\textwidth}
        \centering
        \includegraphics[width=0.9\linewidth]{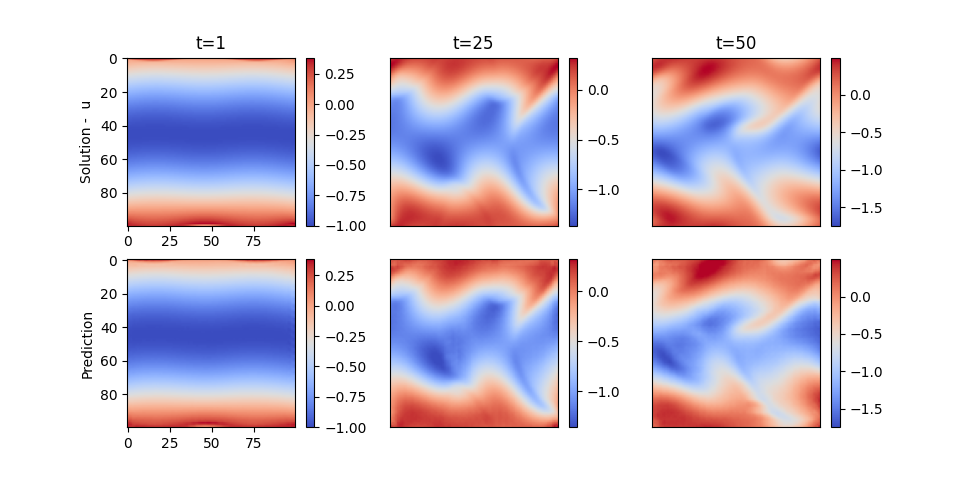}
        \caption{Horizontal velocity - OpsSplit CNO}
        \label{fig:NS_incomp_u_CNO_opsplit}
    \end{subfigure}
    \begin{subfigure}{0.48\textwidth}
        \centering
        \includegraphics[width=0.9\linewidth]{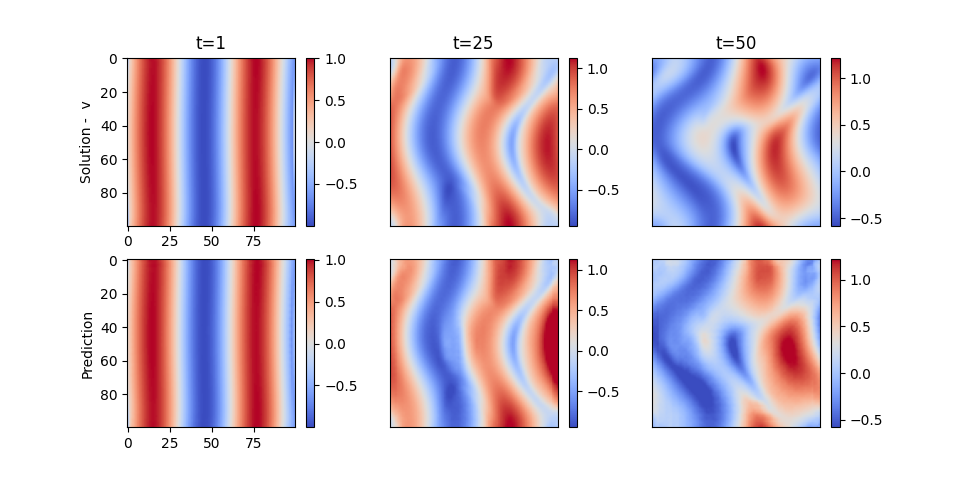}
        \caption{Vertical velocity - OpsSplit CNO}
        \label{fig:NS_incomp_v_CNO_opssplit}
    \end{subfigure}
    \caption{Incompressible Navier--Stokes: Model prediction within test distribution for OpsSplit CNO}
    \label{fig:NS_incomp_CNO_opssplit}
\end{figure}

\subsubsection{UNO}

\begin{figure}[H]
    \centering
    \begin{subfigure}{0.48\textwidth}
        \centering
        \includegraphics[width=0.9\linewidth]{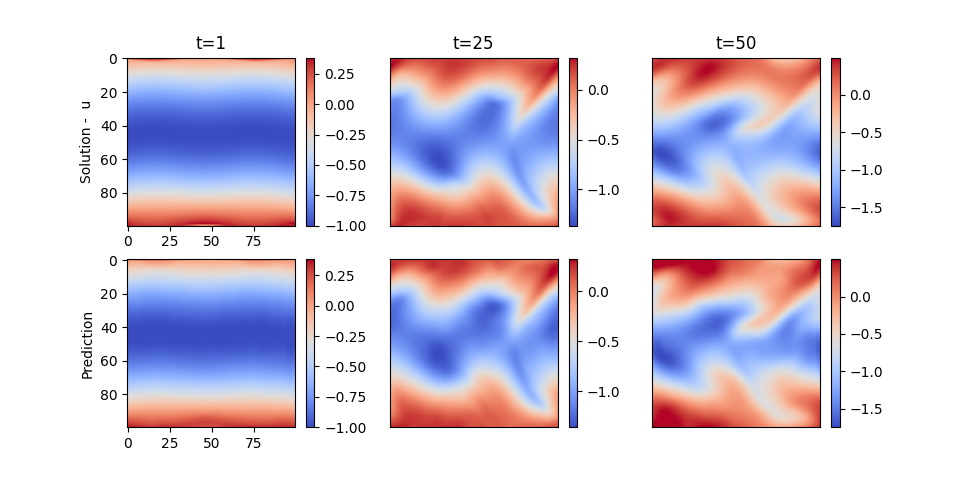}
        \caption{Horizontal velocity - Autoregressive UNO}
        \label{fig:NS_incomp_u_UNO_AR}
    \end{subfigure}
    \begin{subfigure}{0.48\textwidth}
        \centering
        \includegraphics[width=0.9\linewidth]{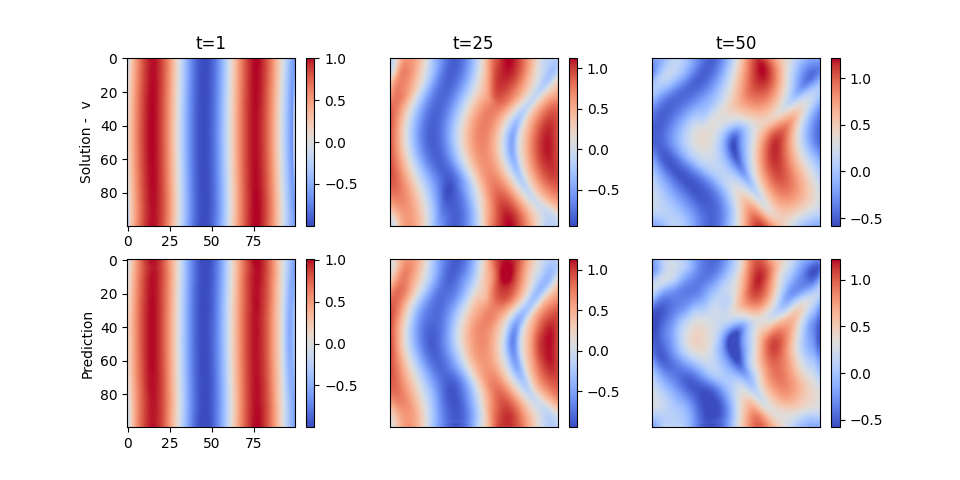}
        \caption{Vertical velocity - Autoregressive UNO}
        \label{fig:NS_incomp_v_UNO_AR}
    \end{subfigure}
    \caption{Incompressible Navier--Stokes: Model prediction within test distribution for Autoregressive UNO}
    \label{fig:NS_incomp_UNO_AR}
\end{figure}

\begin{figure}[H]
    \centering
    \begin{subfigure}{0.48\textwidth}
        \centering
        \includegraphics[width=0.9\linewidth]{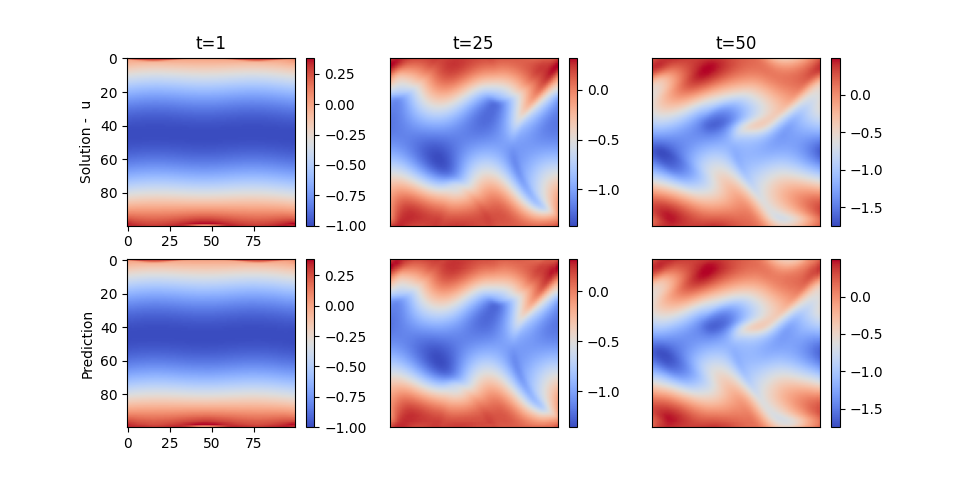}
        \caption{Horizontal velocity - NODE UNO}
        \label{fig:NS_incomp_u_UNO_NODE}
    \end{subfigure}
    \begin{subfigure}{0.48\textwidth}
        \centering
        \includegraphics[width=0.9\linewidth]{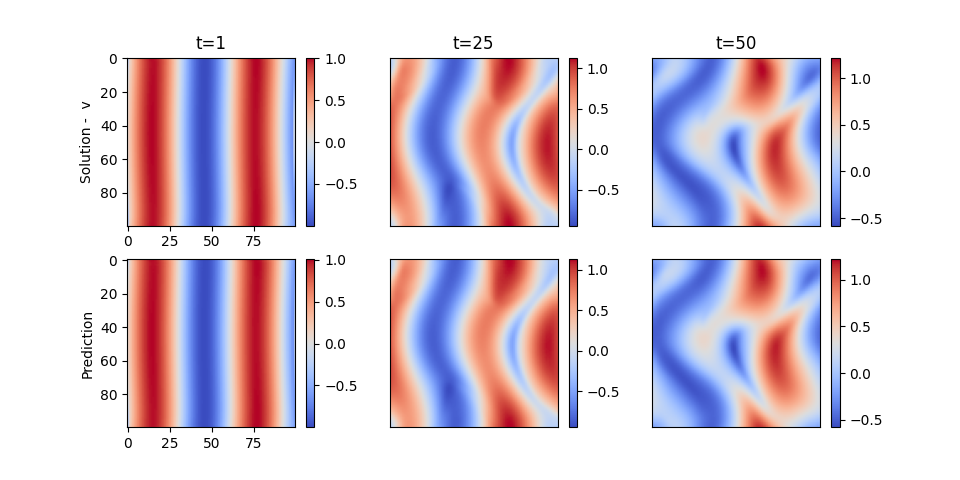}
        \caption{Vertical velocity - NODE UNO}
        \label{fig:NS_incomp_v_UNO_NODE}
    \end{subfigure}
    \caption{Incompressible Navier--Stokes: Model prediction within test distribution for Neural-ODE UNO}
    \label{fig:NS_incomp_UNO_node}
\end{figure}

\begin{figure}[H]
    \centering
    \begin{subfigure}{0.48\textwidth}
        \centering
        \includegraphics[width=0.9\linewidth]{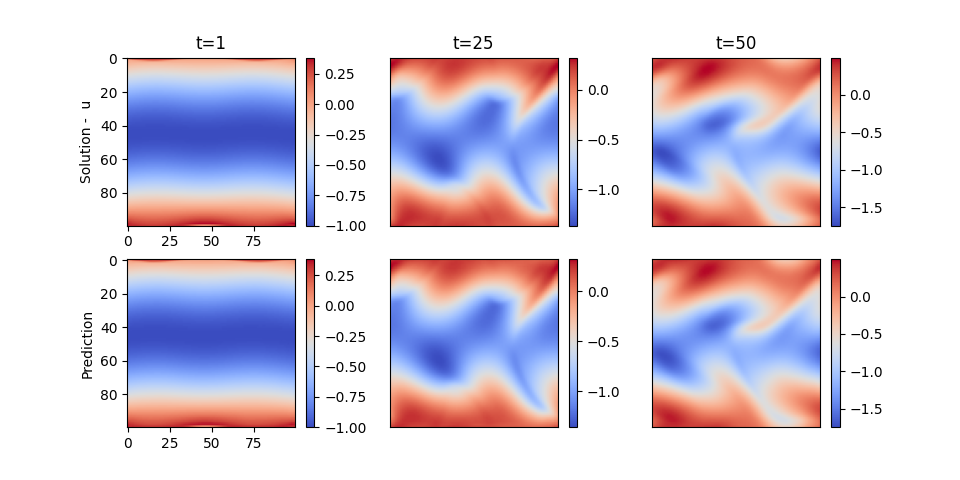}
        \caption{Horizontal velocity - OpsSplit UNO}
        \label{fig:NS_incomp_u_UNO_opsplit}
    \end{subfigure}
    \begin{subfigure}{0.48\textwidth}
        \centering
        \includegraphics[width=0.9\linewidth]{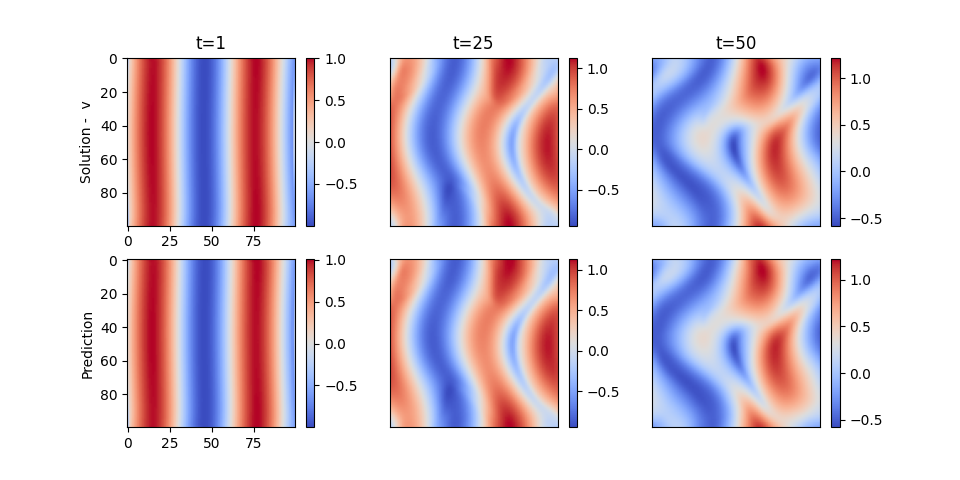}
        \caption{Vertical velocity - OpsSplit UNO}
        \label{fig:NS_incomp_v_UNO_opssplit}
    \end{subfigure}
    \caption{Incompressible Navier--Stokes: Model prediction within test distribution for OpsSplit UNO}
    \label{fig:NS_incomp_UNO_opssplit}
\end{figure}

\clearpage
\section{Compressible Navier--Stokes Equations}
\label{appendix:comp_ns}

\subsection{Physics}
Consider the two-dimensional compressible Navier--Stokes equations under adiabatic and inviscid flow:

\begin{align*}
    \frac{\partial \rho}{\partial t} &= -\nabla \cdot (\rho\mathbf{v}),  \\
    & \hfill  \nonumber \\[1ex]
    \frac{\partial \mathbf{v}}{\partial t} &= -(\mathbf{v} \cdot \nabla)\mathbf{v} - \frac{1}{\rho}\nabla P, \\
    & \hfill \nonumber \\[1ex]
    \frac{\partial P}{\partial t} &= -\mathbf{v} \cdot \nabla P - \gamma P(\nabla \cdot \mathbf{v}), \\
    & \hfill \nonumber
\end{align*}

where $\rho$ determines the density, $u$ defines the x-component of velocity, $v$ defines the y-component of velocity and $P$ determines the pressure of the fluid. The Navier--Stokes equations solve the flow of a compressible fluid given by its specific heat ratio of $\gamma$. The system defines the flow without viscosity, bounded with periodic boundary conditions within the domain, modelling two opposite moving streams under perturbation. The dataset is built by performing a Latin hypercube scan across the defined domain for the parameters $\alpha, \beta$,  which parametrise the initial velocity and pressure fields for each simulation as given in \cref{comp_ns_params_rho,comp_ns_params_vx,comp_ns_params_vy,comp_ns_params_P}. We generate 500 simulation points, each with its initial condition and use them for training. The solver is built using a finite-volume method to solve for the Kelvin-Helmholtz instability as outlined in \href{https://github.com/pmocz/finitevolume-python}{Philip Mocz's code}.

\begin{align}
\rho &= \begin{cases} 
2 & \text{if } |Y - 0.5| < 0.25 \\
1 & \text{otherwise}
\end{cases} \label{comp_ns_params_rho} \\
v_x &= \begin{cases} 
0.5 & \text{if } |Y - 0.5| < 0.25 \\
-0.5 & \text{otherwise} \label{comp_ns_params_vx}
\end{cases}\\
v_y &= \alpha \sin(4\pi X) \left[ \exp\left(-\frac{(Y-0.25)^2}{2\sigma^2}\right) + \exp\left(-\frac{(Y-0.75)^2}{2\sigma^2}\right) \right] \; ; \sigma = \frac{0.05}{\sqrt{2}}  \label{comp_ns_params_vy}\\
P &= \beta \label{comp_ns_params_P}
\end{align}

Each data point, as in each simulation, is generated with a different initial condition as described above. The parameters of the initial conditions are sampled from within the domain as given in \cref{table: data_generation_comp_ns_combined}. Each simulation is run up until wall-clock time reaches $2.0$ $\Delta t = 0.0005$. The spatial domain is uniformly discretised into 128 spatial units in the x and y axes. The temporal domain is subsampled to factor in every $20^{th}$ time instance, and the spatial domain is kept as is with a $128\times 128$ grid for the neural PDE. Parameterisation of the initial conditions and the gas ratio used for both training and testing can be found in \cref{table: data_generation_comp_ns_combined}. 

\begin{table}[h!]
\caption{Parameterisation of the 2D Euler Fluid equations utilised for training and OOD testing}
\label{table: data_generation_comp_ns_combined}
\vspace{0.5cm}
  \centering
  \begin{tabular}{llll}
  \hline 
  Parameter & Training & OOD Testing & Type \\ 
  \hline\\
    $\alpha$ & $[0.1, 0.5]$ & $[0.5, 1.0]$ & Continuous  \\
    $\beta$  & $[1.0, 5.0]$ & $[5.0, 10.0]$ & Continuous \\
    $\gamma$ & $\frac{5}{3}$ & $\frac{2}{3}$  & Discrete \\ 
  \hline
  \end{tabular}
\end{table}

\subsection{Model Details}
We evaluated four neural operator architectures within each deployment method: Fourier Neural Operator (FNO) \citep{Li2021fourier}, U-Net \citep{ronneberger2015unet, gupta2023towards}, Vision Transformer (ViT) \citep{dosovitskiy2021imageworth16x16words, herde2024poseidon} and  U-shaped Neural Operator (UNO) \citep{rahman2023unoushapedneuraloperators}. All models processed 4 variables, the velocity vector field and the scalar fields associated with density and pressure. For each model, the hyperparameters were chosen inspired from the literature and constructed to maximise the GPU utilisation within a single H100 GPU as given below:

\begin{table}[H]
\centering
\begin{tabular}{|l|l|l|}
\hline
\textbf{Model} & \textbf{Configuration Details - AR, NODE} & \textbf{OpsSplit} \\
\hline
FNO & \begin{tabular}[t]{@{}l@{}}
in\_channels: 4 \\
out\_channels: 4 \\
modes: 32 \\
width: 64 \\
n\_layers: 6 \\
activation\_func: GeLU
\end{tabular} & \begin{tabular}[t]{@{}l@{}}
in\_channels: 2 (conv), 3 (div) \\
out\_channels: 2 (conv), 1 (div)\\
modes: 32 \\
width: 64 \\
n\_layers: 3 \\
activation\_func: GeLU
\end{tabular} \\
\hline
U-Net & \begin{tabular}[t]{@{}l@{}}
in\_channels: 4 \\
out\_channels: 4 \\
initial\_width: 64 \\
activation\_func: Tanh
\end{tabular} & \begin{tabular}[t]{@{}l@{}}
in\_channels: 2 (conv), 3 (div) \\
out\_channels: 2 (conv), 1 (div)\\
initial\_width: 32 \\
activation\_func: Tanh
\end{tabular} \\
\hline
ViT & \begin{tabular}[t]{@{}l@{}}
patch\_size: 4 \\
embed\_dim: 512 \\
in\_channels: 4 \\
out\_channels: 4 \\
time\_channels: 1 \\
depth: 12 \\
num\_heads: 10 \\
activation\_func: GeLU
\end{tabular} & \begin{tabular}[t]{@{}l@{}}
patch\_size: 4 \\
embed\_dim: 256 \\
in\_channels: 2 (conv), 3 (div) \\
out\_channels: 2 (conv), 1 (div)\\
time\_channels: 1 \\
depth: 8 \\
num\_heads: 10 \\
activation\_func: GeLU
\end{tabular} \\
\hline
UNO & \begin{tabular}[t]{@{}l@{}}
in\_channels: 4 \\
out\_channels: 4 \\
width: 64 \\
projection\_channels: 256 \\
n\_layers: 5 \\
uno\_out\_channels: [32,64,64,64,32] \\
uno\_n\_modes: [[16,16],[8,8],[8,8],[8,8],[16,16]] \\
domain\_padding: 0.2 \\
norm: group\_norm \\
activation\_func: GeLU
\end{tabular} & \begin{tabular}[t]{@{}l@{}}
in\_channels: 2 (conv), 3 (div) \\
out\_channels: 2 (conv), 1 (div)\\
width: 32 \\
projection\_channels: 256 \\
n\_layers: 5 \\
uno\_out\_channels: [32,64,64,64,32] \\
uno\_n\_modes: [[16,16],[8,8],[8,8],[8,8],[16,16]] \\
domain\_padding: 0.2 \\
norm: group\_norm \\
activation\_func: GeLU
\end{tabular} \\
\hline
\end{tabular}
\label{tab:model_config_comp_NS}
\caption{Model configuration details of neural operators used for modelling the compressible Navier--Stokes equations. Both AR and   NODE-based models have the same configuration, whereas the two models within OpsSplit have slightly different setups to ensure the same parameterisation levels across all methods. }
\end{table}

\subsection{Performance Plots}
In this section, we showcase figures to provide a qualitative comparison of model predictions across different deployment methods and neural operator architectures. Each visualisation displays the ground truth (top row) and model predictions (bottom row) at three representative time instances: t=1 (early-stage dynamics within training temporal resolution), t=25 (mid-simulation behaviour testing model stability), and t=50 (long-term temporal extrapolation). Velocity and scalar field visualisations are represented in the physical space within the Cartesian domain. These visualisations enable assessment of spatial accuracy (how well the model captures field patterns and structures), temporal stability (whether predictions maintain physical consistency over time), error accumulation (how prediction errors grow during autoregressive rollout or temporal extrapolation), and method comparison (relative performance of Autoregressive, Neural ODE, and OpsSplit approaches across different architectures). For quantitative metrics corresponding to these visualisations, refer to \cref{tab:comp_ns} in the main text.

\subsubsection{FNO}

\begin{figure}[H]
    \centering
    \begin{subfigure}{0.48\textwidth}
        \centering
        \includegraphics[width=0.9\linewidth]{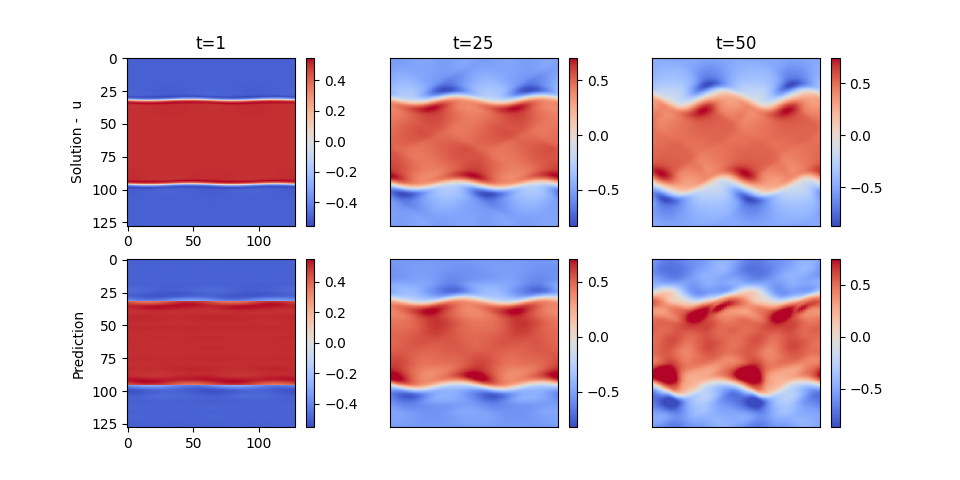}
        \caption{Horizontal velocity - Autoregressive FNO}
        \label{fig:NS_comp_u_FNO_AR}
    \end{subfigure}
    \begin{subfigure}{0.48\textwidth}
        \centering
        \includegraphics[width=0.9\linewidth]{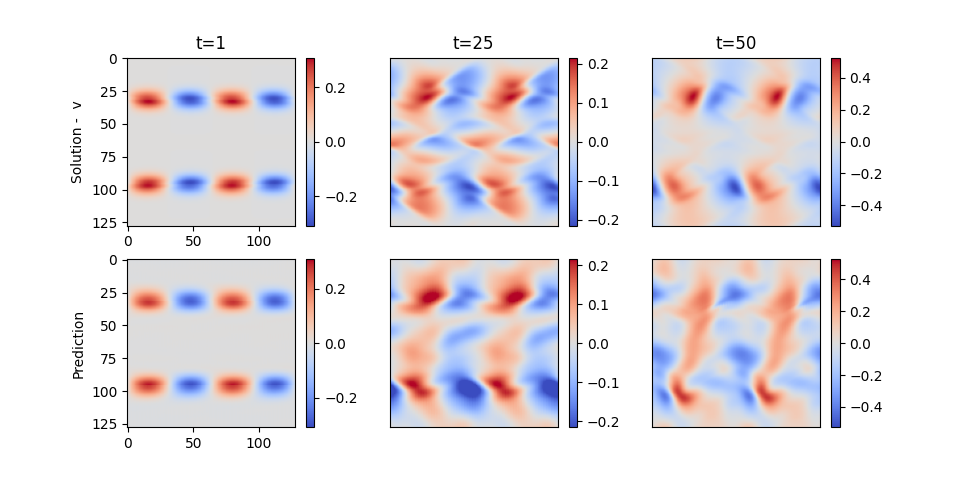}
        \caption{Vertical velocity - Autoregressive FNO}
        \label{fig:NS_comp_v_FNO_AR}
    \end{subfigure}
    \begin{subfigure}{0.48\textwidth}
        \centering
        \includegraphics[width=0.9\linewidth]{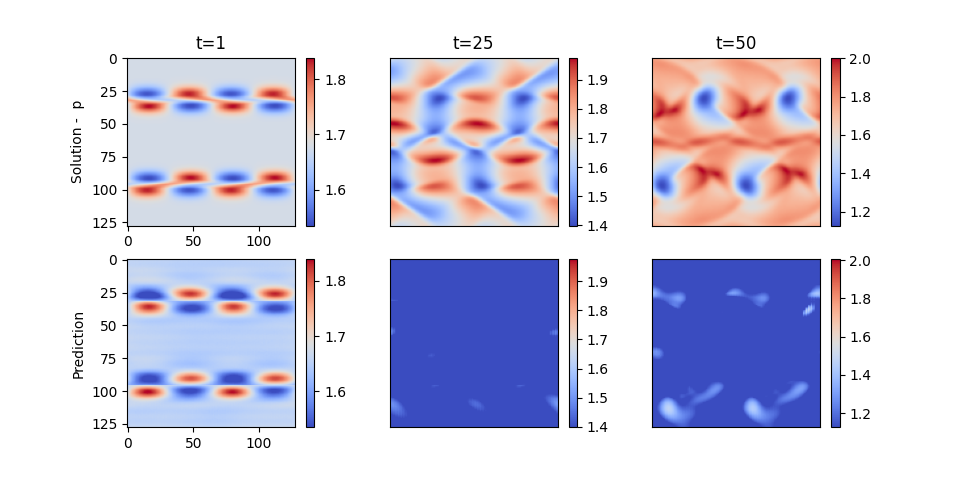}
        \caption{Pressure - Autoregressive FNO}
        \label{fig:NS_comp_p_FNO_AR}
    \end{subfigure}
    \begin{subfigure}{0.48\textwidth}
        \centering
        \includegraphics[width=0.9\linewidth]{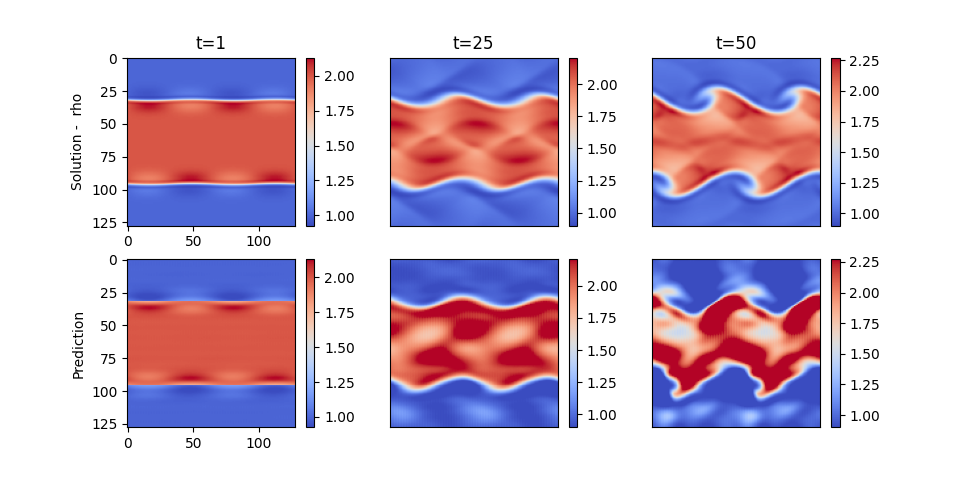}
        \caption{Density - Autoregressive FNO}
        \label{fig:NS_comp_rho_FNO_AR}
    \end{subfigure}
    \caption{Compressible Navier--Stokes: Model prediction within test distribution for Autoregressive FNO}
    \label{fig:NS_comp_FNO_AR}
\end{figure}

\begin{figure}[H]
    \centering
    \begin{subfigure}{0.48\textwidth}
        \centering
        \includegraphics[width=0.9\linewidth]{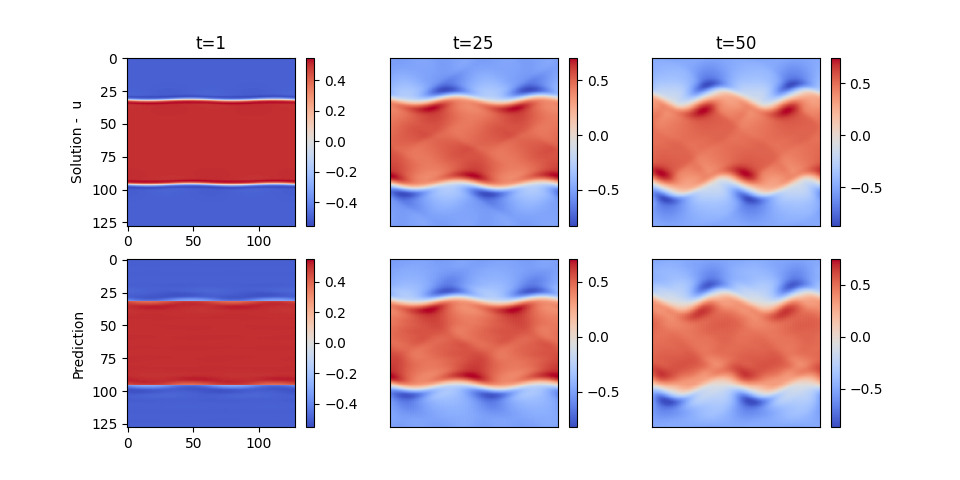}
        \caption{Horizontal velocity - NODE FNO}
        \label{fig:NS_comp_u_FNO_NODE}
    \end{subfigure}
    \begin{subfigure}{0.48\textwidth}
        \centering
        \includegraphics[width=0.9\linewidth]{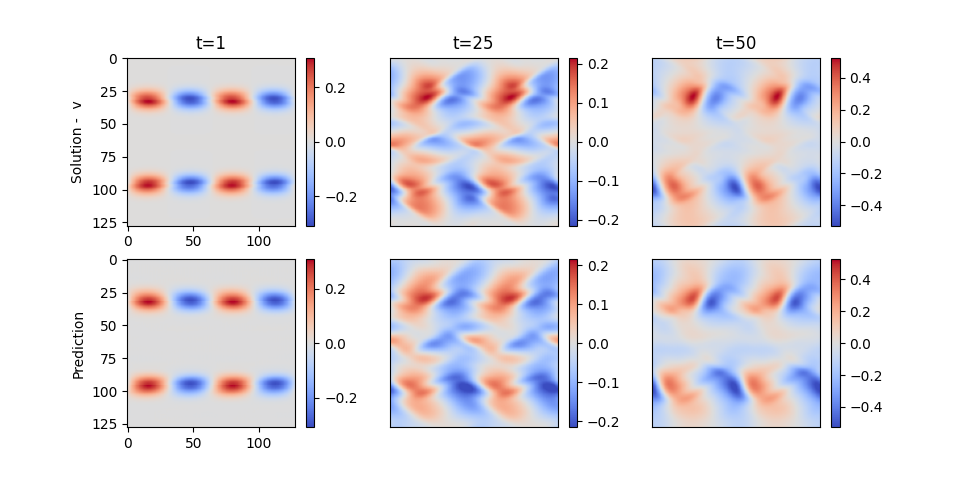}
        \caption{Vertical velocity - NODE FNO}
        \label{fig:NS_comp_v_FNO_NODE}
    \end{subfigure}
    \begin{subfigure}{0.48\textwidth}
        \centering
        \includegraphics[width=0.9\linewidth]{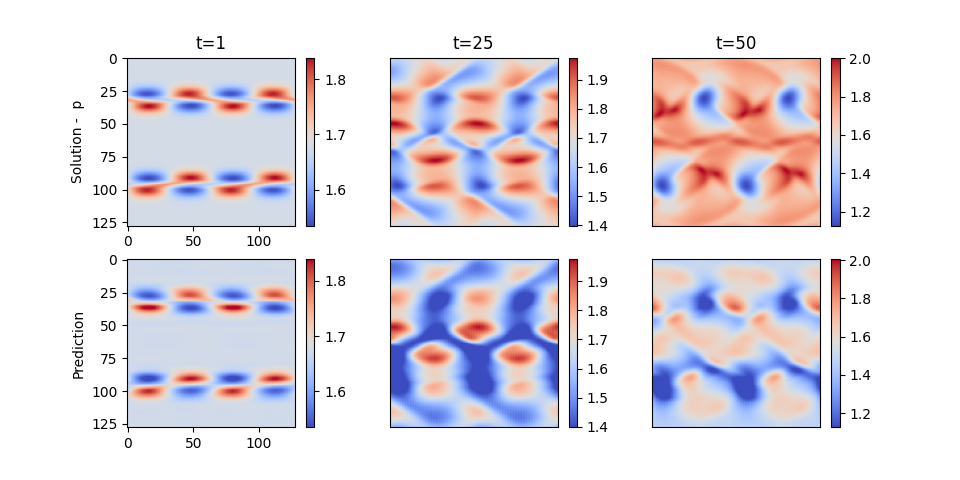}
        \caption{Pressure - NODE FNO}
        \label{fig:NS_comp_p_FNO_NODE}
    \end{subfigure}
    \begin{subfigure}{0.48\textwidth}
        \centering
        \includegraphics[width=0.9\linewidth]{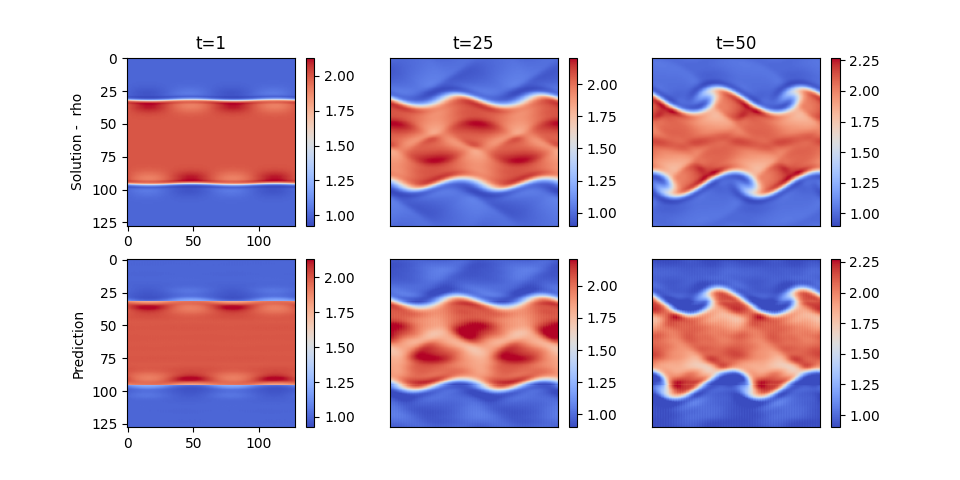}
        \caption{Density - NODE FNO}
        \label{fig:NS_comp_rho_FNO_NODE}
    \end{subfigure}
    \caption{Compressible Navier--Stokes: Model prediction within test distribution for Neural-ODE FNO}
    \label{fig:NS_comp_FNO_node}
\end{figure}

\begin{figure}[H]
    \centering
    \begin{subfigure}{0.48\textwidth}
        \centering
        \includegraphics[width=0.9\linewidth]{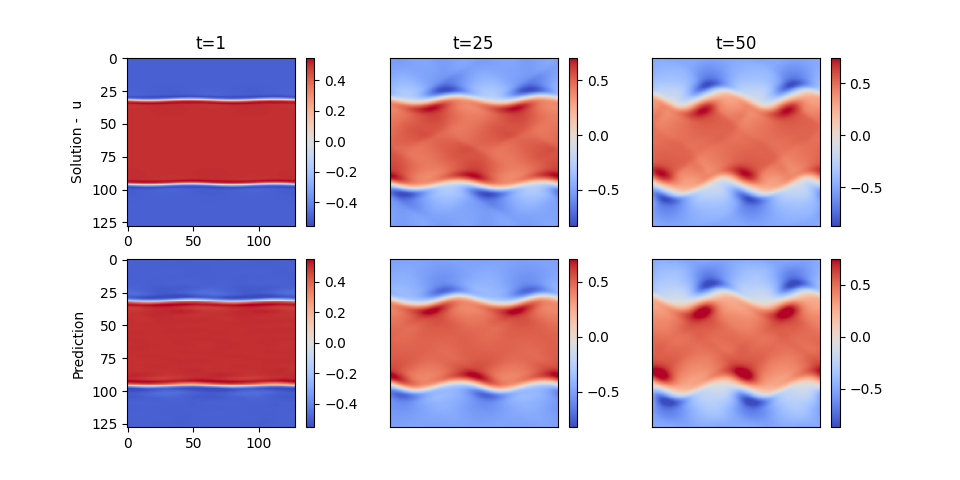}
        \caption{Horizontal velocity - OpsSplit FNO}
        \label{fig:NS_comp_u_FNO_opsplit}
    \end{subfigure}
    \begin{subfigure}{0.48\textwidth}
        \centering
        \includegraphics[width=0.9\linewidth]{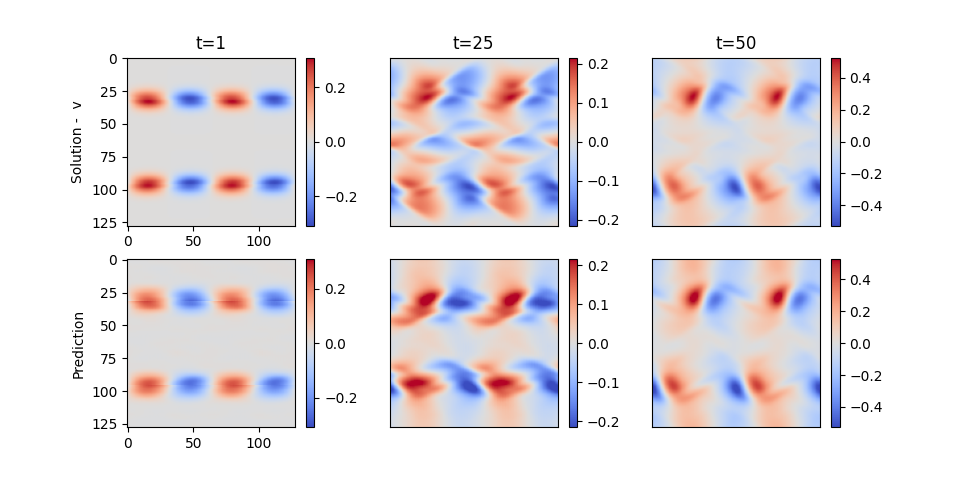}
        \caption{Vertical velocity - OpsSplit FNO}
        \label{fig:NS_comp_v_FNO_opssplit}
    \end{subfigure}
    \begin{subfigure}{0.48\textwidth}
        \centering
        \includegraphics[width=0.9\linewidth]{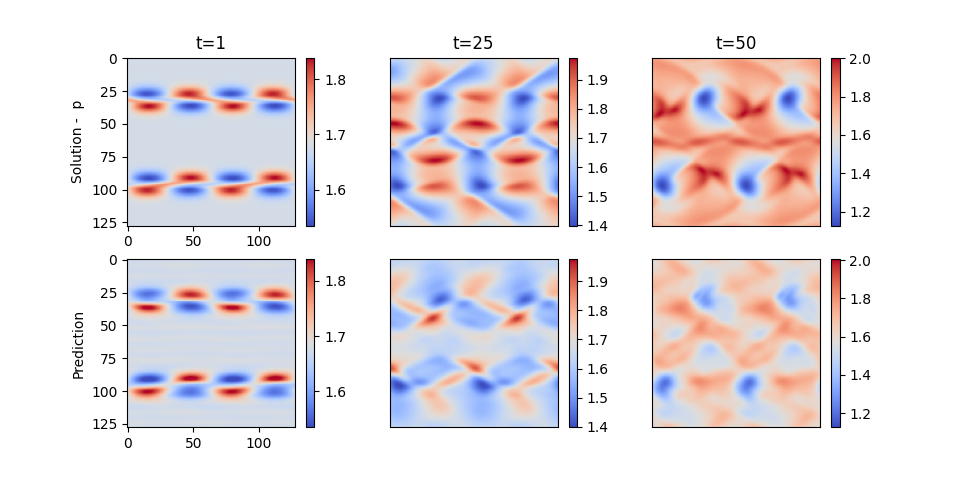}
        \caption{Pressure - OpsSplit FNO}
        \label{fig:NS_comp_p_FNO_opsplit}
    \end{subfigure}
    \begin{subfigure}{0.48\textwidth}
        \centering
        \includegraphics[width=0.9\linewidth]{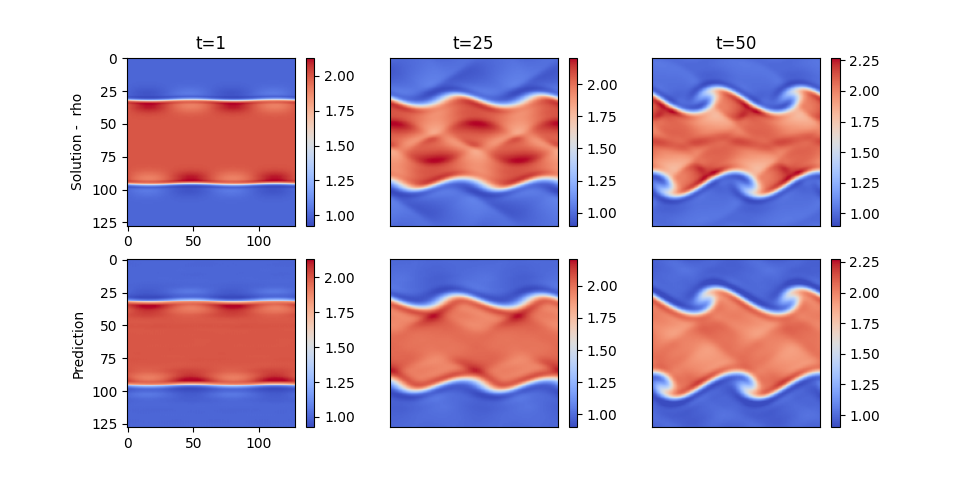}
        \caption{Density - OpsSplit FNO}
        \label{fig:NS_comp_rho_FNO_opssplit}
    \end{subfigure}
    \caption{Compressible Navier--Stokes: Model prediction within test distribution for OpsSplit FNO}
    \label{fig:NS_comp_FNO_opssplit}
\end{figure}

\subsubsection{U-Net}

\begin{figure}[H]
    \centering
    \begin{subfigure}{0.48\textwidth}
        \centering
        \includegraphics[width=0.9\linewidth]{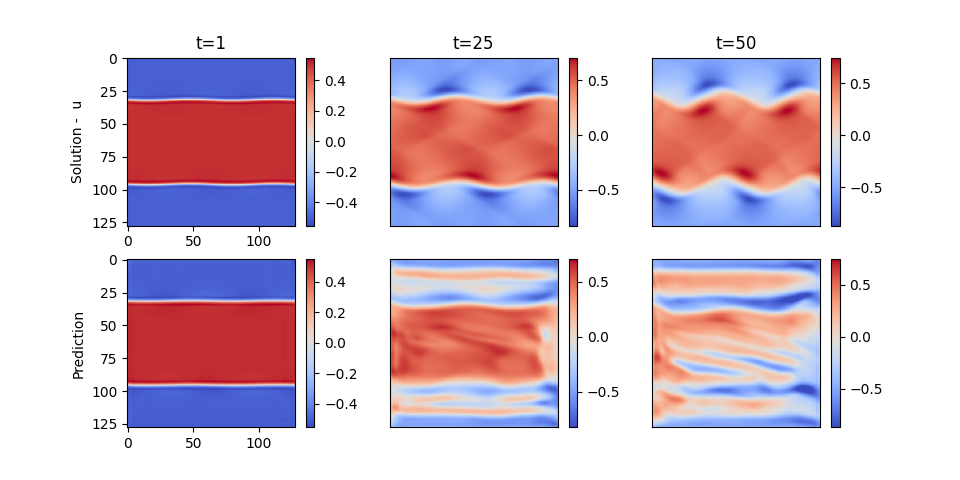}
        \caption{Horizontal velocity - Autoregressive U-Net}
        \label{fig:NS_comp_u_UNet_AR}
    \end{subfigure}
    \begin{subfigure}{0.48\textwidth}
        \centering
        \includegraphics[width=0.9\linewidth]{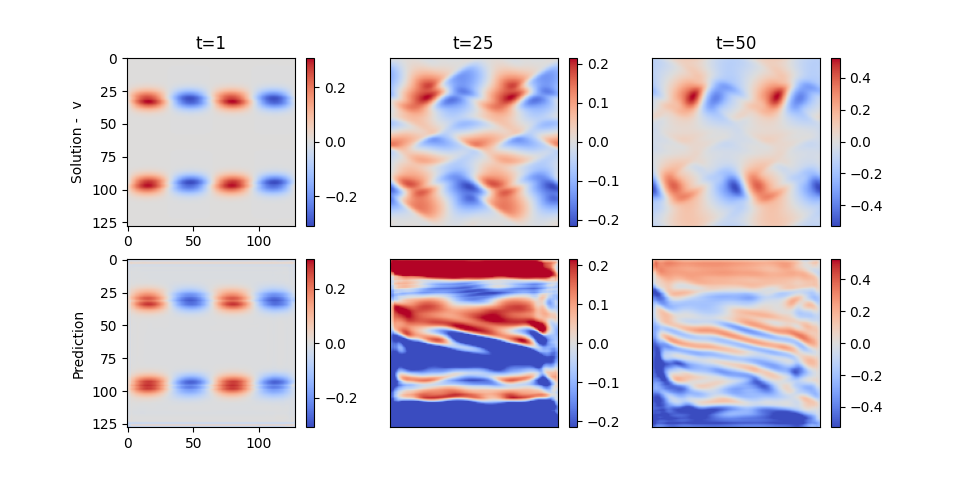}
        \caption{Vertical velocity - Autoregressive U-Net}
        \label{fig:NS_comp_v_UNet_AR}
    \end{subfigure}
    \begin{subfigure}{0.48\textwidth}
        \centering
        \includegraphics[width=0.9\linewidth]{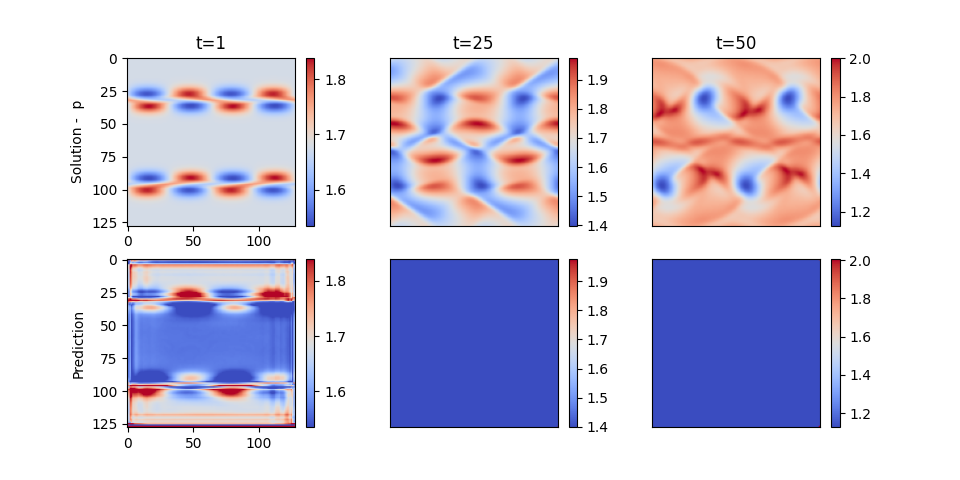}
        \caption{Pressure - Autoregressive U-Net}
        \label{fig:NS_comp_p_UNet_AR}
    \end{subfigure}
    \begin{subfigure}{0.48\textwidth}
        \centering
        \includegraphics[width=0.9\linewidth]{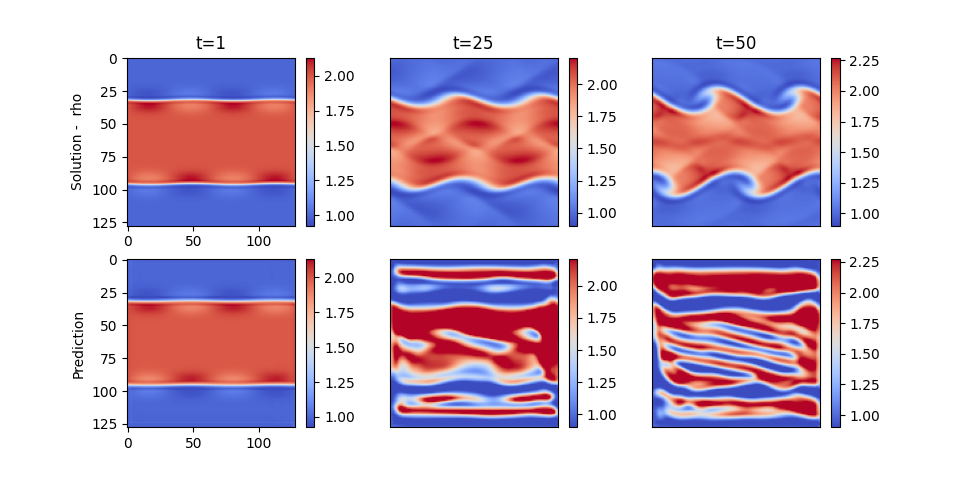}
        \caption{Density - Autoregressive U-Net}
        \label{fig:NS_comp_rho_UNet_AR}
    \end{subfigure}
    \caption{Compressible Navier--Stokes: Model prediction within test distribution for Autoregressive U-Net}
    \label{fig:NS_comp_UNet_AR}
\end{figure}

\begin{figure}[H]
    \centering
    \begin{subfigure}{0.48\textwidth}
        \centering
        \includegraphics[width=0.9\linewidth]{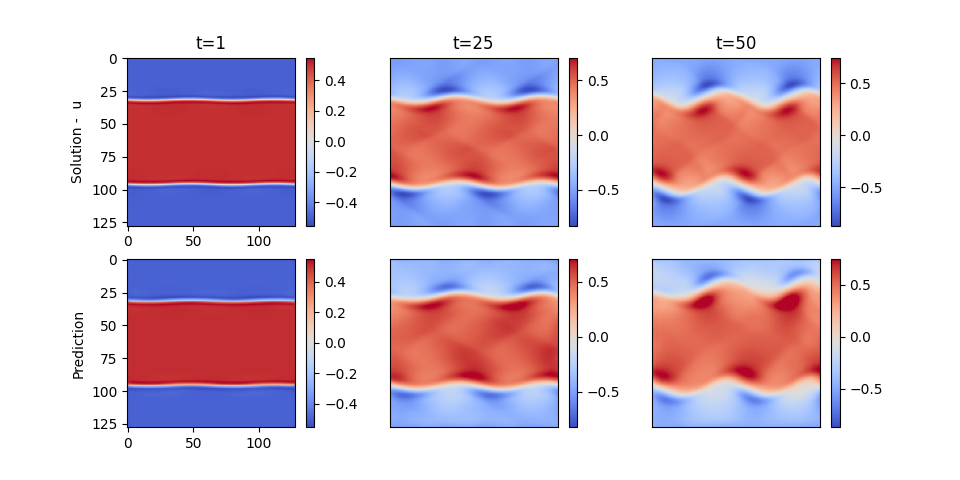}
        \caption{Horizontal velocity - NODE U-Net}
        \label{fig:NS_comp_u_UNet_NODE}
    \end{subfigure}
    \begin{subfigure}{0.48\textwidth}
        \centering
        \includegraphics[width=0.9\linewidth]{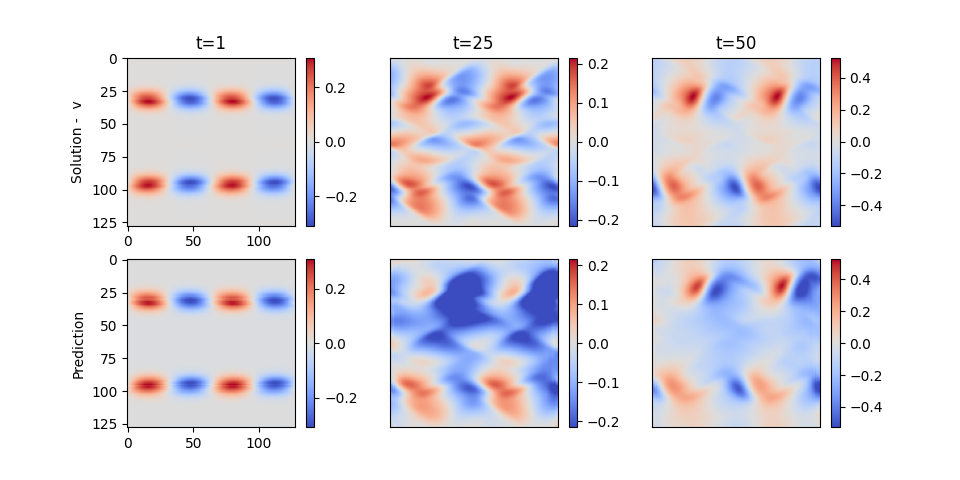}
        \caption{Vertical velocity - NODE U-Net}
        \label{fig:NS_comp_v_UNet_NODE}
    \end{subfigure}
    \begin{subfigure}{0.48\textwidth}
        \centering
        \includegraphics[width=0.9\linewidth]{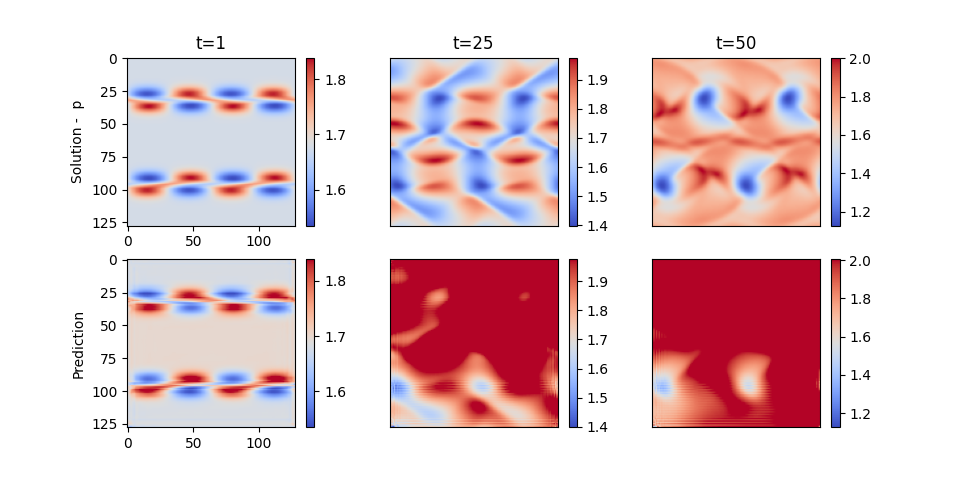}
        \caption{Pressure - NODE U-Net}
        \label{fig:NS_comp_p_UNet_NODE}
    \end{subfigure}
    \begin{subfigure}{0.48\textwidth}
        \centering
        \includegraphics[width=0.9\linewidth]{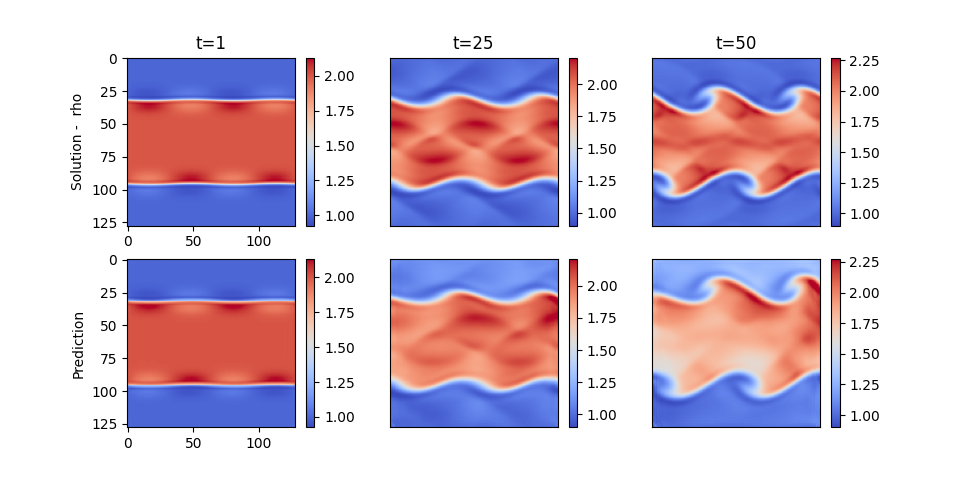}
        \caption{Density - NODE U-Net}
        \label{fig:NS_comp_rho_UNet_NODE}
    \end{subfigure}
    \caption{Compressible Navier--Stokes: Model prediction within test distribution for Neural-ODE U-Net}
    \label{fig:NS_comp_UNet_node}
\end{figure}

\begin{figure}[H]
    \centering
    \begin{subfigure}{0.48\textwidth}
        \centering
        \includegraphics[width=0.9\linewidth]{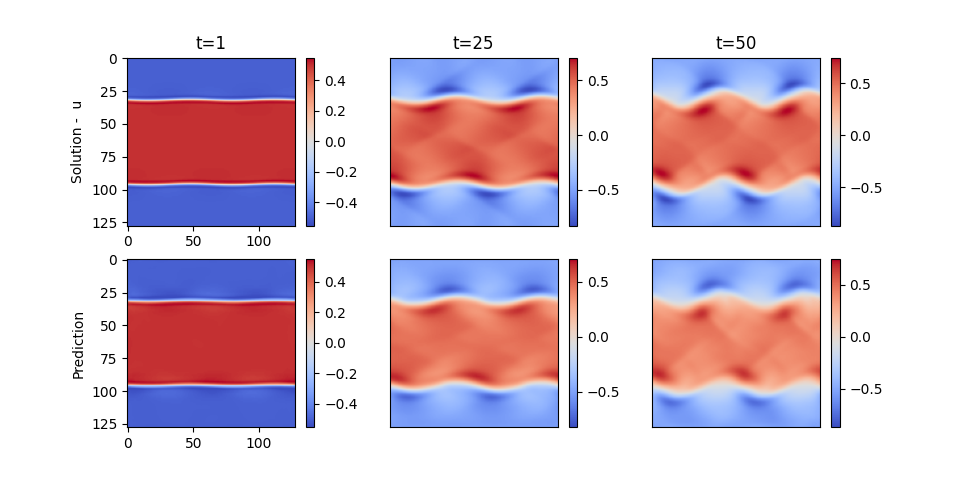}
        \caption{Horizontal velocity - OpsSplit U-Net}
        \label{fig:NS_comp_u_UNet_opsplit}
    \end{subfigure}
    \begin{subfigure}{0.48\textwidth}
        \centering
        \includegraphics[width=0.9\linewidth]{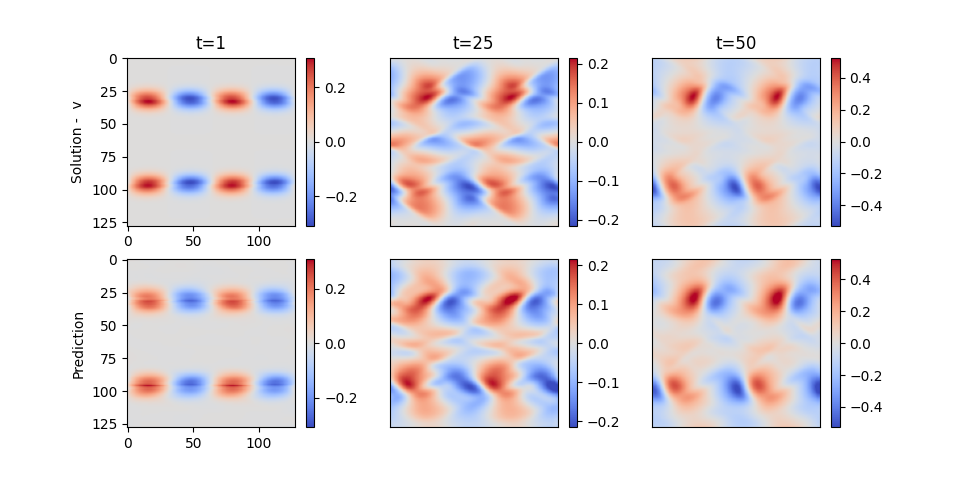}
        \caption{Vertical velocity - OpsSplit U-Net}
        \label{fig:NS_comp_v_UNet_opssplit}
    \end{subfigure}
    \begin{subfigure}{0.48\textwidth}
        \centering
        \includegraphics[width=0.9\linewidth]{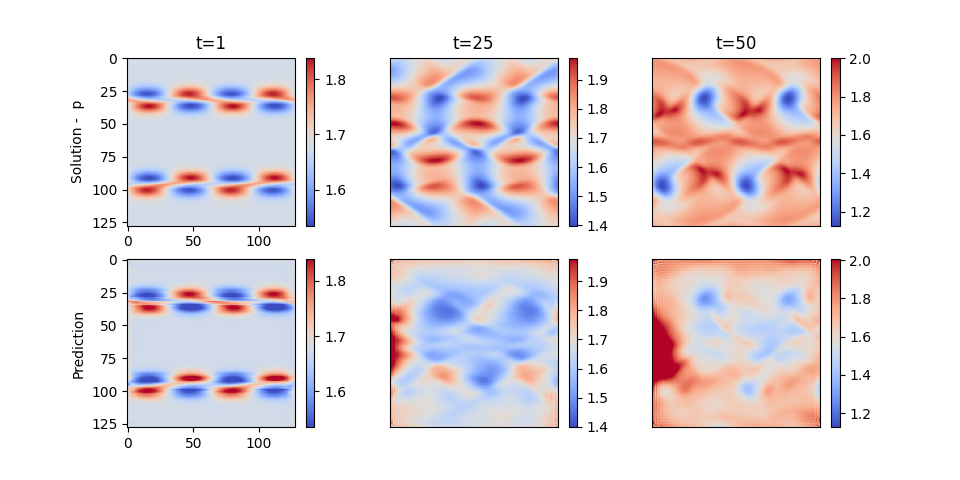}
        \caption{Pressure - OpsSplit U-Net}
        \label{fig:NS_comp_p_UNet_opsplit}
    \end{subfigure}
    \begin{subfigure}{0.48\textwidth}
        \centering
        \includegraphics[width=0.9\linewidth]{Images/rho_terminal-rehab.png}
        \caption{Density - OpsSplit U-Net}
        \label{fig:NS_comp_rho_UNet_opssplit}
    \end{subfigure}
    \caption{Compressible Navier--Stokes: Model prediction within test distribution for OpsSplit U-Net}
    \label{fig:NS_comp_UNet_opssplit}
\end{figure}

\subsubsection{ViT}

\begin{figure}[H]
    \centering
    \begin{subfigure}{0.48\textwidth}
        \centering
        \includegraphics[width=0.9\linewidth]{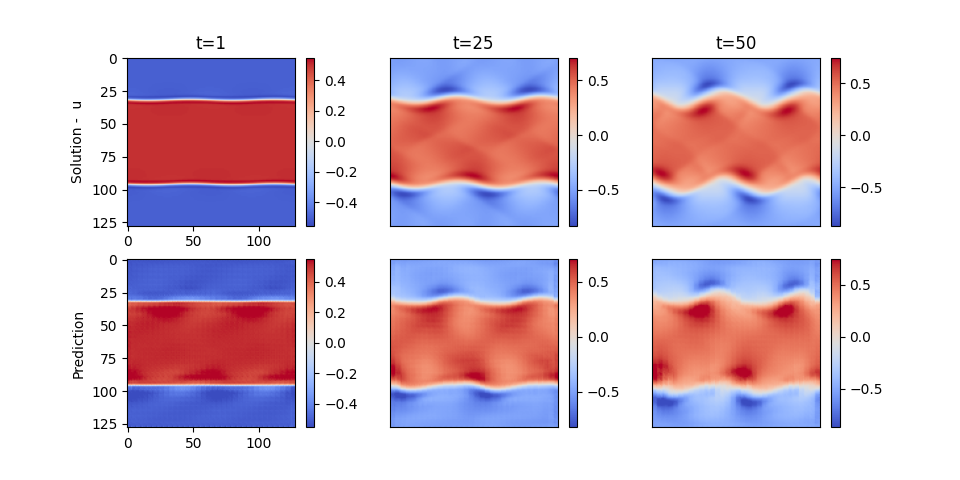}
        \caption{Horizontal velocity - Autoregressive ViT}
        \label{fig:NS_comp_u_ViT_AR}
    \end{subfigure}
    \begin{subfigure}{0.48\textwidth}
        \centering
        \includegraphics[width=0.9\linewidth]{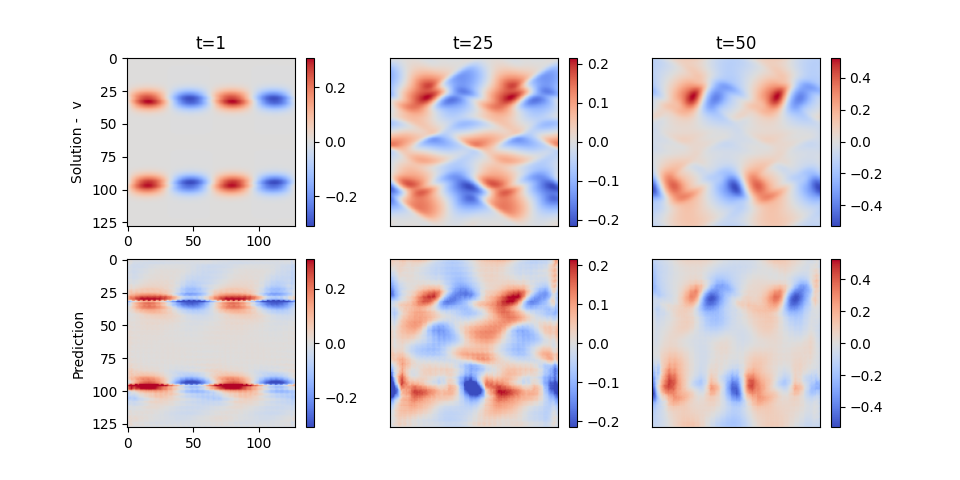}
        \caption{Vertical velocity - Autoregressive ViT}
        \label{fig:NS_comp_v_ViT_AR}
    \end{subfigure}
    \begin{subfigure}{0.48\textwidth}
        \centering
        \includegraphics[width=0.9\linewidth]{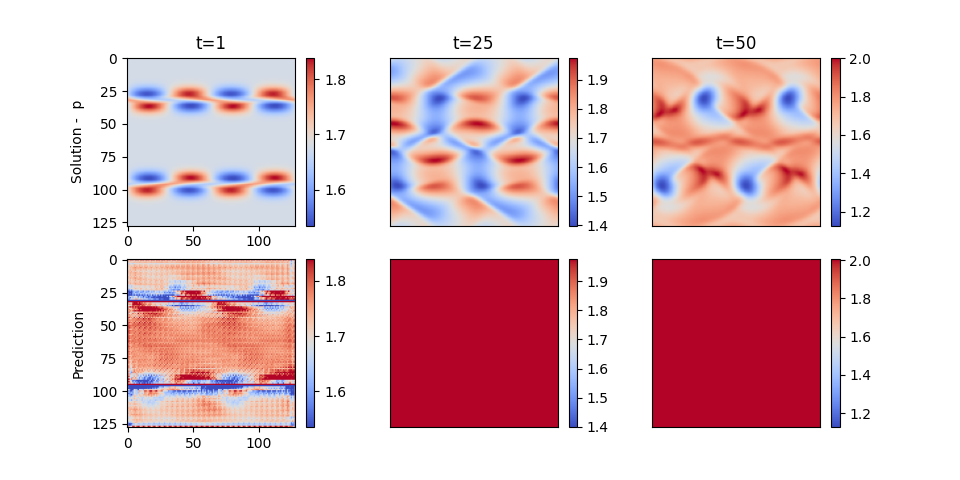}
        \caption{Pressure - Autoregressive ViT}
        \label{fig:NS_comp_p_ViT_AR}
    \end{subfigure}
    \begin{subfigure}{0.48\textwidth}
        \centering
        \includegraphics[width=0.9\linewidth]{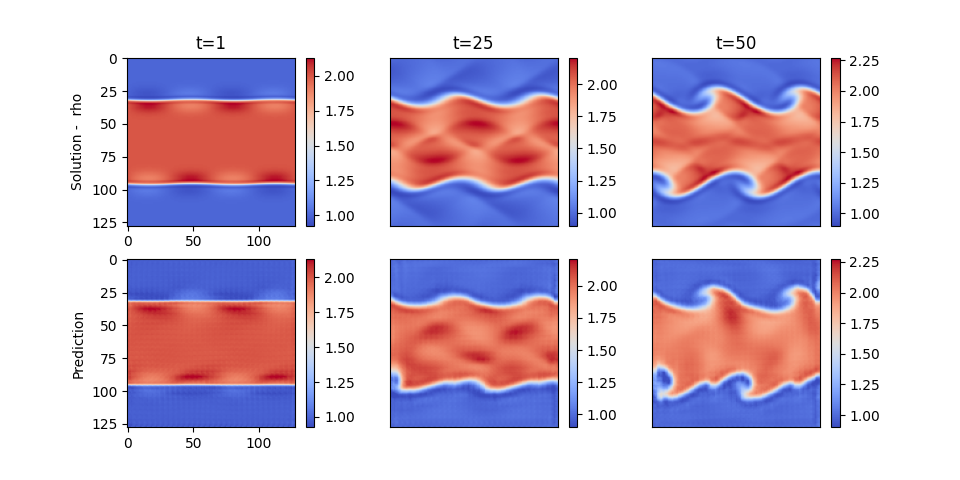}
        \caption{Density - Autoregressive ViT}
        \label{fig:NS_comp_rho_ViT_AR}
    \end{subfigure}
    \caption{Compressible Navier--Stokes: Model prediction within test distribution for Autoregressive ViT}
    \label{fig:NS_comp_ViT_AR}
\end{figure}

\begin{figure}[H]
    \centering
    \begin{subfigure}{0.48\textwidth}
        \centering
        \includegraphics[width=0.9\linewidth]{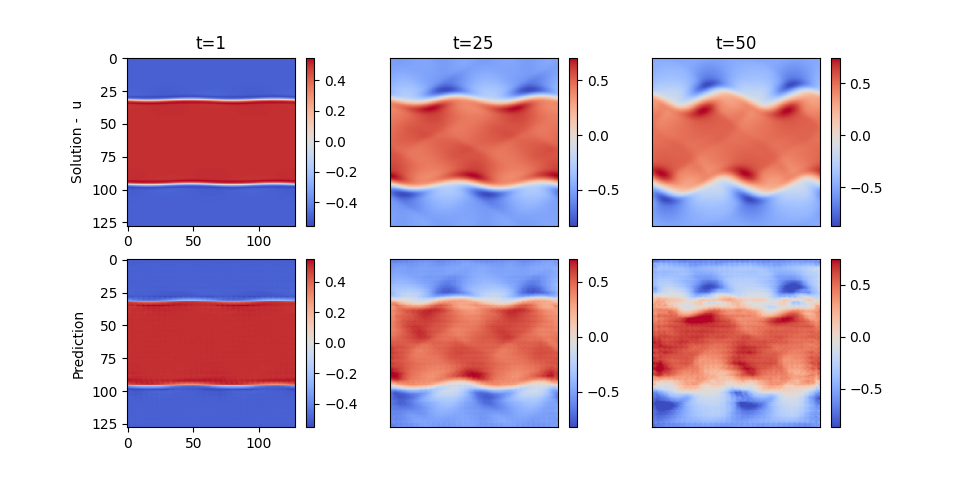}
        \caption{Horizontal velocity - NODE ViT}
        \label{fig:NS_comp_u_ViT_NODE}
    \end{subfigure}
    \begin{subfigure}{0.48\textwidth}
        \centering
        \includegraphics[width=0.9\linewidth]{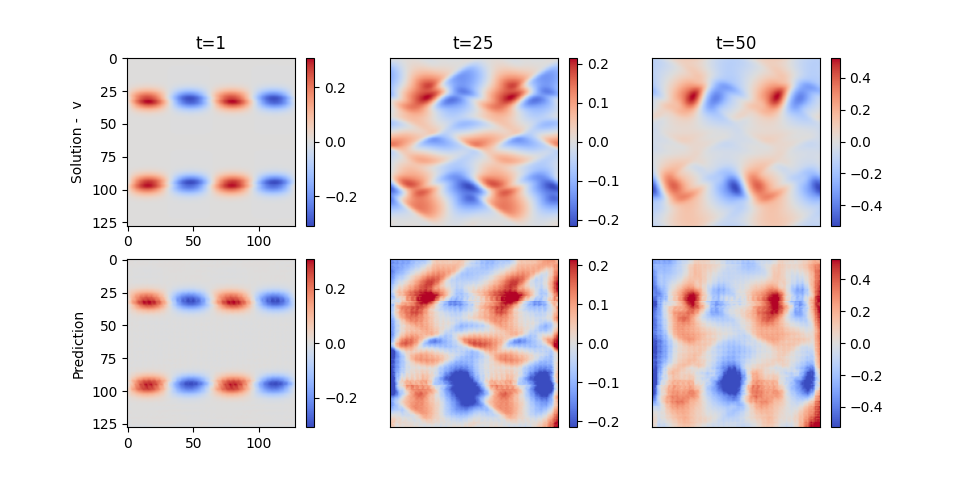}
        \caption{Vertical velocity - NODE ViT}
        \label{fig:NS_comp_v_ViT_NODE}
    \end{subfigure}
    \begin{subfigure}{0.48\textwidth}
        \centering
        \includegraphics[width=0.9\linewidth]{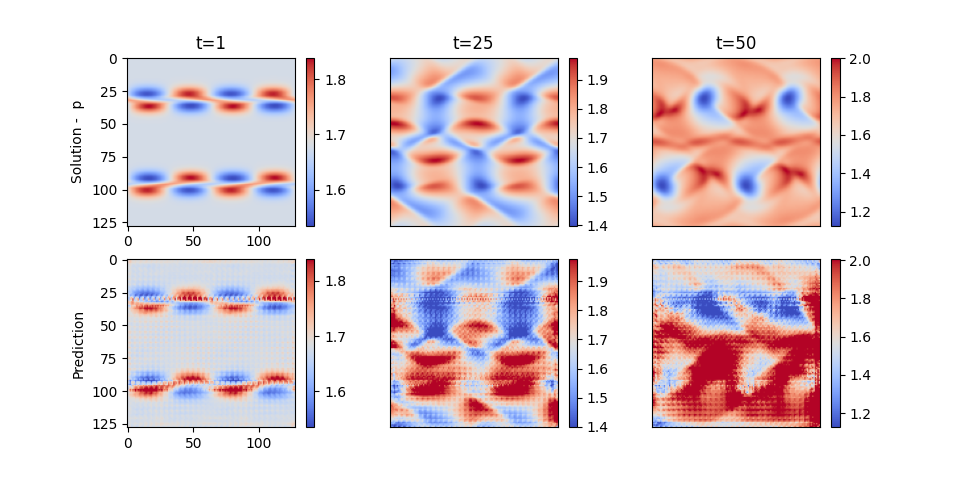}
        \caption{Pressure - NODE ViT}
        \label{fig:NS_comp_p_ViT_NODE}
    \end{subfigure}
    \begin{subfigure}{0.48\textwidth}
        \centering
        \includegraphics[width=0.9\linewidth]{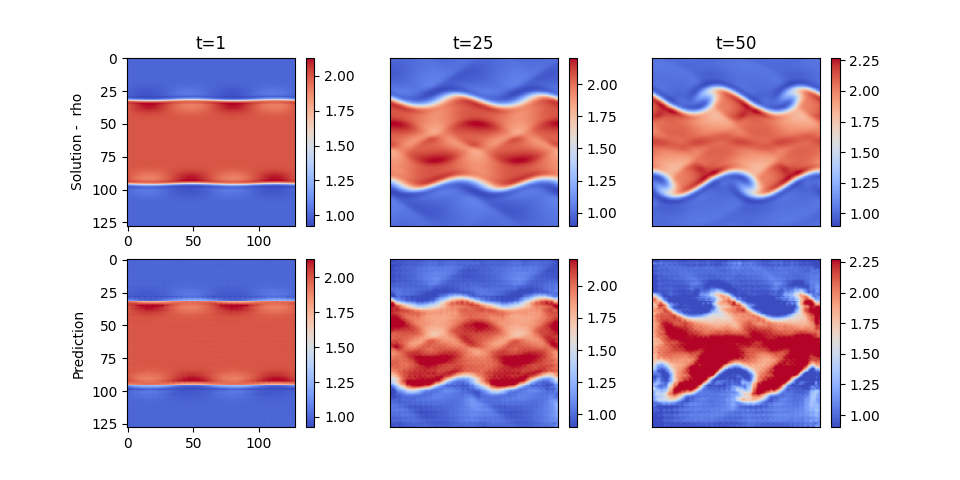}
        \caption{Density - NODE ViT}
        \label{fig:NS_comp_rho_ViT_NODE}
    \end{subfigure}
    \caption{Compressible Navier--Stokes: Model prediction within test distribution for Neural-ODE ViT}
    \label{fig:NS_comp_ViT_node}
\end{figure}

\begin{figure}[H]
    \centering
    \begin{subfigure}{0.48\textwidth}
        \centering
        \includegraphics[width=0.9\linewidth]{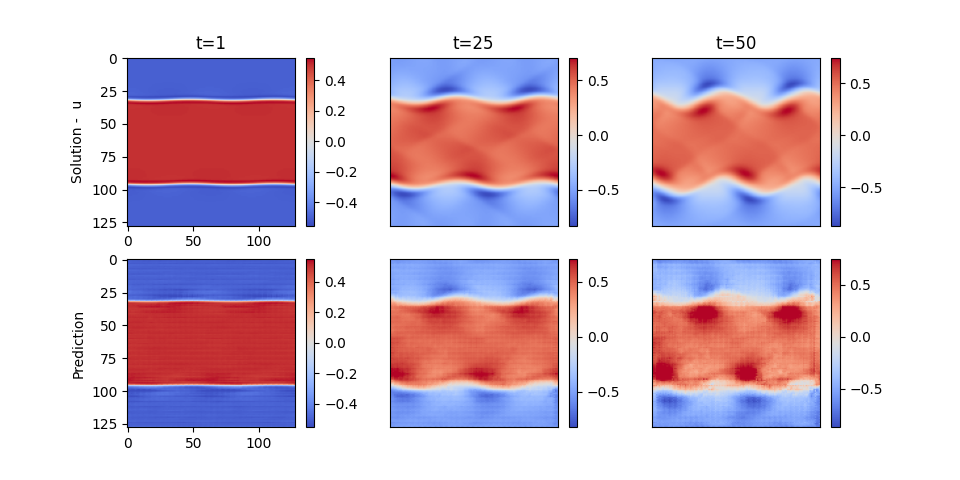}
        \caption{Horizontal velocity - OpsSplit ViT}
        \label{fig:NS_comp_u_ViT_opsplit}
    \end{subfigure}
    \begin{subfigure}{0.48\textwidth}
        \centering
        \includegraphics[width=0.9\linewidth]{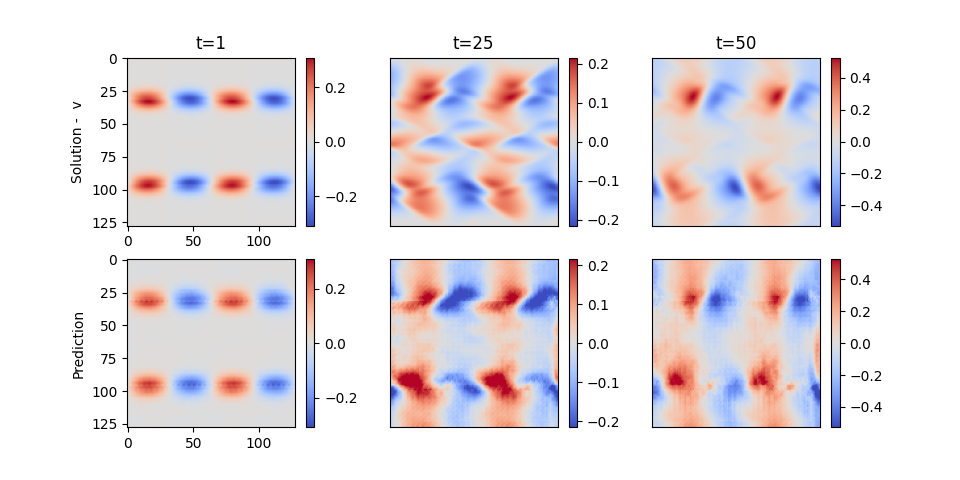}
        \caption{Vertical velocity - OpsSplit ViT}
        \label{fig:NS_comp_v_ViT_opssplit}
    \end{subfigure}
    \begin{subfigure}{0.48\textwidth}
        \centering
        \includegraphics[width=0.9\linewidth]{Images/p_terminal-rehab.png}
        \caption{Pressure - OpsSplit ViT}
        \label{fig:NS_comp_p_ViT_opsplit}
    \end{subfigure}
    \begin{subfigure}{0.48\textwidth}
        \centering
        \includegraphics[width=0.9\linewidth]{Images/rho_terminal-rehab.png}
        \caption{Density - OpsSplit ViT}
        \label{fig:NS_comp_rho_ViT_opssplit}
    \end{subfigure}
    \caption{Compressible Navier--Stokes: Model prediction within test distribution for OpsSplit ViT}
    \label{fig:NS_comp_ViT_opssplit}
\end{figure}

\subsubsection{UNO}

\begin{figure}[H]
    \centering
    \begin{subfigure}{0.48\textwidth}
        \centering
        \includegraphics[width=0.9\linewidth]{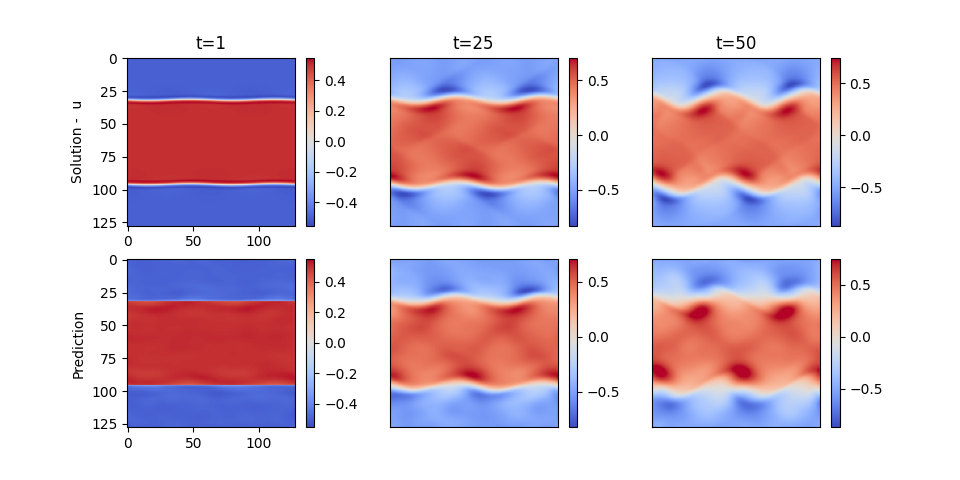}
        \caption{Horizontal velocity - Autoregressive UNO}
        \label{fig:NS_comp_u_UNO_AR}
    \end{subfigure}
    \begin{subfigure}{0.48\textwidth}
        \centering
        \includegraphics[width=0.9\linewidth]{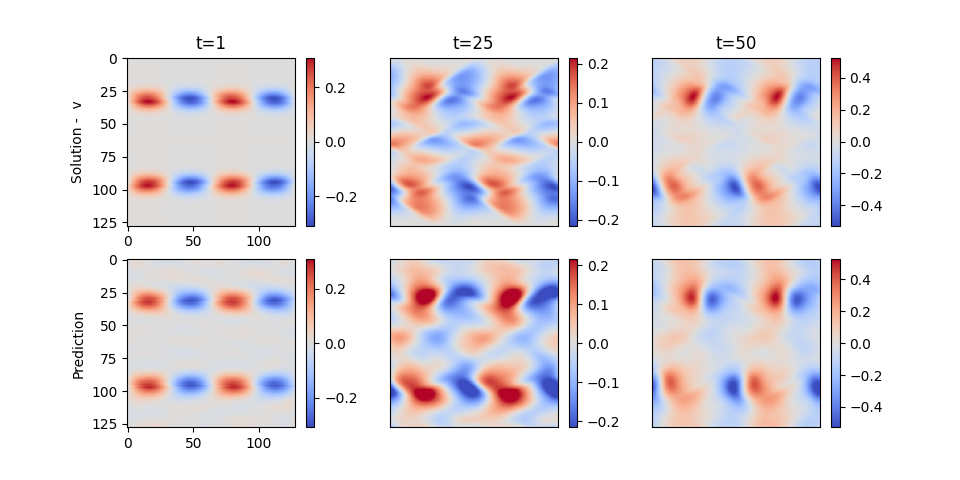}
        \caption{Vertical velocity - Autoregressive UNO}
        \label{fig:NS_comp_v_UNO_AR}
    \end{subfigure}
    \begin{subfigure}{0.48\textwidth}
        \centering
        \includegraphics[width=0.9\linewidth]{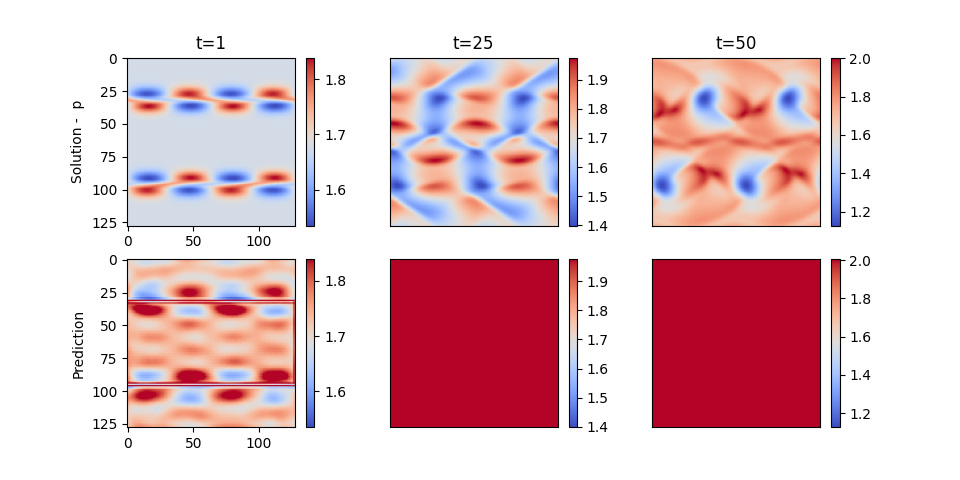}
        \caption{Pressure - Autoregressive UNO}
        \label{fig:NS_comp_p_UNO_AR}
    \end{subfigure}
    \begin{subfigure}{0.48\textwidth}
        \centering
        \includegraphics[width=0.9\linewidth]{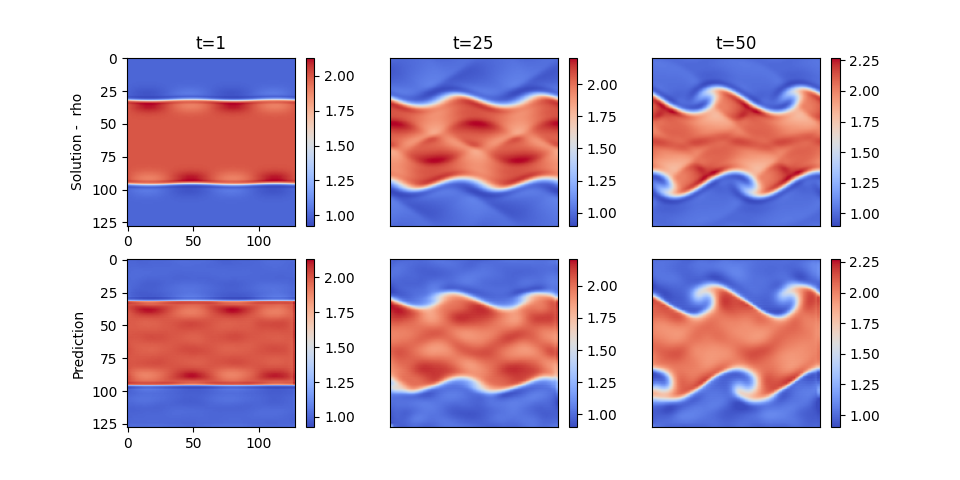}
        \caption{Density - Autoregressive UNO}
        \label{fig:NS_comp_rho_UNO_AR}
    \end{subfigure}
    \caption{Compressible Navier--Stokes: Model prediction within test distribution for Autoregressive UNO}
    \label{fig:NS_comp_UNO_AR}
\end{figure}

\begin{figure}[H]
    \centering
    \begin{subfigure}{0.48\textwidth}
        \centering
        \includegraphics[width=0.9\linewidth]{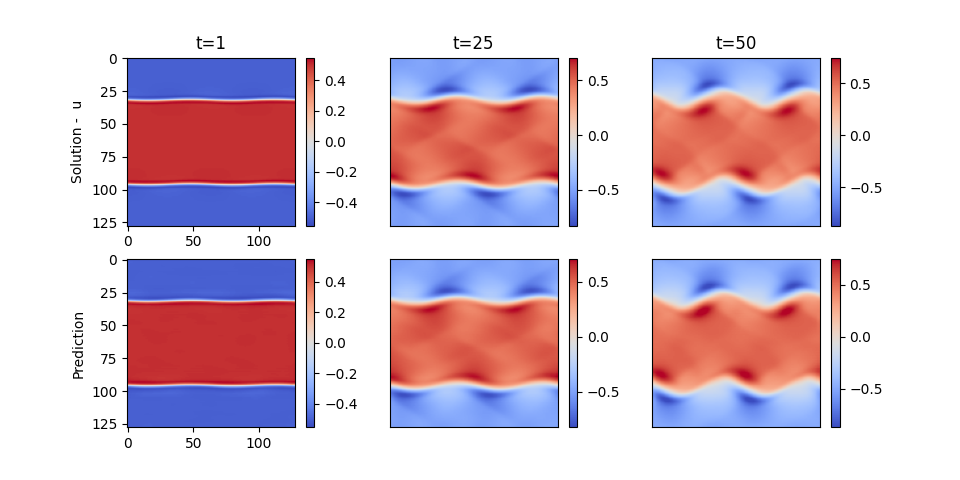}
        \caption{Horizontal velocity - NODE UNO}
        \label{fig:NS_comp_u_UNO_NODE}
    \end{subfigure}
    \begin{subfigure}{0.48\textwidth}
        \centering
        \includegraphics[width=0.9\linewidth]{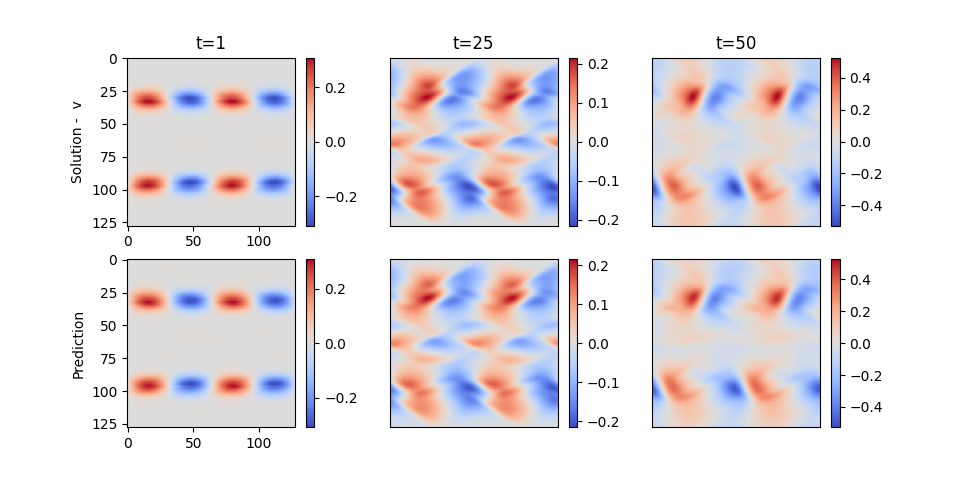}
        \caption{Vertical velocity - NODE UNO}
        \label{fig:NS_comp_v_UNO_NODE}
    \end{subfigure}
    \begin{subfigure}{0.48\textwidth}
        \centering
        \includegraphics[width=0.9\linewidth]{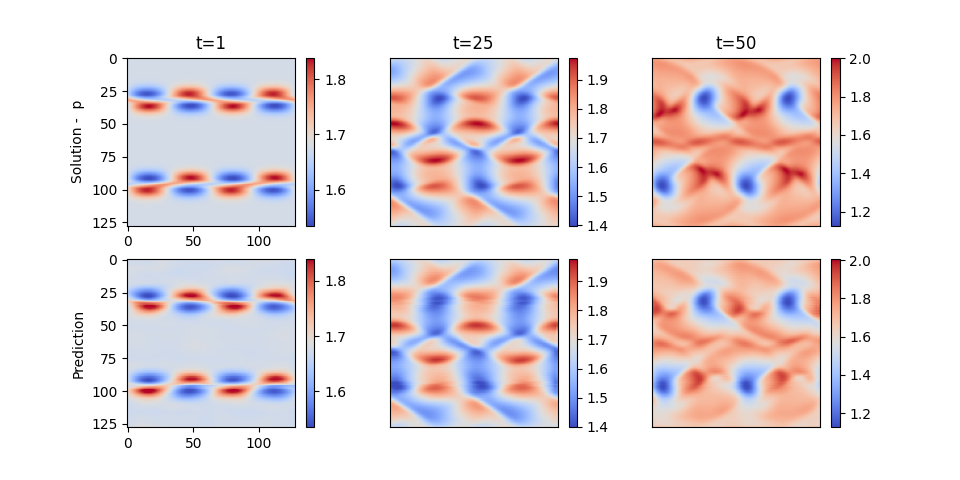}
        \caption{Pressure - NODE UNO}
        \label{fig:NS_comp_p_UNO_NODE}
    \end{subfigure}
    \begin{subfigure}{0.48\textwidth}
        \centering
        \includegraphics[width=0.9\linewidth]{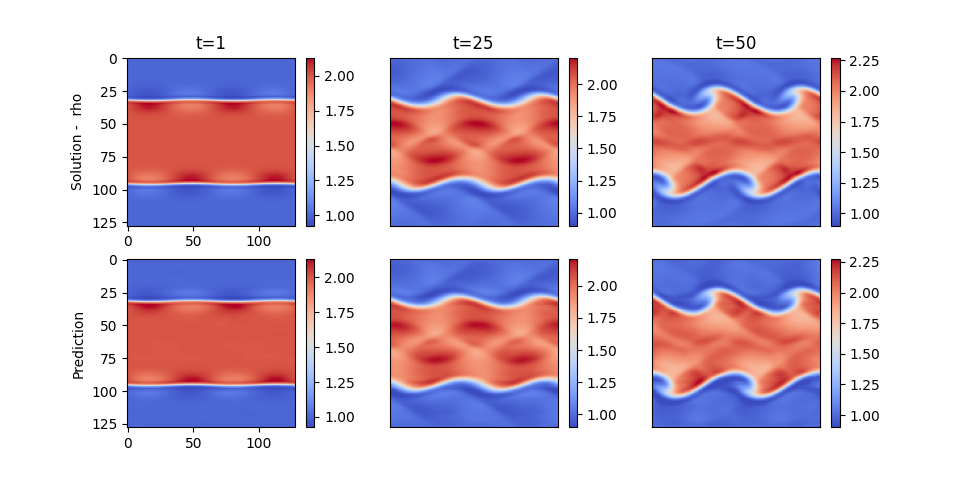}
        \caption{Density - NODE UNO}
        \label{fig:NS_comp_rho_UNO_NODE}
    \end{subfigure}
    \caption{Compressible Navier--Stokes: Model prediction within test distribution for Neural-ODE UNO}
    \label{fig:NS_comp_UNO_node}
\end{figure}

\begin{figure}[H]
    \centering
    \begin{subfigure}{0.48\textwidth}
        \centering
        \includegraphics[width=0.9\linewidth]{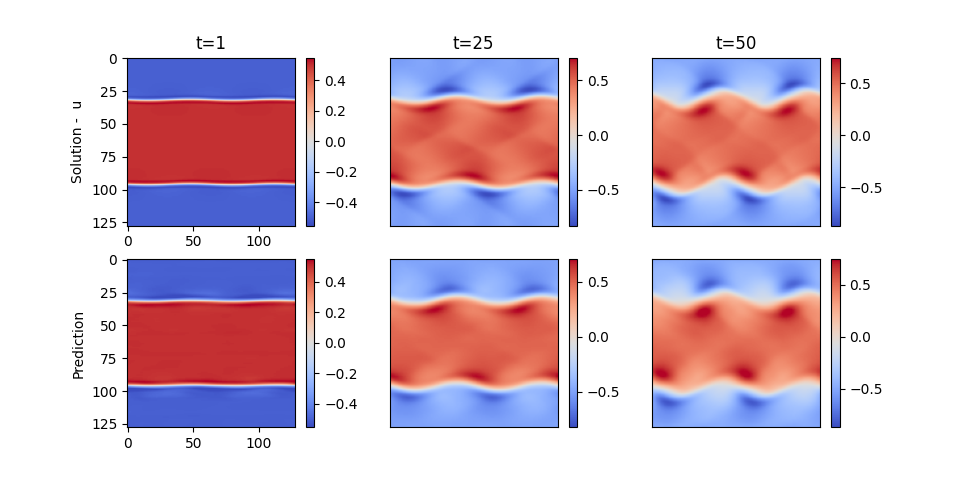}
        \caption{Horizontal velocity - OpsSplit UNO}
        \label{fig:NS_comp_u_UNO_opsplit}
    \end{subfigure}
    \begin{subfigure}{0.48\textwidth}
        \centering
        \includegraphics[width=0.9\linewidth]{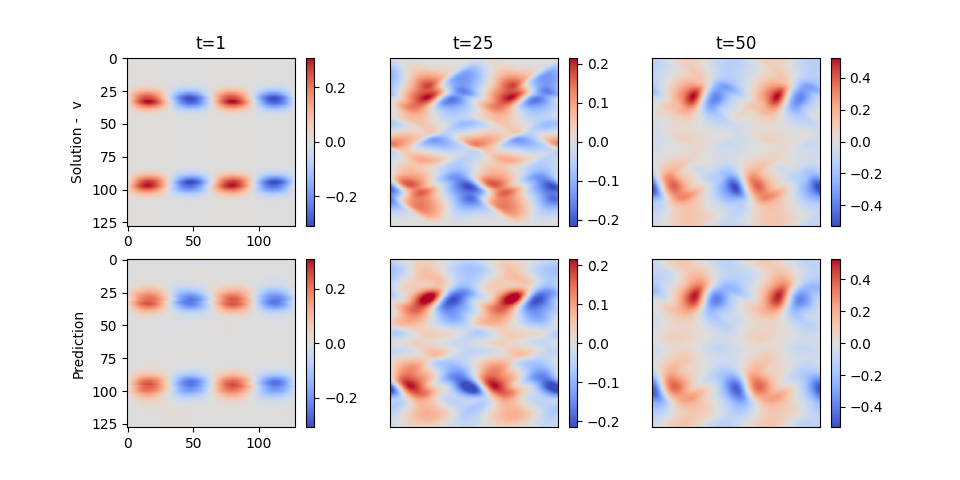}
        \caption{Vertical velocity - OpsSplit UNO}
        \label{fig:NS_comp_v_UNO_opssplit}
    \end{subfigure}
    \begin{subfigure}{0.48\textwidth}
        \centering
        \includegraphics[width=0.9\linewidth]{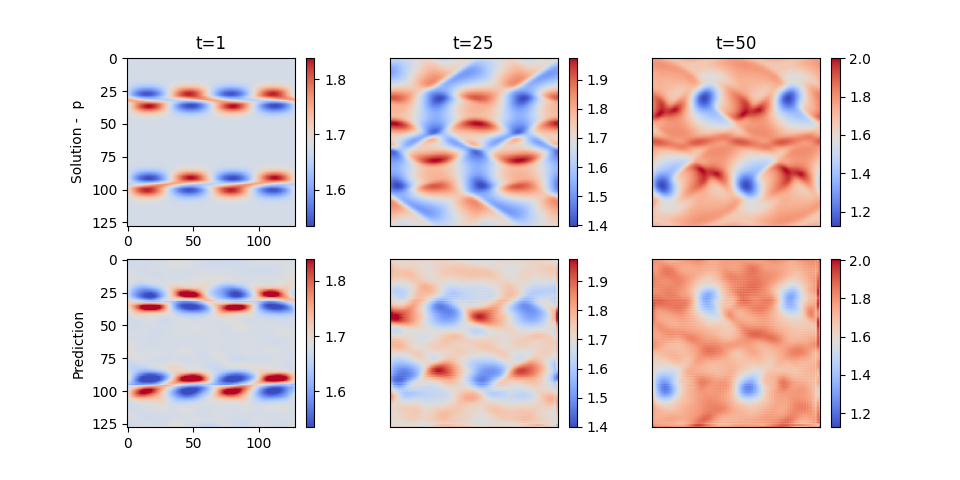}
        \caption{Pressure - OpsSplit UNO}
        \label{fig:NS_comp_p_UNO_opsplit}
    \end{subfigure}
    \begin{subfigure}{0.48\textwidth}
        \centering
        \includegraphics[width=0.9\linewidth]{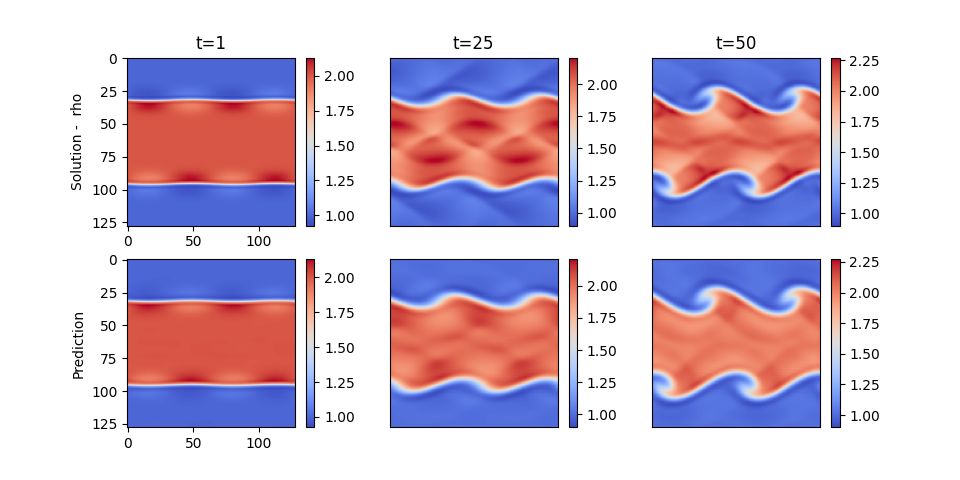}
        \caption{Density - OpsSplit UNO}
        \label{fig:NS_comp_rho_UNO_opssplit}
    \end{subfigure}
    \caption{Compressible Navier--Stokes: Model prediction within test distribution for OpsSplit UNO}
    \label{fig:NS_comp_UNO_opssplit}
\end{figure}

\end{document}